\documentclass{article}

\usepackage{microtype}
\usepackage{graphicx}
\usepackage{subcaption}
\usepackage{booktabs} % for professional tables

\usepackage{hyperref}
\usepackage{natbib}
\usepackage{xcolor}
\usepackage{url}
\usepackage{bm}
\usepackage{amsmath}
\usepackage{amssymb}
\usepackage{mathtools}
\usepackage{amsthm}
\usepackage{thmtools}
\usepackage{amsfonts}
\usepackage{xcolor}
\usepackage{graphicx}
\usepackage{caption}
\usepackage{makecell}
\usepackage{wrapfig}
\usepackage{tikz}
\usepackage{subcaption}
\usepackage{tabularx}
\usepackage{makecell}
\usepackage{multirow}
\usepackage{booktabs}
\usepackage{enumitem}
\usepackage{thm-restate}
\usepackage[table]{xcolor}
\usepackage{float}
\usepackage{enumitem}

\usepackage{algorithm}
\usepackage{algorithmic}

\renewcommand{\algorithmicrequire}{\textbf{Require:}}
\renewcommand{\algorithmicensure}{\textbf{Ensure:}}

\usepackage[accepted]{icml2026}

\usepackage{amsmath}
\usepackage{amssymb}
\usepackage{mathtools}
\usepackage{amsthm}

\usepackage[capitalize, noabbrev]{cleveref}

\theoremstyle{plain}
\newtheorem{theorem}{Theorem}[section]
\newtheorem{proposition}[theorem]{Proposition}
\newtheorem{lemma}[theorem]{Lemma}

\theoremstyle{definition}
\newtheorem{definition}[theorem]{Definition}
\newtheorem{assumption}[theorem]{Assumption}
\theoremstyle{remark}

\crefname{assumption}{Assumption}{Assumptions}

\usepackage[textsize=tiny]{todonotes}
\usepackage[T1]{fontenc}

\icmltitlerunning{Compositional Transduction with Latent Analogies for Offline GCRL}

\begin{document}

\twocolumn[
  \icmltitle{Compositional Transduction with Latent Analogies for Offline Goal-Conditioned Reinforcement Learning}

  \icmlsetsymbol{equal}{*}

  \begin{icmlauthorlist}
    \icmlauthor{Junseok Kim}{snueceasri}
    \icmlauthor{Dohyeong Kim}{indep}
    \icmlauthor{Mineui Hong}{cmu}
    \icmlauthor{Songhwai Oh}{snueceasri}
  \end{icmlauthorlist}

  \icmlaffiliation{snueceasri}{Department of Electrical and Computer Engineering and ASRI, Seoul National University}
  \icmlaffiliation{indep}{Independent researcher}
  \icmlaffiliation{cmu}{Robotics Institute, Carnegie Mellon University}

  \icmlcorrespondingauthor{Songhwai Oh}{songhwai@snu.ac.kr}

  \icmlkeywords{Machine Learning, ICML}

  \vskip 0.3in
]



\printAffiliationsAndNotice{}  % no special notice (required even if empty)

\begin{abstract}
Compositional generalization is essential for reaching unseen goals under novel contextual variations in offline goal-conditioned reinforcement learning (GCRL), where a generalist goal-reaching agent must be learned from limited data. Most prior approaches pursue this via trajectory stitching over temporally contiguous segments, which limits composing behaviors across varying contexts. To overcome this limitation, we formalize \emph{analogy transduction} as synthesizing new plans by composing task-endogenous analogies with given contexts and propose a novel analogy representation tailored for it. Grounded in our theory, this analogy representation captures what changes under optimal task execution, remains invariant to contextual variations, and is sufficient for optimal goal reaching. We further contend that generalization to unseen analogy-context pairs is a practical obstacle in analogy transduction, and introduce a new approach for offline GCRL that enables analogy transduction beyond seen pairs to unseen combinations. We empirically demonstrate the effectiveness of our approach on OGBench manipulation environments, substantially outperforming prior methods that do not perform analogy transduction. Project page: \url{https://rllab-snu.github.io/projects/CTA/}
\end{abstract}

\section{Introduction}\label{sec:intro}
Humans can readily reproduce previously learned behaviors even when the surrounding environment changes. 
For instance, after opening a drawer in a room with an open window, one can perform the same drawer-opening behavior when the window is closed. 
Such reuse and recombination of past behaviors to solve new instances is broadly referred to as \emph{compositional generalization}~\citep{wiedemer2023compositional, ghugare2024closing}, and it is widely regarded as a central challenge in training generalist robotic agents. 
The challenge is particularly acute in offline goal-conditioned reinforcement learning (RL), where a primary goal is to learn a generalist goal-reaching agent from reward-free data, without additional interaction to resolve missing behavior compositions~\citep{kaelbling1993learning, levine2020offline}.

\begin{figure}[t]
  \centering
  \makebox[\columnwidth][l]{%
    \hspace*{0.02\columnwidth}%
    \includegraphics[width=0.92\columnwidth]{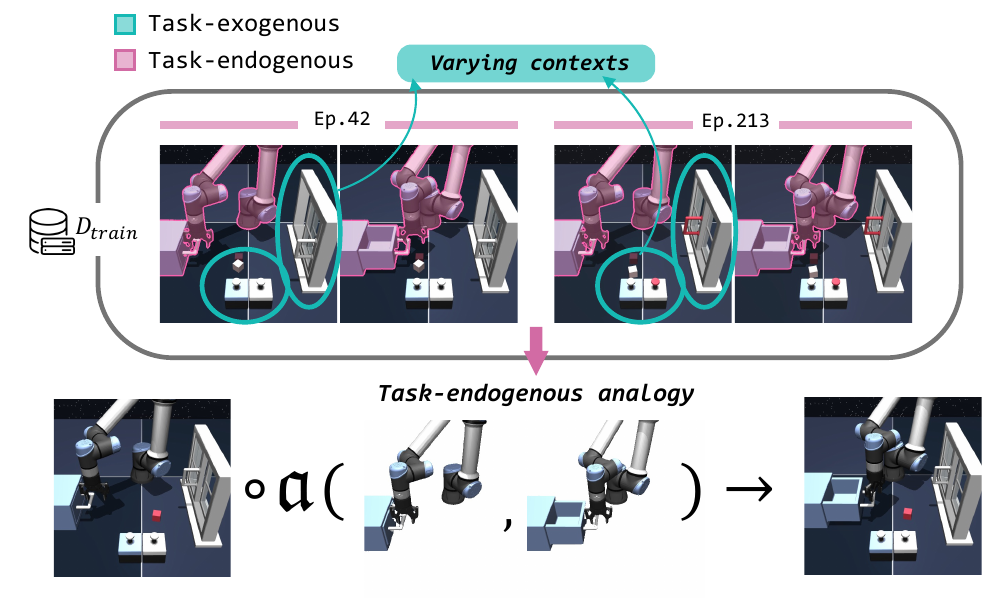}%
  }
  \captionsetup{type=figure, skip=-8pt, font=small, width=\columnwidth}
  \caption{\textbf{Analogy transduction.} Analogy transduction synthesizes new plans by composing task-endogenous analogies into the current context; for instance, an analogy $\mathfrak{a}$ captures \textit{drawer opening} from trajectories under diverse contexts, enabling the agent to open the drawer when the window is closed and unlocked, which may be an absent context from the data.}
  \label{fig:analogy_transduction}
  \vspace{-20pt}
\end{figure}

In offline GCRL, compositional generalization is commonly instantiated through \emph{trajectory stitching}~\citep{ghugare2024closing}, which synthesizes novel goal-reaching behaviors by connecting temporally adjacent transitions. While effective for composing temporally extended behaviors, trajectory stitching does not directly address a complementary regime of compositionality: reusing the same task-relevant transformation across varying task-irrelevant contexts. This limitation naturally motivates the following question: \textit{``Beyond trajectory stitching, can we leverage past behaviors collected under different contexts to infer the same underlying task in the current context?''}
We refer to this synthesis of goal-reaching behavior by transplanting task-endogenous analogies across contexts as \emph{analogy transduction}
%\footnote{\emph{Transductive} learning denotes making predictions for specific test instances by exploiting structure in the observed data, whereas \emph{inductive} learning aims to infer a general rule that applies to unseen instances~\citep{gammerman1998learning}.} (see \cref{fig:analogy_transduction}).
\footnote{\emph{Transduction} denotes predicting specific test instances by exploiting the structure of observed data~\citep{gammerman1998learning}.}.

To enable analogy transduction, we first introduce a novel representation of the task-endogenous analogy. Specifically, we define the analogy between a state $s$ and a goal $g$ as the difference between two optimal temporal distance fields over the entire state space, one anchored at $g$ and the other anchored at $s$.
The proposed analogy representation captures what needs to be changed to reach the goal, while ignoring irrelevant context differences. These analogies are theoretically grounded in a modification of the block controlled Markov process (BCMP) framework~\citep{du2019provably, efroni2022provably, lamb2023guaranteed}, under which task-irrelevant contexts are treated as noise and thus ignored.

We further contend that, under the practical data scarcity of offline GCRL, improving compositional generalization through analogy transduction hinges on the ability to compose unseen analogy--context combinations, which can be viewed as a case of out-of-combination (OOC) generalization~\citep{netanyahu2023learning} (see \cref{sec:prelim}). To this end, we propose \textbf{C}ompositional \textbf{T}ransduction with latent \textbf{A}nalogies (\textbf{CTA}), a new approach for offline GCRL that fully exploits analogy transduction to support goal-reaching across both in-distribution and OOC analogy--context combinations.

Our contributions are as follows. First, we formalize analogy transduction as synthesizing goal-reaching behaviors by composing task-endogenous analogies with task-exogenous contexts, broadening the scope of compositional generalization. Second, we introduce a novel task-endogenous analogy representation, which has useful properties grounded in theoretical analysis. Third, we propose the CTA, a practical approach for offline GCRL that is suitable for analogy transduction with both seen and unseen analogy--context combinations. Finally, we empirically demonstrate the effectiveness of CTA on OGBench~\citep{park2025ogbench} manipulation environments, improving average performance over the strongest baseline by about 42\%.

\section{Related Work}
Offline goal-conditioned RL aims to learn a generalist agent that reaches arbitrary goals in the fewest timesteps possible from unlabeled, reward-free data. Prior work has studied how to estimate minimal timestep-to-go via contrastive learning~\citep{eysenbach2022contrastive, myers2024learning}, by learning a quasimetric~\citep{wang2023optimal, myers2025offline, zheng2026multistep}, through value learning~\citep{park2023hiql, lee2025temporal, giammarino2025physics}, via occupancy matching~\citep{ma2022far, sikchi2023score} or by projecting onto geometric structures~\citep{park2024foundation}, often yielding reusable  representations~\citep{ma2022vip,park2026dual}. In this paper, we propose a novel analogy representation that captures task semantics and enables reuse across contexts in offline GCRL.

Analogies, structure-preserving correspondences across entities, are a common concept in representation learning~\citep{carbonell1983learning}. We focus on explicit vector-space analogies, in contrast to latent analogies used for content-preserving transformations in vision~\citep{karras2019style, radford2021learning} and control~\citep{ghosh2018learning, jang2018grasp2vec, chen2023multi}. 
Classic word embeddings learn approximately linear offsets that encode relations~\citep{mikolov2013efficient, pennington2014glove}, enabling arithmetic transfer, e.g., $\phi(\text{France}) + \big(\phi(\text{Berlin}) - \phi(\text{Germany})\big) \approx \phi(\text{Paris})$. 
In sequential decision making, trajectory analogies support compositional planning~\citep{devin2019plan}, and our closest connection is goal-conditioned bisimulation analogies~\citep{hansen2022bisimulation}. 
Unlike goal-conditioned bisimulation, our temporal distance difference analogies are defined via the optimal temporal distance $d^*$ and avoid reward matching and bootstrapping, making them better suited for offline GCRL
(see \cref{appendix:extended_preliminaries:gcb}).

\section{Preliminaries} \label{sec:prelim}
\paragraph{Notation.} For any set $\mathcal{X}$, $\Delta(\mathcal{X})$ denotes the set of probability distributions over $\mathcal{X}$. For any $q\in\Delta(\mathcal{X})$, $\mathrm{supp}(q)$ denotes its support. %For a function $f$, $\mathrm{dom}(f)$ denotes its domain.
\vspace{-10pt}
\paragraph{Goal-conditioned reinforcement learning.}
A \emph{controlled Markov process} (CMP) is defined by a tuple $\mathcal{M}=(\mathcal{S},\mathcal{A},\mathcal{P})$, where $\mathcal{S}$ is a state space, $\mathcal{A}$ is an action space, and the transition dynamics is a mapping
$\mathcal{P}:\mathcal{S}\times\mathcal{A}\to\Delta(\mathcal{S})$.
Given a goal $g\in\mathcal{S}$,
a goal-conditioned policy $\pi:\mathcal{S}\times\mathcal{S}\to\Delta(\mathcal{A})$,
a goal-conditioned reward $r:\mathcal{S}\times\mathcal{S}\to\mathbb{R}$
and a discount factor $\gamma\in(0,1)$, 
we define a value function as
\begin{equation}
    V^\pi(s,g) \coloneqq \mathbb{E}^\pi\!\left[\sum_{t=0}^{\infty}\gamma^t\,r(s_t,g)\ \Big|\ s_0=s\right],
\end{equation}
where the expectation is taken over trajectories induced by $a_t\sim\pi(\cdot\mid s_t,g)$ and $s_{t+1}\sim\mathcal{P}(\cdot\mid s_t,a_t)$, with $s_0=s$.
We adopt the standard goal-reaching convention $r(s,g)=\mathbf{1}_{\{s=g\}}$ with $g$ absorbing, \textit{i.e.,} once $g$ is reached, the process remains at $g$ and the reward is collected only once.
The goal of goal-conditioned RL (GCRL) is to learn a generalist policy that, for any state--goal pair $(s,g)$, maximizes the expected discounted return; equivalent to reaching $g$ from $s$ in as few steps as possible.
Accordingly, we define the optimal value function as $V^*(s,g)\coloneqq \sup_{\pi} V^\pi(s,g)$. 

We also define the \emph{on-policy} temporal distance as $d^\pi(s,g) \coloneqq \log_\gamma V^\pi(s,g),$
and the \emph{optimal} temporal distance as
\begin{equation}
    d^*(s,g)\coloneqq \log_\gamma V^*(s,g),
\end{equation}
which reduces to the shortest path length from $s$ to $g$ in deterministic environments.
Throughout this paper, we use the term “temporal distance” to refer to the optimal temporal distance $d^*$ unless stated otherwise.

\paragraph{Out-of-combination generalization.}
Consider a problem of estimating a target function $h:\mathcal{X}\to\mathcal{Y}$ at a query input $x\in\mathcal{X}$, where $\mathcal{X}$ is an input space and $\mathcal{Y}$ is a target space. 
Assume  $\mathcal{X}$ is group-structured and admits a displacement mapping
$\delta:\mathcal{X}\times\mathcal{X}\to\delta\mathcal{X}$ together with an apply operator
$\odot:\mathcal{X}\times\delta\mathcal{X}\to\mathcal{X}$ such that $\hat{x}\odot\delta(x,\hat{x})=x$, for all $x, \hat x\in\mathcal X$.
This induces a transductive factorization of a query $x$ into an \emph{anchor} $\hat{x}$ and a \emph{displacement} $\delta(x,\hat{x})$, and yields the reparameterization
$h(x)=\hat{h}\bigl(\hat{x},\delta(x,\hat{x})\bigr)$
with a deterministic map $\hat{h}:\mathcal{X}\times\delta\mathcal{X}\to\mathcal{Y}$.

Let $\bar P_{\mathrm{train}},\bar P_{\mathrm{test}}\in
\Delta(\mathcal X\times\delta\mathcal X)$ denote the train and test
distributions over pairs $(\hat{x},\delta)$, and let
$\bar P_{\cdot,\hat{x}}$ and $\bar P_{\cdot,\delta}$ denote their marginals.
\emph{Out-of-combination} (OOC) corresponds to the regime where anchors and displacements are each
individually in-support, while they are jointly out-of-support:
{\setlength{\abovedisplayskip}{6pt}
 \setlength{\belowdisplayskip}{6pt}
 \setlength{\abovedisplayshortskip}{6pt}
 \setlength{\belowdisplayshortskip}{6pt}
\begin{equation}\nonumber
\begin{aligned}
\mathrm{supp}(\bar P_{\mathrm{test},\hat{x}})&\subseteq
\mathrm{supp}(\bar P_{\mathrm{train},\hat{x}}),\\
\mathrm{supp}(\bar P_{\mathrm{test},\delta})&\subseteq
\mathrm{supp}(\bar P_{\mathrm{train},\delta}),\\
\mathrm{supp}(\bar P_{\mathrm{test}})&\nsubseteq
\mathrm{supp}(\bar P_{\mathrm{train}}).
\end{aligned}
\end{equation}
}

\emph{Extrapolation} then refers to estimating an out-of-support query $x$ by
generalizing to its induced OOC anchor--displacement pair $(\hat{x},\delta)$
at test time.
To extrapolate to an OOC query $x$ and reliably estimate $h(x)$, \citet{netanyahu2023learning} propose \emph{bilinear transduction},
a transductive approach that approximates $\hat{h}$ with the bilinear form
{\setlength{\abovedisplayskip}{6pt}
 \setlength{\belowdisplayskip}{6pt}
 \setlength{\abovedisplayshortskip}{6pt}
 \setlength{\belowdisplayshortskip}{6pt}
\[
\hat{h}\bigl(\hat{x},\delta(x,\hat{x})\bigr)\simeq h_1(\hat{x}) \boldsymbol{\cdot} h_2\bigl(\delta(x,\hat{x})\bigr).
\]
}Under standard assumptions, this approximation admits a guaranteed error bound when estimating a target function on OOC queries due to the low-rank property of the embeddings $h_1$ and $h_2$ (see \cref{appendix:extended_preliminaries:bt}).

In this paper, we adopt this transductive factorization for analogy transduction, using the initial state as an \emph{anchor} and the task-endogenous analogy as a \emph{displacement}.
We apply it to our goal-conditioned value and policy parameterizations to extrapolate to unseen OOC analogy--context compositions, enabling task execution under novel contexts.

\section{Distance Difference Fields as Analogies}
Our goal is to learn a generalist goal-reaching agent with strong compositional generalization by composing diverse task-endogenous analogies and contexts. To this end, we need an effective analogy representation that fully captures task-relevant information and transfers across contexts.
How can we extract such analogies from reward-free data?

\begin{figure*}[t]
  \centering
  %\hspace*{0.05\textwidth}
  \includegraphics[width=0.90\textwidth]{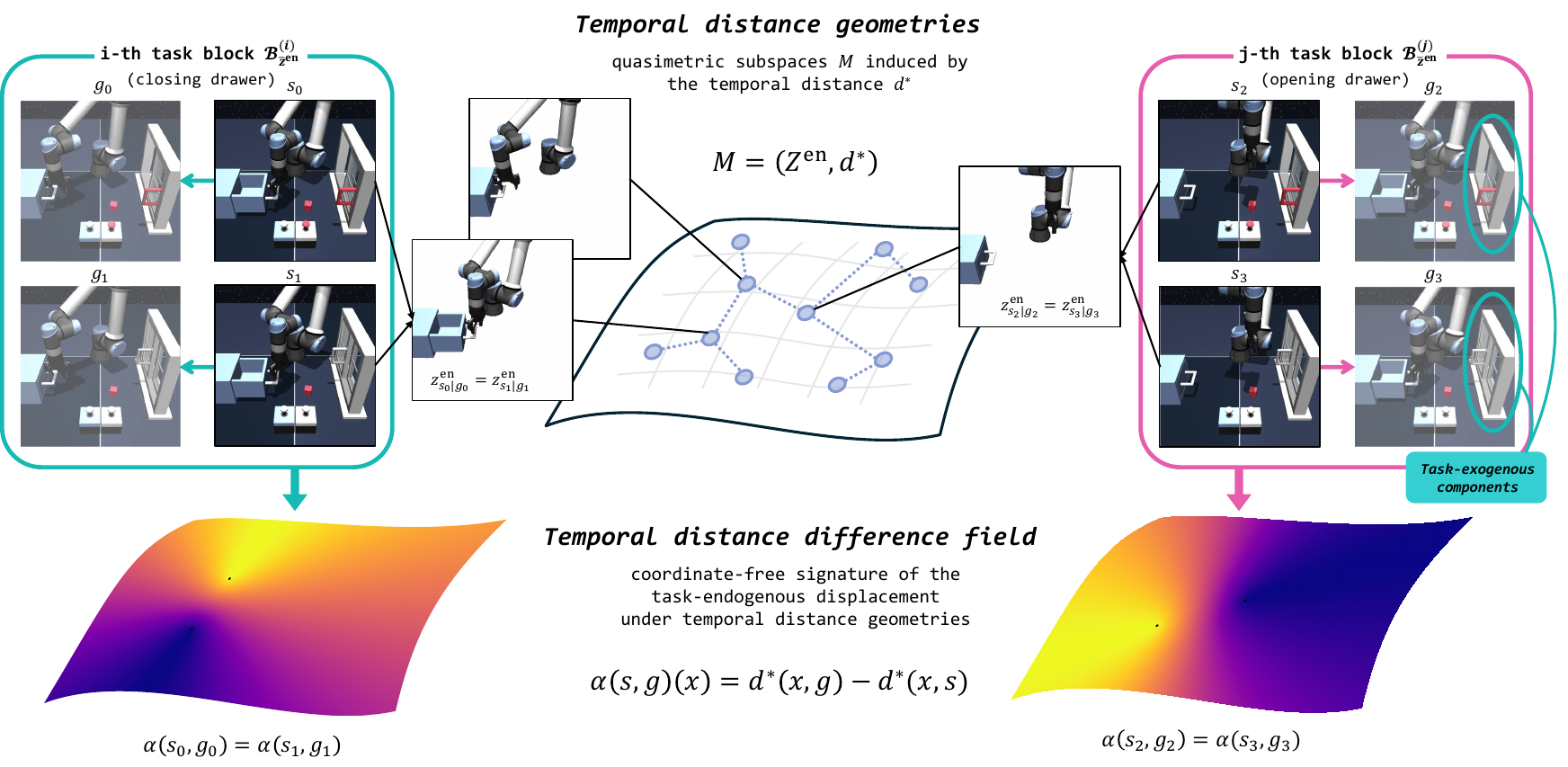}
  \captionsetup{skip=8pt, font=small, width=\textwidth}
  \caption{\textbf{Temporal distance geometry and temporal distance difference field.} \emph{Temporal distance geometry} is the quasimetric space $(\mathcal Z^{\mathrm{en}}, d^*)$ induced by the optimal temporal distance over task-endogenous states, and is invariant to variations in task-exogenous contexts. Here, latent states whose task-endogenous components involve the drawer and the robot arm form a shared geometry across different tasks (e.g., opening vs.\ closing the drawer), independent of task-exogenous contextual factors such as the window state. Each task can be characterized by the \emph{temporal distance difference field}, a coordinate-free signature of the task-endogenous state--goal displacement within the temporal distance geometry. We use this field as an analogy representation, shared within each task block.
  }
  \label{fig:temporal_distance_geometries}
\end{figure*}

\subsection{Key Insights for Analogy Extraction}\label{subsec:insight}
We take an analogous view at the task level and say that a collection of state--goal pairs shares the same \emph{task} if they admit the same optimal execution pattern for reaching a goal from a state. Assume that each state (or goal) can be decomposed into \emph{task-endogenous components} that must change to accomplish the task, and \emph{task-exogenous components} that need not change during the optimal task execution.

In this view, an ideal analogy representation should satisfy two desirable properties. First, it should be invariant to variations in task-exogenous components. Second, it should encode the task-endogenous state--goal displacement in a sufficiently informative manner without degenerate collapse, retaining the information needed for goal-reaching. These properties allow analogies to be reused across diverse contexts as task-level semantics, which is an instance of functional equivariance~\citep{hansen2022bisimulation}.

We recall that the optimal temporal distance $d^*(s,g)$ defines a quasimetric\footnote{A quasimetric is a mathematical metric without the symmetry requirement ($d(x,y)\neq d(y,x)$ in general).} on the state space, and $(\mathcal S,d^*)$ constitutes a quasimetric space~\citep{wang2022learning}, which we refer to as the \emph{temporal distance geometry}.
Our first insight is that the temporal distance geometry is invariant to variations in the task-exogenous components. This intuition can be understood by noting that the temporal distance is determined by which components must be changed under optimal control: whenever two state--goal pairs require the same change in the task-endogenous components, they induce the same relative temporal distance relationships, even in different contexts (see~\cref{fig:temporal_distance_geometries}).
However, a single temporal distance $d^*(s,g)$ can be degenerate, as distinct tasks may collapse to the same value, motivating a richer representation beyond a single scalar.

To obtain such a representation, we draw inspiration from distance-difference representations in geometry.
Prior work~\citep{lassas2019determination, ivanov2020distance} represents a manifold equipped with a geodesic distance $d$ in a coordinate-free manner by embedding each point $x$ into its distance difference function $D_x(y,z)=d(x,y)-d(x,z)$ over an observation set, under which $x$ can be identified from $D_x$ and the underlying metric structure is preserved~\citep{ivanov2020distance}.
Here, our second insight is that a temporal distance difference field yields a sufficiently informative description of the state--goal displacement under the temporal distance geometry. Accordingly,
we represent the analogy of a state--goal pair $(s,g)$ by the temporal distance difference field $\alpha(s,g):\mathcal{S}\to\mathbb{R}$:
\begin{equation} \label{eq:temporal_distance_difference_field}
    \alpha(s,g)(x)=d^*(x,g)-d^*(x,s),
\end{equation}
which can be viewed as evaluating $D_x(g,s)$ under the temporal distance $d^*$. Intuitively, $\alpha(s,g)(x)$ compares the difficulty of reaching $g$ versus $s$ from the probe state $x$. Aggregating this comparison over all states, the field $\alpha(s,g)$ provides an informative signature of the state--goal displacement, mitigating degenerate collapse (see \cref{fig:temporal_distance_geometries}).

In the next section, we formalize these insights within a variant of block CMP (BCMP)~\citep{du2019provably} (see \cref{appendix:extended_preliminaries:exbcmp}) and show that the temporal distance difference field in \eqref{eq:temporal_distance_difference_field} is invariant to variations in task-exogenous components, encodes task-endogenous displacement, and is sufficient for goal-reaching.

\subsection{Goal-Conditioned Endogenous BCMP}\label{subsec:GCE-BCMP}
To formalize our task-endogenous/exogenous decomposition in the previous section, we introduce a goal-conditioned endogenous BCMP. For readability, we provide a brief version here; see~\cref{appendix:GCE-BCMP} for the full definition.
\begin{definition}[GCE-BCMP] \label{def:GCE-BCMP-brief}
A \emph{goal-conditioned endogenous block controlled Markov process} (GCE-BCMP) is specified by a tuple
$(\bar{\mathcal{S}},\bar{\mathcal{Z}},\mathcal{A},\mathcal{P},f^e)$, where
$\bar{\mathcal{S}}:=\mathcal{S}\times\mathcal{S}$ is a product observation space,
$\bar{\mathcal{Z}}:=\mathcal{Z}\times\mathcal{Z}$ is a product latent state space,
$\mathcal{A}$ is an action space,
$\mathcal{P}:\bar{\mathcal{Z}}\times\mathcal{A}\to\Delta(\bar{\mathcal{Z}})$ is a latent transition kernel,
and $f^e:\bar{\mathcal{Z}}\to\Delta(\bar{\mathcal{S}})$ is an emission function.
Let $\mathrm{supp}(f^e):=\bigcup_{\bar z\in\bar{\mathcal Z}}\mathrm{supp}\big(f^e(\cdot\mid \bar z)\big)\subseteq \bar{\mathcal S}$.
The GCE-BCMP makes the following assumptions.

\textbf{(\emph{Block assumption})}
The emission supports are disjoint:
%{\setlength{\abovedisplayskip}{6pt}
% \setlength{\belowdisplayskip}{6pt}
% \setlength{\abovedisplayshortskip}{6pt}
% \setlength{\belowdisplayshortskip}{6pt}
\[
\mathrm{supp}\big(f^e(\cdot\mid \bar z_i)\big) \ \cap \ \mathrm{supp}\big(f^e(\cdot\mid \bar z_j)\big)=\emptyset,
\ \
\forall\, \bar z_i\neq \bar z_j\in\bar{\mathcal Z}.
\]
Thus, there exists a deterministic mapping $f^\ell:\mathrm{supp}(f^e)\to\bar{\mathcal{Z}}$ and two deterministic families
$\{f^\ell_g:\mathcal S\to\mathcal Z\}_{g\in\mathcal S}$ and
$\{f^\ell_s:\mathcal S\to\mathcal Z\}_{s\in\mathcal S}$ such that
$f^\ell(s,g)=\big(f^\ell_g(s),\,f^\ell_s(g)\big):=(z_{s\mid g}, z_{g\mid s})$ for all $(s,g)\in\mathrm{supp}(f^e)$.
Each latent component admits a factorization $\mathcal Z=\mathcal Z^{\mathrm{en}}\times\mathcal Z^{\mathrm{ex}}$, so that
\[
z_{s\mid g}:=\big(z_{s\mid g}^{\mathrm{en}},z_{s\mid g}^{\mathrm{ex}}\big),
\quad
z_{g\mid s}:=\big(z_{g\mid s}^{\mathrm{en}},z_{g\mid s}^{\mathrm{ex}}\big),
\]
For a state--goal pair $(s,g)$, the corresponding $\mathcal Z^{\mathrm{en}}$- and $\mathcal Z^{\mathrm{ex}}$-components are collected to define
\begin{equation}
\begin{aligned}
\bar z^{\mathrm{en}}_{(s,g)}\coloneqq \big(z_{s\mid g}^{\mathrm{en}},\,z_{g\mid s}^{\mathrm{en}}\big)\in\mathcal Z^{\mathrm{en}}\times\mathcal Z^{\mathrm{en}},\nonumber\\
\bar z^{\mathrm{ex}}_{(s,g)}\coloneqq \big(z_{s\mid g}^{\mathrm{ex}},\,z_{g\mid s}^{\mathrm{ex}}\big)\in\mathcal Z^{\mathrm{ex}}\times\mathcal Z^{\mathrm{ex}}.\nonumber
\end{aligned}
\end{equation}
\textbf{(\emph{Task-endogenous abstraction})}
We assume that there exists a Markov kernel
$\mathcal P^{\mathrm{en}}:(\mathcal Z^{\mathrm{en}}\times\mathcal Z^{\mathrm{en}})\times\mathcal A
\to \Delta(\mathcal Z^{\mathrm{en}}\times\mathcal Z^{\mathrm{en}})$
such that $\forall (s,g)\in \mathrm{supp}(f^e)$ and
$\forall \ \mathbf a\in\mathcal A$,
\[
\bar{\mathbf z}' \sim
\mathcal P(\cdot\mid
(\bar z^{\mathrm{en}}_{(s,g)},\bar z^{\mathrm{ex}}_{(s,g)}),\mathbf a)
\ \Longrightarrow\
\bar{\mathbf z}'^{\mathrm{en}}
\sim
\mathcal P^{\mathrm{en}}(\cdot\mid \bar z^{\mathrm{en}}_{(s,g)},\mathbf a),
\]
where $\bar{\mathbf z}'^{\mathrm{en}}$ denotes the
$\mathcal Z^{\mathrm{en}}\times\mathcal Z^{\mathrm{en}}$ component of
$\bar{\mathbf z}'$. Equivalently, the marginal transition of the
task-endogenous component is determined only by the current
task-endogenous component and the action, and is independent of the
task-exogenous context.

This implies that
$\bar z^{\mathrm{en}}_{(s,g)}=(z_{s\mid g}^{\mathrm{en}},\,z_{g\mid s}^{\mathrm{en}})$
encodes the information that determines the task-relevant Bellman
dynamics, while $\bar z^{\mathrm{ex}}_{(s,g)}$ encodes context that does not affect the marginal task-endogenous transition.
In this sense, we refer to $z_{s\mid g}^{\mathrm{en}}$ and
$z_{s\mid g}^{\mathrm{ex}}$ as the
\textbf{\emph{task-endogenous state}} and
\textbf{\emph{task-exogenous context}} of $s$ relative to $g$,
respectively, and analogously to $g$ relative to $s$.
Accordingly, $\bar z^{\mathrm{en}}_{(s,g)}$ is referred to as the
\textbf{\emph{task}} associated with the state--goal pair $(s,g)$, and
the corresponding \textbf{\emph{task block}} is defined as
\[
\mathcal B_{\bar z^{\mathrm{en}}}
\;\coloneqq\;
\big\{(s,g)\in \mathrm{supp}(f^e)\,:\,
\bar z^{\mathrm{en}}_{(s,g)}=\bar z^{\mathrm{en}}\big\}.
\]
\end{definition}

The block assumption is widely used to model rich observations with deterministic latent-state recovery~\citep{du2019provably, zhang2021learning, efroni2022provably, park2026dual} (see \cref{appendix:extended_preliminaries:exbcmp}).
It defines the task-endogenous state and task-exogenous context relative to the paired goal. For the same state $s$, the latent components $z_{s\mid g}^{\mathrm{en}}$ and $z_{s\mid g}^{\mathrm{ex}}$ may vary with $g$, and analogously $z_{g\mid s}^{\mathrm{en}}$ and $z_{g\mid s}^{\mathrm{ex}}$ may vary with $s$, making it more flexible than assuming a consistent task-endogenous partition~\citep{efroni2022provably, ziarko2025contrastive}.
The task-endogenous abstraction assumption is also admissible: assigning distinguishable features to the two subspaces is no more demanding than structured assumptions in prior work~\citep{efroni2022provably, levine2025learning}, and task-endogenous components are often intuitively distinguishable from the contexts that need not change.

Along with the GCE-BCMP, we define a modified reward on state--goal pairs as
%\begin{equation}\label{modified_reward}
$r^\ell(s,g)\coloneqq \mathbf 1_{\{z^\mathrm{en}_{s\mid g}=z^\mathrm{en}_{g\mid s}\}},$
%\end{equation}
so that a state--goal pair receives reward $1$ if and only if the current state matches the goal at the task-endogenous abstraction level.
The induced temporal distance $d^*(s,g)$ can be defined analogously under the modified reward.

We now formalize the two desirable properties of the ideal analogy representation we discussed in \cref{subsec:insight}---invariance to task-exogenous contexts and task-endogenous displacement encoding with enough sufficiency---by introducing the notion of a task-endogenous analogy.

\begin{definition}[Task-endogenous analogy]\label{def:task_endogenous_analogy}
Given a GCE-BCMP, let $\delta\mathcal Z^{\mathrm{en}}$ be a displacement space and let
$\delta:\mathcal Z^{\mathrm{en}}\times\mathcal Z^{\mathrm{en}}\to \delta\mathcal Z^{\mathrm{en}}$ be a well-defined displacement mapping.
A mapping $\mathfrak a:\mathcal S\times\mathcal S\to \delta\mathcal Z^{\mathrm{en}}$ is called a \textbf{\emph{task-endogenous analogy}} if it satisfies the following two conditions:

(\emph{Task-endogenous displacement})
\ For all $(s,g)\in\mathrm{supp}(f^e)$,
\[
\mathfrak{a}(s,g)=\delta\big(z_{s\mid g}^{\mathrm{en}},\,z_{g\mid s}^{\mathrm{en}}\big).
\]
(\emph{Sufficient for optimal goal-reaching})
\ There exists a deterministic policy $\tilde\pi\!:\!\mathcal S\times \delta\mathcal Z^{\mathrm{en}}\!\to\!\mathcal A$ such that, for all $(s,g)\in\mathrm{supp}(f^e)$,
\[
V^{\tilde\pi}(s,g)=V^*(s,g).
\]
Equivalently, the optimal action for $(s,g)$ can be inferred from $(s,\mathfrak a(s,g))$.
\end{definition}

Within GCE-BCMP, the following assumption and proposition show that the temporal distance difference field $\alpha(s,g)$ induces a legitimate task-endogenous analogy and thus a useful analogy representation. The detailed discussion of the assumption and proofs for the proposition are provided in \cref{appendix:theoretical_results}.
\begin{assumption}[Task-block endogenous abstraction consistency]
\label{assump:task_block_coordinate_consistency}
For each task block
$\mathcal B_{\bar z}$,
for every $(s,g)\in\mathcal B_{\bar z}$ and every
probe state $x\in\mathcal S$ for which
$(x,s),(x,g)\in\mathrm{supp}(f^e)$,
\[
z^{\mathrm{en}}_{x\mid s}=z^{\mathrm{en}}_{x\mid g},
\quad
z^{\mathrm{en}}_{s\mid x}=z^{\mathrm{en}}_{s\mid g}, 
\quad
z^{\mathrm{en}}_{g\mid x}=z^{\mathrm{en}}_{g\mid s}.
\]
\end{assumption}
Intuitively, this assumption states that within a task block, task-endogenous states are
interpreted through the same task-endogenous components
(\textit{e.g.}, the drawer and robot arm in the $i$-th task block in
\cref{fig:temporal_distance_geometries}), and remains unchanged under
different comparison endpoints. Hence, $\alpha(s,g)$ isolates the
task-endogenous displacement, while its full distance-difference field still
contains sufficient information for optimal goal-reaching.
\begin{proposition}[Temporal distance difference field is a task-endogenous analogy]
\label{prop:temporal_distance_difference_is_analogy}
Given a GCE-BCMP, let a temporal distance difference function $\alpha:\mathcal{S}\times\mathcal{S}\to(\mathcal{S}\to\mathbb{R})$ be defined as
\begin{equation}\label{eq:global_analogy}
    \alpha(s,g)(x) = d^*(x,g) - d^*(x,s),
\end{equation}
$\forall s,g,x\in\mathcal S.$ Under \cref{assump:task_block_coordinate_consistency}, the field $\alpha(s,g):\mathcal S\!\to\!\mathbb R$ is a task-endogenous analogy. That is, $\alpha(s,g)$ encodes the task-endogenous displacement and is sufficient for optimal goal-reaching.
\end{proposition}

With the notion of a task-endogenous analogy, we provide a formal definition of \emph{analogy transduction}, trajectory synthesis by transplanting analogies into a certain context.
\begin{definition}[Analogy transduction]\label{def:analogy_transduction}
Given a GCE-BCMP, let $\tau_s^g$ denote a trajectory from $s$ to $g$ and let $\mathcal D_{\mathrm{train}}=\{\tau_{(i)}\}_{i=1}^D$ be a training set.
We say that $\tau_{(i)}=\tau_{s_i}^{g_i}$ is \emph{analogously transducible} to $(s,g)$ if $\mathfrak{a}(s_i,g_i)=\mathfrak{a}(s,g)$.
Assume there exists a transduction operator $\mathfrak T:\mathcal S\times \delta\mathcal Z^{\mathrm{en}}\to \Delta(\Gamma(\cdot))$, where $\Gamma(s)$ denotes the set of feasible trajectories starting from $s$.
Then, \emph{analogy transduction} for $(s,g)$ refers to the construction of a new trajectory by choosing any analogously transducible trajectory $\tau_{s_i}^{g_i}\in\mathcal D_{\mathrm{train}}$ and sampling a new trajectory $\boldsymbol{\hat\tau}\in\Gamma(s)$ as
{\setlength{\abovedisplayskip}{6pt}
 \setlength{\belowdisplayskip}{6pt}
 \setlength{\abovedisplayshortskip}{6pt}
 \setlength{\belowdisplayshortskip}{6pt}
\[
\boldsymbol{\hat\tau} \sim \mathfrak T\big(s,\,\mathfrak{a}(s_i,g_i)\big).
\]}
\end{definition}
We emphasize that analogy transduction can encounter both in-distribution and out-of-combination (OOC) analogy--context pairs ($s,\mathfrak{a}$). The in-distribution regime is straightforward: when a start state $s$ and analogy $\mathfrak{a}$ co-occur in the dataset, analogy transduction reduces to retrieving the matching trajectory.
In contrast, in the OOC regime, where $s$ and $\mathfrak{a}$ are both present in the data but never jointly, successful analogy transduction hinges on the OOC extrapolation capability of $\mathfrak T$. In \cref{sec:CTA}, we introduce a practical realization of $\mathfrak T$ that enables such extrapolation. 

Since $\alpha(s,g)$ is sufficient for optimal goal-reaching, using $\alpha$ for analogy transduction preserves optimality while providing a favorable \emph{transductive} bias for generalizing to OOC compositions. Moreover, $\alpha(s,g)$ can improve the stability of analogy transduction in offline settings. Unlike prior on-policy analogies~\citep{hansen2022bisimulation}, it is constructed from the optimal temporal distance $d^*$ and is therefore less susceptible to variability induced by suboptimal data. This comparison is shown in \cref{appendix:additional_experiments:additional_benchmark_results}.

\subsection{Practical Instantiation of the Analogies}\label{subsec:practical_instantiation}
While the temporal distance difference field $\alpha(s,g)$ provides a sufficient task-endogenous analogy, representing the entire field $x\mapsto \alpha(s,g)(x)$ over all $x\in\mathcal S$ is impractical in realistic environments with large or continuous state spaces. To obtain a practical analogy representation, we follow prior work that learns a temporal distance function over all $x\in\mathcal S$ for goal representation~\citep{park2026dual} and approximate the temporal distance as
%{\setlength{\abovedisplayskip}{8pt}
% \setlength{\belowdisplayskip}{8pt}
% \setlength{\abovedisplayshortskip}{8pt}
% \setlength{\belowdisplayshortskip}{8pt}
\begin{equation}\label{eq:d_varphi}
    d^*(s,g)=f\!\left(\phi(s),\varphi(g)\right),
\end{equation}
where $\phi,\varphi:\mathcal{S}\!\to\!\mathbb{R}^d$ are learnable encoders, and $f$ is an arbitrary aggregation function. 
We set $f$ to the inner product, as this choice admits a simple linear decomposition aligned with the difference structure in \eqref{eq:global_analogy} and provides universal approximation with distinct $\phi$ and $\varphi$~\citep{park2023metra}:
%{\setlength{\abovedisplayskip}{8pt}
% \setlength{\belowdisplayskip}{8pt}
% \setlength{\abovedisplayshortskip}{8pt}
% \setlength{\belowdisplayshortskip}{8pt}
\begin{equation}\label{eq:temporal_distance_parameterization}
    f\!\left(\phi(s),\varphi(g)\right)=\phi(s)^\top\varphi(g).
\end{equation}
As a result, the temporal distance difference field can be approximated by the following representation:
%{\setlength{\abovedisplayskip}{8pt}
% \setlength{\belowdisplayskip}{8pt}
% \setlength{\abovedisplayshortskip}{8pt}
% \setlength{\belowdisplayshortskip}{8pt}
\begin{align}
    \alpha(s,g)(x) &= d^*(x,g)-d^*(x,s)\nonumber\\
    &= \phi(x)^\top\varphi(g) - \phi(x)^\top\varphi(s) \nonumber\\
    &= \phi(x)^\top(\varphi(g) - \varphi(s)). \label{eq:practical_parameterizaiton}
\end{align}
We use the term $\varphi(g)-\varphi(s)$ as a practical representation of the temporal distance difference field $\alpha(s,g)$.
Since the parameterization in \eqref{eq:practical_parameterizaiton} is universal for representing $\alpha(s,g)(x)$, $\varphi(g)-\varphi(s)$ is enough to summarize the temporal distance difference relations with all probe states $x\in\mathcal S$ while being independent of $x$.
We denote this quantity by
%{\setlength{\abovedisplayskip}{8pt}
% \setlength{\belowdisplayskip}{8pt}
% \setlength{\abovedisplayshortskip}{8pt}
% \setlength{\belowdisplayshortskip}{8pt}
\begin{equation}\label{eq:dual_analogy}
    \alpha^\vee(s,g) = \varphi(g) - \varphi(s),
\end{equation}
and refer to it as the \textbf{\emph{dual analogy}} following the terminology of prior work~\citep{park2026dual}, emphasizing that it is a geometric signal defined through temporal distance difference relations to all other states.

To extract the dual analogy, we require the parameterized function $f\!\left(\phi(s),\varphi(g)\right)=\phi(s)^\top\varphi(g)$ to approximate the temporal distance between $s$ and $g$.
Among existing temporal distance learning approaches, we adopt a value-learning formulation and instantiate it with goal-conditioned IQL~\citep{kostrikov2022offline}, given its strong empirical performance. Concretely, we use a modified reward $\tilde r(s,g)=-\mathbf{1}_{\{s\neq g\}}$ following~\citep{park2024foundation, giammarino2025physics} and jointly optimize $\phi,\varphi$, and the $Q$-function by minimizing the following losses:
%{\setlength{\abovedisplayskip}{8pt}
% \setlength{\belowdisplayskip}{8pt}
% \setlength{\abovedisplayshortskip}{8pt}
% \setlength{\belowdisplayshortskip}{8pt}
\begin{equation}
\begin{aligned}
    & \quad \ \ \mathcal{L}(\phi,\varphi)=\mathbb{E}_{(s,a,g)}\big[ \ell^\iota_2( \phi(s)^\top\varphi(g)\!-\!\bar Q(s,a,g)) \big],\\
    &\mathcal{L}(Q)=\mathbb{E}_{(s,a,s',g)}\big[ ( Q(s,a,g) \!+ \mathbf{1}_{\{s\neq g\}}\!-\!\gamma \bar \phi(s')^\top \bar\varphi(g))^2 \big],
\end{aligned}
\end{equation}
where $\ell^\iota_2(x)=|\iota-\mathbf{1}_{\{x<0\}}|x^2$ is the expectile loss~\citep{newey1987asymmetric} with $\iota\in(0,1)$, and $\bar{\cdot}$ denotes the target network. After training, we obtain the dual analogy defined in \eqref{eq:dual_analogy} with learned $\varphi$.
While the learned dual analogy admits a broad range of potential applications, the next section presents a generalist goal-reaching agent built via dual analogy transduction as one concrete application.

\section{Compositional Transduction with Analogies}\label{sec:CTA}
In this section, we present a practical method for training a generalist goal-reaching agent in offline GCRL by instantiating the analogy transduction operator $\mathfrak T$ in \cref{def:analogy_transduction}. As discussed in \cref{subsec:GCE-BCMP}, successful analogy transduction hinges on the OOC extrapolation capability of $\mathfrak T$. We therefore propose \textbf{C}ompositional \textbf{T}ransduction with latent \textbf{A}nalogies (\textbf{CTA}), a hierarchical approach that fully leverages analogy transduction by enabling goal-reaching under both in-distribution and OOC analogy--context compositions. Although we instantiate CTA with the practical dual analogy $\alpha^\vee(s,g): \mathcal{S}\times\mathcal{S}\to\mathbb{R}^d$, note that it can be paired with any task-endogenous analogy in \cref{def:task_endogenous_analogy}.

To learn a generalist goal-reaching agent, we follow the hierarchical IQL principle~\citep{park2023hiql} of extracting two hierarchical policies from one shared value function $V:\mathcal{S}\times\mathbb{R}^d\to\mathbb R$ that estimates the temporal distance from a state to a goal.
The high-level policy $\pi_h: \mathcal{S}\times \mathbb{R}^d\to\mathbb{R}^d$ produces the next $k$-step analogy $\alpha^\vee(s_t,s_{t+k})$ treating it as an action, while the low-level policy $\pi_\ell: \mathcal{S}\times \mathbb{R}^d\to\Delta(\mathcal{A})$ outputs the primitive action $a_t$ to realize it.

To enable the value function and policies to extrapolate to OOC analogy--context compositions, we adopt bilinear transduction~\citep{netanyahu2023learning, song2024compositional} (see \cref{sec:prelim}, \cref{appendix:extended_preliminaries:bt}) in the value function and each policy.
Our key intuition is that, given a goal $g$, goal-reaching admits an anchor--displacement view, where the current state $\boldsymbol s$ serves as the \emph{anchor} and the analogy $\boldsymbol{\alpha^\vee(s,g)}$ serves as the \emph{displacement}. This is enabled by $\alpha^\vee(s,g)$ being a well-defined displacement in \cref{def:task_endogenous_analogy}, which satisfies the requirement of bilinear transduction.
Importantly, instead of reusing $f(\phi(s),\varphi(g))$ in \eqref{eq:d_varphi}, we need to learn a separate value function $V$ that enforces the low-rank structure required for bilinear transduction by embedding the anchor $s$ and displacement $\alpha^\vee(s,g)$ into a $b$-dimensional bottleneck.
Concretely, we parameterize $V$ as
{\setlength{\abovedisplayskip}{8pt}
 \setlength{\belowdisplayskip}{8pt}
 \setlength{\abovedisplayshortskip}{8pt}
 \setlength{\belowdisplayshortskip}{8pt}
\begin{equation}
V(s,g)
\;=\;
\Omega_1(s) \boldsymbol{\cdot} \Omega_2(\alpha^\vee(s,g)),
\label{eq:bilinear-value}
\end{equation}
}where $\Omega_1:\mathcal S\!\to\!\mathbb R^{b}$ and $\Omega_2:\mathbb R^{d}\!\to\!\mathbb R^{b}$ are learnable anchor and displacement encoders, respectively, with $b\ll d$.

Both policies are also parameterized via bilinear transduction to support extrapolation to OOC compositions.
Specifically, we model them as Gaussian actors with fixed covariance $\Sigma_h$ and $\Sigma_\ell$, respectively: for all $s,g\in \mathcal S$,
$
\pi_h(\,\cdot\mid s,\alpha^\vee(s,g))
=
\mathcal N\!\big(\mu_h(s,\alpha^\vee(s,g)),\Sigma_h\big),
\ 
\pi_\ell(\,\cdot\mid s,\alpha^\vee(s,g))
=
\mathcal N\!\big(\mu_\ell(s,\alpha^\vee(s,g)),\Sigma_\ell\big),
$
such that
\begin{align}
    \mu_h(s,\alpha^\vee(s,g))&=\omega_{h1}(s) \boldsymbol{\cdot}  \omega_{h2}(\alpha^\vee(s,g)),\label{eq:bilinear-policy-h} \\
    \mu_\ell(s,\alpha^\vee(s,g))&=\omega_{\ell1}(s) \boldsymbol{\cdot} \omega_{\ell2}(\alpha^\vee(s,g)),\label{eq:bilinear-policy-l}
\end{align}
where $\omega_{h1}:\mathcal S\to\mathbb R^{b\times d},
\omega_{\ell1}:\mathcal S\to\mathbb R^{b\times \dim(\mathcal A)}$ are the learnable anchor encoders and $\omega_{h2}:\mathbb R^{d}\to\mathbb R^{b\times d}, \omega_{\ell2}:\mathbb R^{d}\to\mathbb R^{b\times \dim(\mathcal A)}$ are the learnable displacement encoders.

We train $V$ by minimizing the following action-free variant of IQL loss~\citep{kostrikov2022offline, park2023hiql, ghosh2023reinforcement, giammarino2025physics} to effectively mitigate out-of-distribution value estimation:
%{\setlength{\abovedisplayskip}{8pt}
% \setlength{\belowdisplayskip}{8pt}
% \setlength{\abovedisplayshortskip}{8pt}
% \setlength{\belowdisplayshortskip}{8pt}
\begin{equation}
    \mathcal{L}(\Omega_1, \Omega_2)=\mathbb E_{(s,s',g)}\big[ \ell^\kappa_2(-\mathbf{1}_{\{s\neq g\}} + \gamma \bar V(s',g) - V(s,g)) \big],
\end{equation}
where $\kappa\in(0,1)$, and $\bar{\cdot}$ denotes the target network. Both the high- and low-level policies are trained by maximizing the following advantage-weighted regression objectives, respectively~\citep{peng2019advantage, park2025ogbench}:
%{\setlength{\abovedisplayskip}{8pt}
% \setlength{\belowdisplayskip}{8pt}
% \setlength{\abovedisplayshortskip}{8pt}
% \setlength{\belowdisplayshortskip}{8pt}
\begin{align}
&\mathcal{L}(\omega_{h1}, \omega_{h2})=\mathbb E_{({s}_t,s_{t+k},g)}
\Big[
\exp\!\big(\beta_h A(s_t, s_{t+k},g)\big)\nonumber\\
& \ \ \ \ \ \ \qquad\qquad\log \pi_h(\alpha^\vee(s_t,s_{t+k})\mid s_t,\alpha^\vee(s_t,g))
\Big],\\
&\mathcal{L}(\omega_{\ell1},\omega_{\ell2})=\mathbb E_{({s}_t, a_t,s_{t+1}, s_{t+k})}
\Big[
\exp\!\big(\beta_\ell A(s_t, s_{t+1},s_{t+k})\big) \nonumber\\
& \ \ \ \ \ \qquad\qquad\qquad\quad \ \log \pi_\ell(a_t\mid s_t,\alpha^\vee(s_t,s_{t+k}))
\Big],
\end{align}
where $A(s,s',g):=V(s',g)-V(s,g)$ is an advantage function and $\beta_h, \beta_\ell$ are temperature parameters that adjust the relative weight of behavior cloning. Additional details for training CTA are provided in \cref{appendix:algorithm_detail:CTA}.

During inference, given a goal $g$, the high-level policy $\pi_h$ samples a next $k$-step analogy $\alpha^\vee(s_t,s_{t+k}) \sim \pi_h(\cdot \mid s_t,\alpha^\vee(s_t,g))$ and the low-level policy $\pi_\ell$ executes a primitive action $a_t \sim \pi_\ell(\cdot \mid s_t,\alpha^\vee(s_t,s_{t+k}))$ at each timestep.
Under bilinear transduction, $\pi_h$ can extrapolate to novel anchor--goal pairs $(s_t,g)$ to propose a meaningful analogy.
Likewise, $\pi_\ell$ can extrapolate to unseen anchor--analogy pairs $(s_t,\alpha^\vee(s_t,s_{t+k}))$ to recover primitive controls that reflect task semantics transferred from past experiences.

The hierarchical structure improves compositional generalization by making analogy transduction more effective and stable. Since long-horizon analogies are sparse in offline datasets~\citep{hong2023diffused, myers2025horizon}, we decompose behavior into $k$-step analogies, expanding the pool of reusable ones. Conditioning the low-level policy on proposed analogies stabilizes transduction while avoiding out-of-distribution analogy queries beyond the intended OOC regime. We validate these effects in \cref{appendix:additional_experiments:hierarchical_structure}.

%Note that the CTA implements an implicit form of analogy transduction: instead of explicitly transplanting actions from other trajectories, it infers the value of executing an analogy under the current context and derives actions accordingly.
%In particular, the bilinearly parameterized value enables OOC inference over context--analogy compositions, and assigns low value to infeasible analogies rather than forcing blind action transfer, which stabilizes the transduction.

\begin{table*}[t]
\centering
\captionsetup{skip=3pt, font=small, width=\textwidth}
\caption{\textbf{Results in OGBench manipulation environments (8 seeds).}
Top-3 methods are highlighted with color gradation where darker indicates higher rank; methods within 95\% of a higher-ranked score share the same color. \textbf{Bold} indicates the best score for the average.}
\label{tab:benchmark}
\renewcommand{\arraystretch}{1.1}
\scriptsize
\begin{tabularx}{\textwidth}{ll *{12}{>{\centering\arraybackslash}X}}
\toprule
\multicolumn{2}{c}{} &
\multicolumn{6}{c}{without representations} &
\multicolumn{4}{c}{dual goal representations} &
\multicolumn{2}{c}{dual analogies} \\
\cmidrule(lr){3-8}\cmidrule(lr){9-12}\cmidrule(lr){13-14}
 & \textbf{Dataset} &
\textbf{GCBC} & \textbf{QRL} & \textbf{CRL} & \textbf{GCIVL} & \textbf{GCIQL} & \textbf{HIQL} &
\textbf{CRL}$^\vee$ & \textbf{GCIVL}$^\vee$ & \textbf{GCIQL}$^\vee$ & \textbf{HIQL}$^\vee$ &
\textbf{HIQL}$^\vee_{+\alpha^\vee}$ & \textbf{CTA} \\
\midrule

\multirow[c]{1}{*}{\texttt{scene}}
 & \texttt{play}
   & $5{\scalebox{0.7}{$\pm 1$}}$
   & $5{\scalebox{0.7}{$\pm 1$}}$
   & $19{\scalebox{0.7}{$\pm 2$}}$
   & $42{\scalebox{0.7}{$\pm 4$}}$
   & $51{\scalebox{0.7}{$\pm 4$}}$
   & $38{\scalebox{0.7}{$\pm 3$}}$
   & $44{\scalebox{0.7}{$\pm 5$}}$
   & {\cellcolor{juncolorblue!18}$72{\scalebox{0.7}{$\pm 6$}}$}
   & $53{\scalebox{0.7}{$\pm 3$}}$
   & {\cellcolor{juncolorblue!90}$87{\scalebox{0.7}{$\pm 4$}}$}
   & $80{\scalebox{0.7}{$\pm 5$}}$
   & {\cellcolor{juncolorblue!90}$90{\scalebox{0.7}{$\pm 4$}}$} \\
\midrule

\multirow[c]{3}{*}{\texttt{cube}}
 & \texttt{single-play}
   & $6{\scalebox{0.7}{$\pm 2$}}$
   & $5{\scalebox{0.7}{$\pm 1$}}$
   & $19{\scalebox{0.7}{$\pm 2$}}$
   & $53{\scalebox{0.7}{$\pm 4$}}$
   & $68{\scalebox{0.7}{$\pm 6$}}$
   & $15{\scalebox{0.7}{$\pm 3$}}$
   & $60{\scalebox{0.7}{$\pm 1$}}$
   & {\cellcolor{juncolorblue!90}$89{\scalebox{0.7}{$\pm 3$}}$}
   & {\cellcolor{juncolorblue!90}$87{\scalebox{0.7}{$\pm 2$}}$}
   & $69{\scalebox{0.7}{$\pm 3$}}$
   & $74{\scalebox{0.7}{$\pm 4$}}$
   & {\cellcolor{juncolorblue!90}$86{\scalebox{0.7}{$\pm 3$}}$} \\
 & \texttt{double-play}
   & $1{\scalebox{0.7}{$\pm 1$}}$
   & $1{\scalebox{0.7}{$\pm 0$}}$
   & $10{\scalebox{0.7}{$\pm 2$}}$
   & $36{\scalebox{0.7}{$\pm 3$}}$
   & {\cellcolor{juncolorblue!18}$40{\scalebox{0.7}{$\pm 5$}}$} % <-- (추가) GCIQL도 3등 동률이라 색칠
   & $6{\scalebox{0.7}{$\pm 2$}}$
   & $24{\scalebox{0.7}{$\pm 5$}}$
   & {\cellcolor{juncolorblue!90}$60{\scalebox{0.7}{$\pm 4$}}$}
   & {\cellcolor{juncolorblue!18}$40{\scalebox{0.7}{$\pm 5$}}$}
   & $38{\scalebox{0.7}{$\pm 8$}}$
   & $30{\scalebox{0.7}{$\pm 3$}}$
   & {\cellcolor{juncolorblue!55}$50{\scalebox{0.7}{$\pm 5$}}$} \\
 & \texttt{triple-play}
   & $1{\scalebox{0.7}{$\pm 1$}}$
   & $0{\scalebox{0.7}{$\pm 0$}}$
   & $4{\scalebox{0.7}{$\pm 1$}}$
   & $1{\scalebox{0.7}{$\pm 0$}}$
   & $3{\scalebox{0.7}{$\pm 1$}}$
   & $3{\scalebox{0.7}{$\pm 1$}}$
   & $8{\scalebox{0.7}{$\pm 1$}}$
   & $2{\scalebox{0.7}{$\pm 0$}}$
   & $1{\scalebox{0.7}{$\pm 0$}}$
   & {\cellcolor{juncolorblue!90}$18{\scalebox{0.7}{$\pm 1$}}$}
   & {\cellcolor{juncolorblue!18}$11{\scalebox{0.7}{$\pm 2$}}$}
   & {\cellcolor{juncolorblue!55}$17{\scalebox{0.7}{$\pm 1$}}$}\\
\midrule

\multirow[c]{4}{*}{\texttt{puzzle}}
 & \texttt{3x3-play}
   & $2{\scalebox{0.7}{$\pm 0$}}$
   & $1{\scalebox{0.7}{$\pm 0$}}$
   & $3{\scalebox{0.7}{$\pm 1$}}$
   & $6{\scalebox{0.7}{$\pm 1$}}$
   & {\cellcolor{juncolorblue!90}$95{\scalebox{0.7}{$\pm 1$}}$}
   & $12{\scalebox{0.7}{$\pm 2$}}$
   & $6{\scalebox{0.7}{$\pm 1$}}$
   & $5{\scalebox{0.7}{$\pm 1$}}$
   & $42{\scalebox{0.7}{$\pm 1$}}$
   & {\cellcolor{juncolorblue!18}$79{\scalebox{0.7}{$\pm 12$}}$}
   & $72{\scalebox{0.7}{$\pm 9$}}$
   & {\cellcolor{juncolorblue!90}$94{\scalebox{0.7}{$\pm 11$}}$} \\
 & \texttt{4x4-play}
   & $0{\scalebox{0.7}{$\pm 0$}}$
   & $0{\scalebox{0.7}{$\pm 0$}}$
   & $0{\scalebox{0.7}{$\pm 0$}}$
   & $13{\scalebox{0.7}{$\pm 2$}}$
   & $26{\scalebox{0.7}{$\pm 3$}}$
   & $7{\scalebox{0.7}{$\pm 2$}}$
   & $2{\scalebox{0.7}{$\pm 0$}}$
   & $23{\scalebox{0.7}{$\pm 3$}}$
   & {\cellcolor{juncolorblue!18}$34{\scalebox{0.7}{$\pm 2$}}$}
   & $16{\scalebox{0.7}{$\pm 4$}}$
   & {\cellcolor{juncolorblue!55}$50{\scalebox{0.7}{$\pm 5$}}$}
   & {\cellcolor{juncolorblue!90}$84{\scalebox{0.7}{$\pm 3$}}$} \\
 & \texttt{4x5-play}
   & $0{\scalebox{0.7}{$\pm 0$}}$
   & $0{\scalebox{0.7}{$\pm 0$}}$
   & $1{\scalebox{0.7}{$\pm 0$}}$
   & $7{\scalebox{0.7}{$\pm 1$}}$
   & {\cellcolor{juncolorblue!55}$14{\scalebox{0.7}{$\pm 1$}}$}
   & $4{\scalebox{0.7}{$\pm 1$}}$
   & $0{\scalebox{0.7}{$\pm 0$}}$
   & $5{\scalebox{0.7}{$\pm 1$}}$
   & {\cellcolor{juncolorblue!18}$10{\scalebox{0.7}{$\pm 1$}}$}
   & $5{\scalebox{0.7}{$\pm 1$}}$
   & $0{\scalebox{0.7}{$\pm 0$}}$
   & {\cellcolor{juncolorblue!90}$17{\scalebox{0.7}{$\pm 1$}}$} \\
 & \texttt{4x6-play}
   & $0{\scalebox{0.7}{$\pm 0$}}$
   & $0{\scalebox{0.7}{$\pm 0$}}$
   & $4{\scalebox{0.7}{$\pm 0$}}$
   & {\cellcolor{juncolorblue!18}$10{\scalebox{0.7}{$\pm 1$}}$}
   & {\cellcolor{juncolorblue!90}$12{\scalebox{0.7}{$\pm 1$}}$}
   & $3{\scalebox{0.7}{$\pm 1$}}$
   & $0{\scalebox{0.7}{$\pm 0$}}$
   & $2{\scalebox{0.7}{$\pm 1$}}$
   & $6{\scalebox{0.7}{$\pm 1$}}$
   & $2{\scalebox{0.7}{$\pm 1$}}$
   & $0{\scalebox{0.7}{$\pm 0$}}$
   & {\cellcolor{juncolorblue!90}$12{\scalebox{0.7}{$\pm 2$}}$} \\
\midrule

\multicolumn{2}{c}{Average}
 & $1.9$ & $1.5$ & $7.5$ & $21.0$ & $38.6$ & $11.0$
 & $18.0$ & $32.2$ & $34.1$ & $39.3$ & $39.6$ & $\mathbf{56.3}$ \\
\bottomrule
\end{tabularx}
\vspace{-5pt}
\end{table*}

\section{Experiments}\label{sec:main_experiments}
In this section, we empirically validate the effectiveness of analogy transduction and examine how efficiently and robustly CTA leverages it.
Our experiments are designed to answer the following questions.
(1) Does CTA with dual analogies achieve competitive generalization performance?
(2) Are CTA's compositional generalization gains genuinely driven by OOC extrapolation?
(3) Do our dual analogies indeed capture task-endogenous displacements?

\subsection{Experimental Setup}
\paragraph{Environments and datasets.}
We conduct our main experiments on eight manipulation environments from OGBench, an offline GCRL benchmark~\citep{park2025ogbench} where compositional generalization is essential due to task-exogenous variations.
The environments fall into \texttt{scene}, \texttt{cube}, and \texttt{puzzle}: the agent controls a robotic arm to match a target goal state, with \texttt{cube} requiring block rearrangement, \texttt{scene} requiring sequential object interactions, and \texttt{puzzle} requiring combinatorial button pressing.
We use the standard \texttt{play} datasets provided by OGBench, which comprise diverse, reward-free interaction trajectories collected without task-specific reward engineering. Full experimental details are provided in \cref{appendix:experimental_detail:environments}.

\paragraph{Baselines.}
We compare against prior baselines that have been evaluated in OGBench manipulation environments to ensure a relevant and fair assessment.
Baselines are grouped into methods without explicit state--goal representations and those using \textbf{dual goal representations}~\citep{park2026dual}, which encode goals via relative distance relationships to other states and are closely related to our temporal distance difference fields in terms of practical implementation.
The first group includes \textbf{GCBC}~\citep{ghosh2021learning}, a goal-conditioned behavior cloning method;
\textbf{QRL}~\citep{wang2023optimal}, a quasimetric learning approach;
\textbf{CRL}~\citep{eysenbach2022contrastive}, a contrastive RL;
\textbf{GCIVL}~\citep{kostrikov2022offline, park2025ogbench} and \textbf{GCIQL}~\citep{kostrikov2022offline}, TD-based offline value and Q-learning methods, respectively;
and \textbf{HIQL}~\citep{park2023hiql}, a hierarchical IQL.
The second group includes \textbf{CRL}$\boldsymbol{^\vee}$, \textbf{GCIVL}$\boldsymbol{^\vee}$, \textbf{GCIQL}$\boldsymbol{^\vee}$, and \textbf{HIQL}$\boldsymbol{^\vee}$,
which augment their corresponding base algorithms with the dual goal representation. 
Notably, HIQL$^\vee$ adopts a goal-conditioned value function and a hierarchical policy structure similar to ours but does not use the analogy representation or analogy transduction.
\textbf{HIQL}$\boldsymbol{^\vee\!\!_{+\!\alpha^\vee}}$ replaces the dual goal representation $\varphi(g)$ in the value function and hierarchical policies of HIQL$^\vee$ with our dual analogy $\alpha^\vee(s,g)$, thus inheriting the representational expressivity of dual analogies while remaining incapable of OOC analogy--context composition.
Further implementation details of the baselines can be found in \cref{appendix:experimental_detail:baselines}.

\subsection{Results in OGBench Manipulation Suite}\label{subsec:benchmark_result}
\cref{tab:benchmark} reports the performance of CTA and baseline methods across eight manipulation environments in OGBench.
CTA achieves the best or near-best results on six of the eight tasks and improves the overall average performance by about $42\%$ over the strongest baseline.
Remarkably, the gains are most pronounced on \texttt{puzzle} environments, where the exponentially large state space makes compositional generalization critical~\citep{park2025ogbench}. 
CTA improves the average performance over the four \texttt{puzzle} environments by about $40\%$ compared to the strongest baseline on this subset, and achieves about a $2.5\times$ higher score than the best-performing baseline on the \texttt{$4\times4$} environment.
Moreover, CTA consistently outperforms baselines with dual goal representations, indicating that its advantage stems from robust generalization via analogy transduction rather than from dual goal representations.

\begin{figure}[t]
    \centering
    \captionsetup{font=small, skip=0pt, width=\columnwidth}
    \includegraphics[width=\columnwidth]{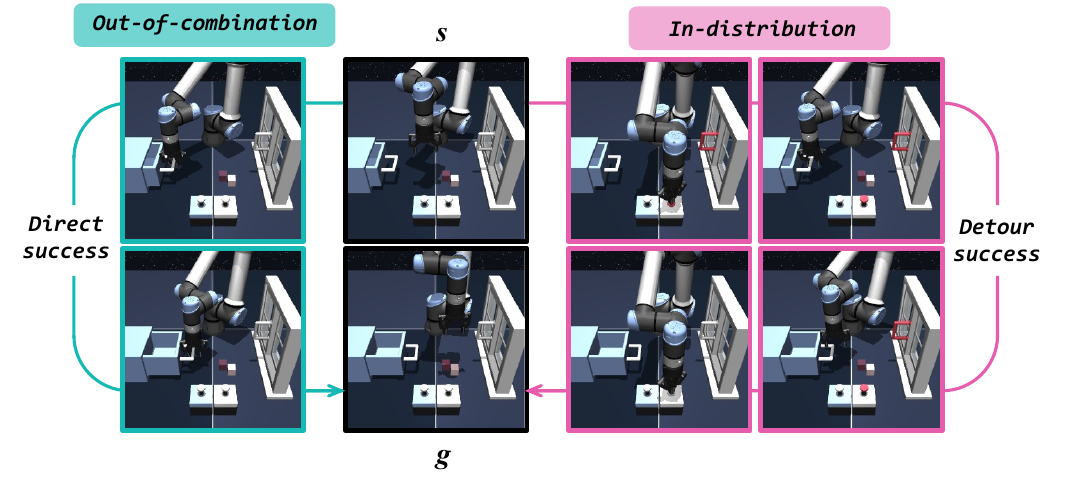}
    \caption{\textbf{Example of direct OOC case study.}
    We remove all direct drawer-opening trajectories when the drawer and window are closed and both are unlocked.
    The agent can achieve \emph{direct success} by extrapolating to this OOC context--task pair, or \emph{detour success} via an in-distribution sequence: lock the window, open the drawer, then unlock the window.}
    \label{fig:ooc_example}
    \vspace{-15pt}
\end{figure}

\subsection{Analogy Transduction Requires Extrapolation}
To verify whether CTA's strong generalization arises from its OOC extrapolation capability when performing analogy transduction, we compare CTA with HIQL$^\vee$ and HIQL$^\vee\!\!_{+\!\alpha^\vee}$.
As shown in the last three columns of \cref{tab:benchmark}, and consistent with the theoretical sufficiency of the dual analogy in \cref{prop:temporal_distance_difference_is_analogy}, HIQL$^\vee$ and HIQL$^\vee\!\!_{+\!\alpha^\vee}$ attain comparable average performance; in contrast, CTA substantially outperforms both by explicitly performing OOC extrapolation via bilinear transduction.
These results suggest that the generalization gains are driven by the OOC extrapolation capability, and empirically underscore the importance of OOC analogy--context inference in analogy transduction.
Although CTA adopts a bilinear transductive parameterization and thus differs architecturally from HIQL$^\vee\!\!_{+\!\alpha^\vee}$, it has about $20\%$ fewer parameters, suggesting that the gains are unlikely to be explained by the model capacity.

\begin{table}[t]
\centering
\caption{\textbf{Direct OOC case study results on \texttt{scene-play-v0} and \texttt{puzzle-4x4-play-v0} (4 seeds)}. Each entry is reported as direct success rate (success rate).}
\label{tab:ooc_case_study_overall}
\setlength{\tabcolsep}{4pt}
\scriptsize
\begin{tabular}{lccccc}
\toprule
Dataset & \textbf{HIQL} & \textbf{GCIQL$\boldsymbol{^\vee}$} & \textbf{HIQL$\boldsymbol{^\vee}$} & \textbf{HIQL}$^\vee_{+\alpha^\vee}$ & \textbf{CTA} \\
\midrule
\texttt{scene}
& \begin{tabular}[c]{@{}c@{}}$19{\scalebox{0.7}{$\pm$}}10$\\($42{\scalebox{0.7}{$\pm$}}12$)\end{tabular}
& \begin{tabular}[c]{@{}c@{}}$51{\scalebox{0.7}{$\pm$}}10$\\($63{\scalebox{0.7}{$\pm$}}11$)\end{tabular}
& \begin{tabular}[c]{@{}c@{}}$45{\scalebox{0.7}{$\pm$}}11$\\($87{\scalebox{0.7}{$\pm$}}7$)\end{tabular}
& \begin{tabular}[c]{@{}c@{}}$48{\scalebox{0.7}{$\pm$}}14$\\($86{\scalebox{0.7}{$\pm$}}6$)\end{tabular}
& \begin{tabular}[c]{@{}c@{}}$\mathbf{73}{\scalebox{0.7}{$\pm$}}\mathbf{9}$\\($\mathbf{94}{\scalebox{0.7}{$\pm$}}\mathbf{4}$)\end{tabular} \\
\midrule
\texttt{puzzle-4x4}
& \begin{tabular}[c]{@{}c@{}}$37{\scalebox{0.7}{$\pm$}}11$\\($69{\scalebox{0.7}{$\pm$}}9$)\end{tabular}
& \begin{tabular}[c]{@{}c@{}}$44{\scalebox{0.7}{$\pm$}}11$\\($55{\scalebox{0.7}{$\pm$}}12$)\end{tabular}
& \begin{tabular}[c]{@{}c@{}}$35{\scalebox{0.7}{$\pm$}}17$\\($62{\scalebox{0.7}{$\pm$}}13$)\end{tabular}
& \begin{tabular}[c]{@{}c@{}}$66{\scalebox{0.7}{$\pm$}}11$\\($95{\scalebox{0.7}{$\pm$}}4$)\end{tabular}
& \begin{tabular}[c]{@{}c@{}}$\mathbf{80}{\scalebox{0.7}{$\pm$}}\mathbf{8}$\\($\mathbf{100}{\scalebox{0.7}{$\pm$}}\mathbf{1}$)\end{tabular} \\
\bottomrule
\end{tabular}
\end{table}

To test whether CTA indeed extrapolates to OOC context--task pairs, we construct  direct OOC case studies on \texttt{scene} and \texttt{puzzle-4x4} by holding out selected pairs from training and evaluating them at inference. We remove three pairs from \texttt{scene} and five from \texttt{puzzle-4x4} from the original datasets. For example, one held-out \texttt{scene} pair requires opening the drawer when the window is closed and unlocked and the drawer is closed. Since the goal can still be achieved through an indirect in-distribution sequence---e.g., lock the window, open the drawer, and unlock the window again (see \cref{fig:ooc_example})---we report a \emph{direct success rate}, which counts only trajectories that solve the intended held-out task directly. 
\cref{tab:ooc_case_study_overall} aggregates results over all held-out pairs. CTA achieves the highest direct success and success rates in both environments, suggesting that baselines often prefer indirect in-distribution executions, whereas CTA directly solves unseen context--task combinations more reliably through analogy transduction. Additional experimental details are provided in \cref{appendix:main_exp_details}.

\subsection{Dual Analogies Encode the Task-Endogenous Displacement}

To qualitatively verify that the dual analogies capture task-endogenous displacements, we sample 20{,}000 state--goal pairs $(s,g)$ from the re-collected validation split of \texttt{scene-play} and visualize dual analogies $\alpha^\vee(s,g)$ using a 2D t-SNE projection~\citep{maaten2008visualizing} (see \cref{fig:nearest_analogies}). We further visualize three representative examples of the nearest analogies in the analogy space, together with their corresponding locations in the t-SNE plot, which cluster by intuitive task semantics (e.g., opening/closing the drawer and placing the cube into the drawer). For example, $\alpha^\vee(s_{7824}, s_{7839})$ and $\alpha^\vee(s_{16977}, s_{16987})$ both correspond to closing the drawer despite differing task-exogenous factors (e.g., window or button states).

\begin{figure}[t]
  \centering
  \includegraphics[width=0.94\columnwidth]{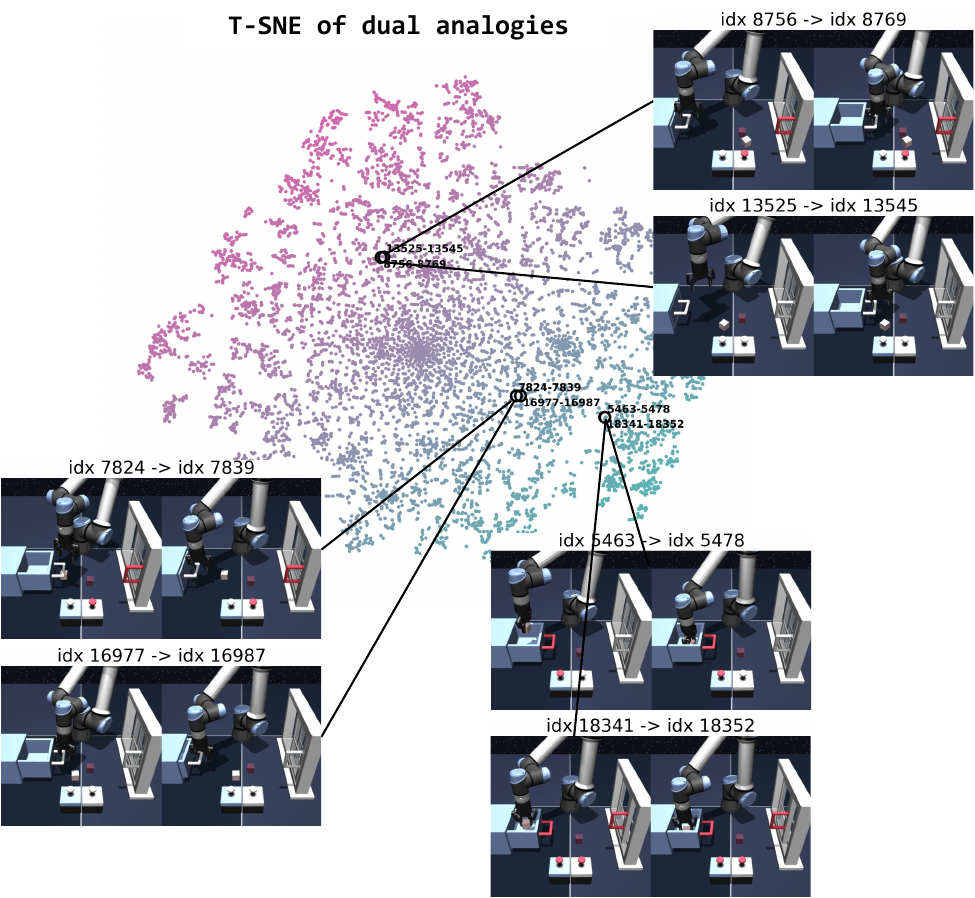}
  \captionsetup{type=figure, skip=10pt, font=small, width=\columnwidth}
  \caption{\textbf{t-SNE visualization of nearest analogies.} Each point represents $\alpha^\vee(s,g)$ for a state--goal pair from the \texttt{scene-play-v0} dataset; we visualize three nearest-neighbor analogy pairs.}
  \label{fig:nearest_analogies}
\end{figure}

\section{Discussion}\label{sec:discussion}
Despite the strong compositional generalization enabled by analogy transduction, dual analogies and CTA have several limitations.
First, \cref{assump:task_block_coordinate_consistency} implies that, as discussed in \cref{subsec:insight}, only task-endogenous components affect each task-endogenous state.
This assumption connects the latent semantics of GCE-BCMPs to an intuitive task--context decomposition, but may be violated in realistic environments. We further discuss the assumption in \cref{appendix:theoretical_results}.
Second, even if the temporal distance difference field is a valid task-endogenous analogy, a gap remains when implementing it through its practical surrogate, the dual analogy.
This gap makes it hard to guarantee that the dual analogy itself is invariant to the task-exogenous context. Our discussion of the universality of the parameterization in \eqref{eq:practical_parameterizaiton} was intended not only to explain why it can serve as a useful practical parameterization of the invariant distance-difference field, but also to motivate this choice in practice, without implying that the learned coordinate vector is itself minimal or identifiable. The gap between theory and practice is shared with prior work~\citep{park2026dual}, and more accurately approximating and implementing the dual field remains an important direction for future work.

CTA with dual analogy is designed to improve compositional generalization between task-endogenous analogies and task-exogenous contexts. Hence, it is most beneficial in environments where task-endogenous and task-exogenous components are intuitively separable and their combinations are diverse, as shown in \cref{subsec:benchmark_result}. Although its benefits may be more limited in environments with little or no task-exogenous variation, such as \texttt{maze} environments, CTA does not underperform other baselines in such settings in our experiments (see \cref{appendix:additional_experiments:additional_benchmark_results}).

\section{Conclusion}
In this paper, we formalize \emph{analogy transduction} by viewing offline GCRL as transductive inference, which predicts specific test instances by exploiting structure in the observed data~\citep{gammerman1998learning}.
In this view, goal-reaching trajectories are synthesized from observed relational patterns across entities, which we call analogies.
To enable analogy transduction, we propose a new analogy representation based on the temporal distance difference field, and theoretically show that it is invariant to task-exogenous variations, captures task-endogenous displacements, and is sufficient for optimal goal-reaching.
We further introduce its practical instantiation, the \emph{dual analogy}, and propose CTA, an offline GCRL method that enables OOC extrapolation through analogy transduction. Although dual analogies and CTA still entail important limitations (see \cref{sec:discussion}), we hope that each will be broadly useful for future work on compositional generalization, sequential decision making, and beyond.
%%%%%%%%%%%%%%%%%%%%%%%%%%%%%%%%%%%%%%%%%%%%%%%%%%%%%%%%%%%%%%%%%%%%%%%%%%%%%%%%

\section*{Acknowledgements}
This work was partly supported by the Institute of Information \& Communications Technology Planning \& Evaluation (IITP) grant funded by the Korea government (MSIT) (No. 2022-0-00480, Development of Training and Inference Methods for Goal-Oriented Artificial Intelligent Agents, 50\%, and No. 2019-0-01190, (SW Star Lab) Robot Learning: Efficient, Safe, and Socially-Acceptable Machine Learning, 50\%).

\section*{Impact Statements}
This paper presents work whose goal is to advance the field of Machine Learning. There are many potential societal consequences of our work, none of which we feel must be specifically highlighted here.


\bibliography{mybib}
\bibliographystyle{icml2026}

%%%%%%%%%%%%%%%%%%%%%%%%%%%%%%%%%%%%%%%%%%%%%%%%%%%%%%%%%%%%%%%%%%%%%%%%%%%%%%%
%%%%%%%%%%%%%%%%%%%%%%%%%%%%%%%%%%%%%%%%%%%%%%%%%%%%%%%%%%%%%%%%%%%%%%%%%%%%%%%
% APPENDIX
%%%%%%%%%%%%%%%%%%%%%%%%%%%%%%%%%%%%%%%%%%%%%%%%%%%%%%%%%%%%%%%%%%%%%%%%%%%%%%%
%%%%%%%%%%%%%%%%%%%%%%%%%%%%%%%%%%%%%%%%%%%%%%%%%%%%%%%%%%%%%%%%%%%%%%%%%%%%%%%
\newpage
\appendix
\onecolumn
\crefalias{section}{appendix}
\crefname{appendix}{Appendix}{Appendices}
\Crefname{appendix}{Appendix}{Appendices}
\crefalias{subsection}{appendix}
\crefname{subsection}{Appendix}{Appendices}
\Crefname{subsection}{Appendix}{Appendices}

\newcommand{\sg}{\operatorname{sg}}

\section{Extended Related Work}
\paragraph{Compositional generalization in sequential decision making.}
In sequential decision making, compositional generalization is most commonly studied through trajectory stitching, which synthesizes new trajectories by connecting segments from different demonstrations. Trajectory stitching can emerge implicitly through dynamic programming in value or metric learning~\citep{mnih2013playing, haarnoja2018soft, fujimoto2021minimalist, kumar2020conservative, kostrikov2022offline, chen2023offline, park2023hiql, fujimoto2023sale, zheng2024contrastive}, or be enforced explicitly through architectural choices such as model-based~\citep{zhuang2024reinformer, wu2023elastic, zhou2023free} and sequence-modeling approaches~\citep{janner2022planning, kim2024stitching, li2024diffstitch, luo2025generative}. 
In this paper, we study analogy transduction as a new axis of compositional generalization, where task-endogenous analogies are transplanted across contexts beyond trajectory stitching.

\paragraph{Metric learning for sequential decision making.}
Metric learning for sequential decision making is also a well-established approach.
%The learned metrics quantify specific relations among states, actions and rewards, such as task similarity and the expected remaining steps toward a goal.
\citet{zhang2021learning,hansen2022bisimulation,wang2023efficient,calo2024bisimulation} learn bisimulation metrics to capture the behavioral similarity between state abstractions, and \citet{gleave2021quantifying, wulfe2022dynamics, skalse2023starc} develop metrics over canonicalized rewards that measure how task goals differ, while remaining invariant to reward shaping. Notably, temporal distance between states can be used directly as a goal-conditioned value function~\citep{wang2023optimal} and is consequently a central object of study in the GCRL paradigm. Temporal distances can be learned by globally separating states while locally aligning temporal distance in a latent space~\citep{wang2022improved, liu2023metric, wang2023optimal, lee2025temporal}. Complementarily, \citet{eysenbach2022contrastive, zheng2024contrastive, myers2024learning} learn the probability of attaining the specified goals under the discounted state occupancy measure through contrastive learning objectives. Our approach extracts shared task semantics in a metric space, leveraging temporal distance learning.

\section{Theoretical Analysis}\label{appendix:theoretical_results}
For clarity, we present our theoretical development in the discrete setting, assuming that the state space $\mathcal S$ and the action space $\mathcal A$ are discrete.
The corresponding continuous-space statements can be derived via the usual extensions (e.g., replacing sums by integrals and maxima by suprema) under standard measurability and regularity conditions, and we omit these technical details.

We begin by recalling the definition of task-endogenous analogy.
\begin{definition}[Task-endogenous analogy]\label{appendix:def:task_endogenous_analogy}
Given a GCE-BCMP, let $\delta\mathcal Z^{\mathrm{en}}$ be a displacement space and let
$\delta:\mathcal Z^{\mathrm{en}}\times\mathcal Z^{\mathrm{en}}\to \delta\mathcal Z^{\mathrm{en}}$ be a well-defined displacement mapping.
A mapping $\mathfrak a:\mathcal S\times\mathcal S\to \delta\mathcal Z^{\mathrm{en}}$ is called a \textbf{\emph{task-endogenous analogy}} if it satisfies the following two conditions:

(\emph{Task-endogenous displacement})
\ For all $(s,g)\in\mathrm{supp}(f^e)$,
\[
\mathfrak{a}(s,g)=\delta\big(z_{s\mid g}^{\mathrm{en}},\,z_{g\mid s}^{\mathrm{en}}\big).
\]
(\emph{Sufficient for optimal goal-reaching})
\ There exists a deterministic policy $\tilde\pi\!:\!\mathcal S\times \delta\mathcal Z^{\mathrm{en}}\!\to\!\mathcal A$ such that, for all $(s,g)\in\mathrm{supp}(f^e)$,
\[
V^{\tilde\pi}(s,g)=V^*(s,g).
\]
Equivalently, the optimal action for $(s,g)$ can be inferred from $(s,\mathfrak a(s,g))$.
\end{definition}

Given a GCE-BCMP, we can first derive that the temporal distance depends only on the task-endogenous states.
\begin{lemma}[Endogenous Bellman closure]
\label{lemma:endogenous_bellman_closure}
Given a GCE-BCMP with the modified reward
$r^\ell(s,g)=\mathbf 1\{z^{\mathrm{en}}_{s\mid g}=z^{\mathrm{en}}_{g\mid s}\},$
there exists an optimal endogenous value function
$V^*_{\mathrm{en}}:\mathcal Z^{\mathrm{en}}\times
\mathcal Z^{\mathrm{en}}\to\mathbb R$ such that, for all
$(s,g)\in\mathrm{supp}(f^e)$,
\[
V^*(s,g)
=
V^*_{\mathrm{en}}(\bar z^{\mathrm{en}}_{(s,g)}).
\]
Consequently, whenever the temporal distance is finite,
\[
d^*(s,g)
=
D^*_{\mathrm{en}}(\bar z^{\mathrm{en}}_{(s,g)}),
\qquad
D^*_{\mathrm{en}}(\bar z)
:=
\log_\gamma V^*_{\mathrm{en}}(\bar z).
\]
\end{lemma}

\begin{proof}
Define the endogenous reward
\[
r^{\mathrm{en}}(z_1,z_2)
:=
\mathbf 1\{z_1=z_2\}.
\]
By construction,
\[
r^\ell(s,g)
=
r^{\mathrm{en}}(\bar z^{\mathrm{en}}_{(s,g)}).
\]
For any function
$U:\mathcal Z^{\mathrm{en}}\times\mathcal Z^{\mathrm{en}}\to\mathbb R$,
define the endogenous Bellman operator
\[
(\mathcal T^{\mathrm{en}}U)(\bar z)
=
r^{\mathrm{en}}(\bar z)
+
(1-r^{\mathrm{en}}(\bar z))\gamma
\max_{a\in\mathcal A}
\mathbb E_{\bar z'^{\mathrm{en}}\sim
\mathcal P^{\mathrm{en}}(\cdot\mid \bar z,a)}
\left[
U(\bar z'^{\mathrm{en}})
\right].
\]
The factor $(1-r^{\mathrm{en}}(\bar z))$ reflects the one-time
goal-reaching reward convention.

Now lift $U$ to the observation space by
\[
\widetilde U(s,g)
:=
U(\bar z^{\mathrm{en}}_{(s,g)}).
\]
Using the task-endogenous abstraction in the GCE-BCMP, the marginal
transition of $\bar{\mathbf z}'^{\mathrm{en}}$ under any action $a$ depends
only on $\bar z^{\mathrm{en}}_{(s,g)}$ and $a$, and is independent of
$\bar z^{\mathrm{ex}}_{(s,g)}$. Therefore, for all
$(s,g)\in\mathrm{supp}(f^e)$,
\[
(\mathcal T\widetilde U)(s,g)
=
(\mathcal T^{\mathrm{en}}U)(\bar z^{\mathrm{en}}_{(s,g)}),
\]
where $\mathcal T$ is the optimal Bellman operator on the observation space
under $r^\ell$.

Thus, the Bellman operator maps lifted endogenous functions to lifted
endogenous functions. Since the discounted optimal Bellman operator has a
unique fixed point, the optimal value function must be of the lifted form:
\[
V^*(s,g)
=
V^*_{\mathrm{en}}(\bar z^{\mathrm{en}}_{(s,g)}),
\]
where $V^*_{\mathrm{en}}$ is the unique fixed point of
$\mathcal T^{\mathrm{en}}$. Taking $\log_\gamma$ on both sides gives
\[
d^*(s,g)
=
D^*_{\mathrm{en}}(\bar z^{\mathrm{en}}_{(s,g)}).
\]
\end{proof}

Within GCE-BCMP, we additionally make the following assumptions.

\begin{assumption}
\label{appendix:assump:probe_pair_support_completeness}
For every $(s,g)\in\mathrm{supp}(f^e)$, all probe pairs needed to
evaluate the temporal distance difference field are also in the support, \textit{i.e.,}
$(x,s),(x,g)\in\mathrm{supp}(f^e),\forall x\in\mathcal S.$
\end{assumption}
\begin{assumption}
\label{appendix:assump:finite_temporal_distance}
For every $(s,g)\in\mathrm{supp}(f^e)$ and every $x\in\mathcal S$, both $s$ and $g$ are reachable from $x$ in finite time, \textit{i.e.,} $d^*(x,s)<\infty$ and $d^*(x,g)<\infty$.
\end{assumption}

\begin{assumption}[Task-block coordinate consistency]
\label{appendix:assump:task_block_coordinate_consistency}
For each task block
$\mathcal B_{\bar z}$ with
$\bar z=(z_1,z_2)\in
\mathcal Z^{\mathrm{en}}\times\mathcal Z^{\mathrm{en}}$,
there exists a map
$\rho_{\bar z}:\mathcal S\to\mathcal Z^{\mathrm{en}}$
such that, for every $(s,g)\in\mathcal B_{\bar z}$ and every
probe state $x\in\mathcal S$ for which
$(x,s),(x,g)\in\mathrm{supp}(f^e)$,
\[z^{\mathrm{en}}_{x\mid s}=z^{\mathrm{en}}_{x\mid g}=
\rho_{\bar z}(x),
\qquad
z^{\mathrm{en}}_{s\mid x}
=
z^{\mathrm{en}}_{s\mid g}
=
z_1,
\qquad
z^{\mathrm{en}}_{g\mid x}
=
z^{\mathrm{en}}_{g\mid s}
=
z_2.
\]
\end{assumption}

Intuitively, Assumption~\ref{appendix:assump:task_block_coordinate_consistency} requires the
states involved in one task block to be read in a consistent
task-endogenous coordinate system. For example, consider a drawer-opening
task, where
\[
s=(\text{window closed},\,\text{drawer closed}),\qquad
g=(\text{window closed},\,\text{drawer open}).
\]
Here, the robot and drawer states are task-endogenous, while the window
state is task-exogenous. If a probe state is
\[
x=(\text{window open},\,\text{drawer half-open}),
\]
then the equality
\[
z^{\mathrm{en}}_{x\mid s}=z^{\mathrm{en}}_{x\mid g}
\]
means that $x$ is read as the same robot--drawer state, i.e., drawer
half-open, whether it is compared with $s$ or $g$. Likewise,
\[
z^{\mathrm{en}}_{s\mid x}=z^{\mathrm{en}}_{s\mid g}
\]
means that $s$ is consistently read as drawer closed, and
\[
z^{\mathrm{en}}_{g\mid x}=z^{\mathrm{en}}_{g\mid s}
\]
means that $g$ is consistently read as drawer open. Thus,
$d^*(x,g)-d^*(x,s)$ compares drawer-open and drawer-closed endpoints in the
same robot--drawer geometry, independent of the window context. 

We now present a lemma proving the task-endogenous displacement condition in \cref{appendix:def:task_endogenous_analogy}.
\begin{lemma}[Temporal distance difference field encodes the task-endogenous displacement]
\label{lemma:alpha_task_endogenous_displacement}
Under~\cref{appendix:assump:probe_pair_support_completeness,appendix:assump:finite_temporal_distance,appendix:assump:task_block_coordinate_consistency}, the
temporal distance difference field
\[
\alpha(s,g)(x)
=
d^*(x,g)-d^*(x,s)
\]
satisfies the task-endogenous displacement condition. That is, there exists
a displacement space $\delta\mathcal Z^{\mathrm{en}}$ and a mapping
$\delta:\mathcal Z^{\mathrm{en}}\times\mathcal Z^{\mathrm{en}}
\to\delta\mathcal Z^{\mathrm{en}}$ such that, for all
$(s,g)\in\mathrm{supp}(f^e)$,
\[
\alpha(s,g)
=
\delta(z^{\mathrm{en}}_{s\mid g},z^{\mathrm{en}}_{g\mid s}).
\]
\end{lemma}

\begin{proof}
Set $\delta\mathcal Z^{\mathrm{en}}
:=
(\mathcal S\to\mathbb R).$
Fix any task block
$\mathcal B_{\bar z}$ with
$\bar z=(z_1,z_2)$, and let $(s,g)\in\mathcal B_{\bar z}$.
Then
\[
z^{\mathrm{en}}_{s\mid g}=z_1,
\qquad
z^{\mathrm{en}}_{g\mid s}=z_2.
\]
By Lemma~\ref{lemma:endogenous_bellman_closure}, for any probe state $x$
such that $(x,s),(x,g)\in\mathrm{supp}(f^e)$,
\[
d^*(x,g)
=
D^*_{\mathrm{en}}
(z^{\mathrm{en}}_{x\mid g},z^{\mathrm{en}}_{g\mid x}),
\]
and
\[
d^*(x,s)
=
D^*_{\mathrm{en}}
(z^{\mathrm{en}}_{x\mid s},z^{\mathrm{en}}_{s\mid x}).
\]
By Assumption~\ref{appendix:assump:task_block_coordinate_consistency},
\[
z^{\mathrm{en}}_{x\mid g}
=
z^{\mathrm{en}}_{x\mid s}
=
\rho_{\bar z}(x),
\]
\[
z^{\mathrm{en}}_{g\mid x}
=
z^{\mathrm{en}}_{g\mid s}
=
z_2,
\qquad
z^{\mathrm{en}}_{s\mid x}
=
z^{\mathrm{en}}_{s\mid g}
=
z_1.
\]
Hence,
\[
d^*(x,g)
=
D^*_{\mathrm{en}}(\rho_{\bar z}(x),z_2),
\qquad
d^*(x,s)
=
D^*_{\mathrm{en}}(\rho_{\bar z}(x),z_1).
\]
Therefore,
\[
\alpha(s,g)(x)
=
D^*_{\mathrm{en}}(\rho_{\bar z}(x),z_2)
-
D^*_{\mathrm{en}}(\rho_{\bar z}(x),z_1).
\]

Now define
$\delta:\mathcal Z^{\mathrm{en}}\times\mathcal Z^{\mathrm{en}}
\to(\mathcal S\to\mathbb R)$ by
\[
\delta(z_1,z_2)(x)
:=
D^*_{\mathrm{en}}(\rho_{(z_1,z_2)}(x),z_2)
-
D^*_{\mathrm{en}}(\rho_{(z_1,z_2)}(x),z_1).
\]
Then, for every $(s,g)\in\mathcal B_{\bar z}$,
\[
\alpha(s,g)(x)
=
\delta(z^{\mathrm{en}}_{s\mid g},
z^{\mathrm{en}}_{g\mid s})(x)
\]
for all relevant $x\in\mathcal S$. Hence,
\[
\alpha(s,g)
=
\delta(z^{\mathrm{en}}_{s\mid g},
z^{\mathrm{en}}_{g\mid s}),
\]
which proves the task-endogenous displacement condition.
\end{proof}

We additionally prove the optimal goal-reaching sufficiency condition in
\cref{appendix:def:task_endogenous_analogy}.

\begin{lemma}[Sufficiency of the temporal distance difference field]
\label{lemma:sufficiency_temporal_distance_difference}
Given a GCE-BCMP, let $\alpha(s,g):\mathcal S\to\mathbb R$ be the temporal
distance difference field in \eqref{appendix:eq:global_analogy}. Assume that the
relevant temporal distances are finite and that the maximizers below exist.
Then there exists a deterministic mapping
$\tilde\pi:\mathcal S\times(\mathcal S\to\mathbb R)\to\mathcal A$
such that, when evaluated with the field $\alpha(\cdot,g)$, it satisfies
for all $(s,g)\in\mathrm{supp}(f^e)$,
\[
V^{\tilde\pi}(s,g)=V^*(s,g).
\]
Equivalently, the optimal action for $(s,g)$ can be inferred from
$(s,\alpha(s,g))$.
\end{lemma}

\begin{proof}
Let $P_{\mathcal S}(\cdot\mid s,a)$ denote the induced state-transition
kernel over $\mathcal S$ after applying action $a$ at state $s$. Fix a
state--goal pair $(s,g)\in\mathrm{supp}(f^e)$.
For any candidate next state $s'\in\mathcal S$, by the definition
\[
\alpha(s,g)(x)=d^*(x,g)-d^*(x,s)
\]
and $d^*(x,g)=\log_\gamma V^*(x,g)$, we have
\[
V^*(s',g)
=
\gamma^{d^*(s',g)}
=
\gamma^{\alpha(s,g)(s')+d^*(s',s)}
=
\gamma^{\alpha(s,g)(s')}\gamma^{d^*(s',s)}.
\]

For nonterminal $(s,g)$, the Bellman optimality
equation gives
\[
V^*(s,g)
=
\gamma
\max_{a\in\mathcal A}
\mathbb E_{s'\sim P_{\mathcal S}(\cdot\mid s,a)}
\left[V^*(s',g)\right].
\]
Therefore, an optimal action can be chosen as
\[
\pi^*(s,g)
\in
\arg\max_{a\in\mathcal A}
\mathbb E_{s'\sim P_{\mathcal S}(\cdot\mid s,a)}
\left[
\gamma^{\alpha(s,g)(s')}\gamma^{d^*(s',s)}
\right].
\]
For terminal $(s,g)$, any action is optimal under the one-time
goal-reaching reward convention.

Now define, for any field $F:\mathcal S\to\mathbb R$,
\[
\tilde\pi(s,F)
\in
\arg\max_{a\in\mathcal A}
\mathbb E_{s'\sim P_{\mathcal S}(\cdot\mid s,a)}
\left[
\gamma^{F(s')}\gamma^{d^*(s',s)}
\right].
\]
Substituting
$F=\alpha(s,g)$ gives
\[
\tilde\pi(s,\alpha(s,g))
\in
\arg\max_{a\in\mathcal A}
\mathbb E_{s'\sim P_{\mathcal S}(\cdot\mid s,a)}
\left[V^*(s',g)\right].
\]
Hence, $\tilde\pi(s,\alpha(s,g))$ is Bellman-greedy with respect to
$V^*(\cdot,g)$ at every nonterminal state $s$. By the standard optimality
theorem for discounted Markov decision processes,
\[
V^{\tilde\pi}(s,g)=V^*(s,g)
\]
for all $(s,g)\in\mathrm{supp}(f^e)$, where $V^{\tilde\pi}$ denotes the
value induced by $\tilde\pi$.
\end{proof}

Finally, we can conclude that the temporal distance difference field is a task-endogenous analogy.
\begin{proposition}[Temporal distance difference field is a task-endogenous analogy]
\label{appendix:prop:temporal_distance_difference_is_analogy}
Given a GCE-BCMP, let a temporal distance difference function $\alpha:\mathcal{S}\times\mathcal{S}\to(\mathcal{S}\to\mathbb{R})$ be defined as
\begin{equation}\label{appendix:eq:global_analogy}
    \alpha(s,g)(x) = d^*(x,g) - d^*(x,s),
\end{equation}
$\forall s,g,x\in\mathcal S.$ Under \cref{appendix:assump:probe_pair_support_completeness,appendix:assump:finite_temporal_distance,appendix:assump:task_block_coordinate_consistency}, the field $\alpha(s,g):\mathcal S\!\to\!\mathbb R$ is a task-endogenous analogy. That is, $\alpha(s,g)$ encodes the task-endogenous displacement and is sufficient for optimal goal-reaching.
\end{proposition}
\begin{proof}
By \cref{lemma:alpha_task_endogenous_displacement}, $\alpha$ satisfies the task-endogenous displacement condition in \cref{appendix:def:task_endogenous_analogy}. Moreover, by \cref{lemma:sufficiency_temporal_distance_difference}, $\alpha$ satisfies the sufficiency condition in \cref{appendix:def:task_endogenous_analogy}. Therefore, $\alpha$ is a task-endogenous analogy.
\end{proof}

\newpage
\section{Extended Preliminaries} \label{appendix:extended_preliminaries}

\subsection{Goal-Conditioned Bisimulation Metric}\label{appendix:extended_preliminaries:gcb}
In this section, we review the goal-conditioned bisimulation metric~\citep{hansen2022bisimulation}, which is most closely related to our formulation, and explain why it is not directly applicable to multiple goal-reaching environments in offline GCRL.

\paragraph{Bisimulation.}
Bisimulation offers a criterion for state abstraction by grouping states that are ``behaviorally equivalent"~\citep{li2006towards}. Two states $s_i$ and $s_j$ are considered bisimilar if they produce identical immediate rewards and the same probability distribution over the next group of bisimilar states~\citep{larsen1989bisimulation, givan2003equivalence} for all possible actions.

\begin{definition}[Bisimulation Relations~\citep{givan2003equivalence}]\label{def:bisimulation}
Given an MDP $M$, an equivalence relation $B$ over the state space $S$ is a \emph{bisimulation relation} if, for all states $s_i, s_j \in S$ that are equivalent under $B$ (denoted $s_i \equiv_B s_j$), the following conditions hold:
\begin{equation}
\begin{alignedat}{3}
R(s_i,a)       &{}= R(s_j,a)       &\qquad& \forall a &\in& \mathcal{A}, \\
P(G\mid s_i,a) &{}= P(G\mid s_j,a) &\qquad& \forall a &\in& \mathcal{A},  \forall G \in \mathcal{S}_B,
\end{alignedat}
\end{equation}
where $\mathcal{S}_B$ is the partition of $S$ induced by $B$ (the set of equivalence classes), and $P(G \mid s, a)  =  \sum_{s' \in G} P(s' \mid s, a).$
\end{definition}

\paragraph{Bisimulation metric and goal-conditioned bisimulation metric.}
For practical representation learning with continuous or high-dimensional state spaces to capture bisimilar relations, a bisimulation metric~\citep{ferns2011bisimulation, ferns2014bisimulation, castro2020scalable, zhang2021learning, calo2024bisimulation} is defined with a pseudometric space $(\mathcal{S}, d_\mathrm{bisim})$ where the distance function $d_\mathrm{bisim}:\mathcal{S}\times \mathcal{S}\to \mathbb{R}_{\geq0}$ on $\mathcal{S}$ refers to the ``behavioral similarity" between two states.
Our work is motivated by the goal-conditioned  bisimulation (GCB) metric~\citep{hansen2022bisimulation}: 
\begin{equation} \label{eq:gcbisim_def}
\begin{split}
    d^\pi_\mathrm{bisim}((\mathbf{s}_i, \mathbf{g}_i), (\mathbf{s}_j, \mathbf{g}_j)) & = 
    | \mathcal{R}(\mathbf{s}_i, \pi(\mathbf{s}_i,\mathbf{g}_i), \mathbf{g}_i) - \mathcal{R}(\mathbf{s}_j, \pi(\mathbf{s}_j,\mathbf{g}_j), \mathbf{g}_j) | \\
    & + \gamma \mathcal{W}_1(d^\pi_\mathrm{bisim})(\mathcal{P}(\mathbf{s}'_i|\mathbf{s}_i, \pi(\mathbf{s}_i, \mathbf{g}_i)), \mathcal{P}(\mathbf{s}'_j|\mathbf{s}_j, \pi(\mathbf{s}_j, \mathbf{g}_j))),
\end{split}
\end{equation}
which is an on-policy, goal-conditioned variant of $d_\mathrm{bisim}$, where $(\mathbf{s}_i, \mathbf{g}_i), (\mathbf{s}_j, \mathbf{g}_j) \in \mathcal{S} \times\mathcal{S}$ are state--goal pairs, $\pi$ is the deterministic goal-conditioned policy, and $\mathcal{W}_1$ is the 1-Wasserstein distance metric~\citep{van2001towards}.
%In deterministic settings, the distributions $\mathcal{P}$ are Dirac measures.
As a smaller $d^\pi_\mathrm{bisim}$ indicates greater behavioral similarity, \citet{hansen2022bisimulation} train a goal-conditioned analogy encoder $\psi:\mathcal{S}\times\mathcal{S}\to\mathbb{R}^d$ to place the analogies from such pairs close in the latent space by minimizing the GCB objective:
\begin{equation} \label{eq:org_gcbisim}
\begin{split}
    \mathcal{J}_{gcb}(\psi) = \mathbb{E}&_{(s_i,a_i,s_i',g_i),(s_j,a_j,s_j',g_j)}\Bigl[\Bigl( \| \psi(s_i, g_i)-\psi(s_j,g_j)\|_1 \\
    & - |\mathcal{R}(s_i,a_i,g_i)-\mathcal{R}(s_j,a_j,g_j)|-\gamma\|\bar\psi(s_i',g_i)-\bar\psi(s_j',g_j)\|_2 \Bigr)^2\Bigr],
\end{split}
\end{equation}
where $\bar\cdot$ denotes a stop-gradient. As a result, the $\ell_1$ norm of the difference between analogy vectors captures $d^\pi_\mathrm{bisim}$.

\paragraph{Drawbacks of the bisimulation families.}
Despite providing a principled notion of behavioral similarity, GCB objectives are not directly suitable for offline generalist goal-reaching.
First, the induced metric is fundamentally \emph{on-policy} through $\pi(s,g)$ in \cref{eq:gcbisim_def}, so its notion of equivalence changes with a certain goal-conditioned policy and can be unreliable when learned purely from fixed, potentially suboptimal offline data.
Second, the GCB loss in \cref{eq:org_gcbisim} relies on a bootstrapped next-state term $\|\bar\psi(s_i',g_i)-\bar\psi(s_j',g_j)\|_2$, which can be noisy under dataset shift and exacerbate representation collapse.
In particular, in sparse-reward goal-reaching settings, bisimulation objectives are prone to over-abstraction, merging states with identical immediate rewards and hindering fine-grained goal discrimination, which ultimately limits compositional generalization across many goals and contexts.

Unlike goal-conditioned bisimulation objectives, our temporal distance difference analogy formulation in \cref{appendix:eq:global_analogy} does not inherit the drawbacks.
First, our formulation is grounded in the \emph{optimal} temporal distance $d^*$ rather than an on-policy quantity $d^\pi$ that depends on a particular goal-conditioned policy $\pi(s,g)$.
As a result, it does not require defining behavioral equivalence with respect to an unknown and potentially suboptimal behavior policy in the offline dataset.
Second, it does not rely on reward matching, and therefore avoids the over-abstraction pathology in sparse-reward goal-reaching where most states appear indistinguishable until reward is observed.
Third, dual analogies are not trained via bootstrapping, which mitigates representation collapse driven by noisy bootstrapping under distribution shift. We empirically compare our proposed dual analogy with GCB analogy in~\cref{appendix:additional_experiments:additional_benchmark_results}.

\medskip

\subsection{Exogenous Block Controlled Markov Process (EX-BCMP)} \label{appendix:extended_preliminaries:exbcmp}
In this section, we provide an overview of the block CMP, which is widely used to model rich observations with deterministic latent-state recovery~\citep{du2019provably, zhang2021learning, efroni2022provably, park2026dual}. We also review the exogenous block CMP, which is similar to our GCE-BCMP in that it assumes a decomposition of the latent space into endogenous and exogenous components.

\paragraph{Block CMP.}
A \emph{block CMP} (BCMP)~\citep{du2019provably} is defined as a tuple $(\mathcal{S}, \mathcal{Z}, \mathcal{A}, \mathcal{P}, f^e)$,
where $\mathcal{S}$ is an observation space, $\mathcal{Z}$ is a latent state space, $\mathcal{A}$ is an action space,
$\mathcal{P}:\mathcal{Z}\times\mathcal{A}\to\Delta(\mathcal{Z})$ is a latent transition dynamics, and
$f^e:\mathcal{Z}\to\Delta(\mathcal{S})$ is an emission function. The BCMP makes the \emph{block assumption}, \textit{i.e.,} the emission distributions corresponding to any two distinct latent states have disjoint supports:
\[
\mathrm{supp}\big(f^e(\cdot\mid z_i)\big)\cap \mathrm{supp}\big(f^e(\cdot\mid z_j)\big)=\emptyset
\quad \forall\, z_i\neq z_j.
\]
Each $\mathrm{supp}\big(f^e(\cdot\mid z))$ is referred to as a \emph{block}. The block assumption guarantees the existence of a deterministic mapping
$f^\ell:\mathcal{S}\to\mathcal{Z}$ satisfying $f^\ell(s)=z$ for all
$s\sim f^e(\cdot\mid z)$.

\paragraph{Exogenous block CMP.}
An \emph{exogenous block CMP} (ExBCMP)~\citep{efroni2022provably} is a BCMP that additionally assumes a product structure on the latent space and a corresponding decoupling of initialization and dynamics. Formally, the latent state space decomposes as $\mathcal{Z}=\mathcal{Z}_{\mathrm{en}}\times\mathcal{Z}_{\mathrm{ex}}$ with $z=(z_{\mathrm{en}},z_{\mathrm{ex}})$, and there exist initial distributions $\mu_{\mathrm{en}}\in\Delta(\mathcal{Z}_{\mathrm{en}})$, $\mu_{\mathrm{ex}}\in\Delta(\mathcal{Z}_{\mathrm{ex}})$ and latent transition dynamics $\mathcal{P}_{\mathrm{en}}:\mathcal{Z}_{\mathrm{en}}\times\mathcal{A}\to\Delta(\mathcal{Z}_{\mathrm{en}})$, $\mathcal{P}_{\mathrm{ex}}:\mathcal{Z}_{\mathrm{ex}}\to\Delta(\mathcal{Z}_{\mathrm{ex}})$ such that
\begin{equation}\label{eq:exbcmp_decomposition}
\begin{split}
\mu(z) = \mu_{\mathrm{en}}(&z_{\mathrm{en}})\,\mu_{\mathrm{ex}}(z_{\mathrm{ex}})\\
\mathcal{P}\big(z'\mid z,a\big)
= \mathcal{P}_{\mathrm{en}}(z_{\mathrm{en}}'&\mid z_{\mathrm{en}},a)\,
   \mathcal{P}_{\mathrm{ex}}(z_{\mathrm{ex}}'\mid z_{\mathrm{ex}}).
\end{split}
\end{equation}
The endogenous state $z_{\mathrm{en}}$ captures the part of the latent dynamics affected by the agent through actions, whereas the exogenous state $z_{\mathrm{ex}}$ represents the nuisances not affected by the action. Combined with the block identifiability, this decomposition correspondingly implies the existence of deterministic mappings $f^\ell_{\mathrm{en}}:\mathcal{S}\to\mathcal{Z}_{\mathrm{en}}$ and $f^\ell_{\mathrm{ex}}:\mathcal{S}\to\mathcal{Z}_{\mathrm{ex}}$.

Intuitively, in a BCMP, observations within the same block share a recoverable latent factor \(z\) that is invariant to within-block variations and non-overlapping across blocks. Prior work typically interprets this shared factor as capturing observation-dependent components, such as nuisance variables or observation noise, and therefore adopts a single, state-dependent decomposition in which the endogenous and exogenous parts of a state are assumed to be globally consistent~\citep{du2019provably, efroni2022provably, levine2025learning, park2026dual}. In contrast, we extend (Ex-)BCMPs to the goal-conditioned setting by defining task-endogenous states and task-exogenous contexts relative to the state--goal pair \((s,g)\). This goal-augmented decomposition is inherently task-dependent: even for the same current state \(s\), changing the goal \(g\) can induce a different partition, which in turn enables compositional generalization through transferable analogies (see~\cref{subsec:GCE-BCMP}).

\medskip

\subsection{Bilinear Transduction}\label{appendix:extended_preliminaries:bt}
In this section, we provide an extended overview of \emph{bilinear transduction}~\citep{netanyahu2023learning}, a transductive scheme for out-of-support (OOS) prediction that reduces extrapolation to an out-of-combination (OOC) generalization problem.

Let $\mathcal{X}$ be an input space and let $\mathcal{Y}\subseteq \mathbb{R}^B$ be a $B$-dimensional target space.
We aim to learn a predictor $\omega_\theta:\mathcal{X}\to\mathcal{Y}$ that approximates an unknown ground-truth mapping $\omega^*:\mathcal{X}\to\mathcal{Y}$ under a loss $\ell(\cdot,\cdot)$.
For a distribution $P\in\Delta(\mathcal{X})$, define the population risk
\begin{equation}
\mathcal{R}(\omega_\theta; P)
~:=~
\mathbb{E}_{x\sim P}\big[\ell\big(\omega_\theta(x), \omega^*(x)\big)\big].
\end{equation}
In OOS extrapolation, the test distribution $P_{\mathrm{test}}$ may place mass on regions outside the support of the training distribution $P_{\mathrm{train}}$.

Bilinear transduction assumes that $\mathcal{X}$ admits a group-like structure equipped with
(i) a {displacement mapping} $\delta:\mathcal{X}\times\mathcal{X}\to \delta\mathcal{X}$ into an associated displacement space $\delta\mathcal{X}$,
and (ii) an {apply operator} $\odot:\mathcal{X}\times \delta\mathcal{X}\to \mathcal{X}$ such that, for any $(x,\hat{x})\in\mathcal{X}\times\mathcal{X}$,
the displacement element $d=\delta(x,\hat{x})\in\delta\mathcal{X}$ is the unique element satisfying
$\hat{x}\odot d ~=~ x.$
Given a query $x\in\mathcal{X}$, choose an \emph{anchor} $\hat{x}\in\mathcal{X}$ and rewrite the prediction as
\begin{equation}
\omega_\theta(x)
~:=~
\widehat{\omega}_\theta\big(\hat{x}, \delta(x,\hat{x})\big),
\label{eq:transductive-reparam}
\end{equation}
where $\widehat{\omega}_\theta:\mathcal{X}\times \delta\mathcal{X}\to\mathcal{Y}$ is a deterministic transductive predictor.
Intuitively, $\widehat{\omega}_\theta$ is trained to predict from \emph{(anchor, displacement)} rather than from the raw query itself.

Let $\mathcal{D}_{\mathrm{train}}=\{x_i\}_{i=1}^n\subset\mathcal{X}$ be the training set and define the set of seen displacements
\begin{equation}
\delta\mathcal{X}_{\mathrm{train}}
~:=~
\big\{\delta(x_i,x_j): x_i,x_j\in \mathcal{D}_{\mathrm{train}}\big\}
~\subseteq~ \delta\mathcal{X}.
\end{equation}
For an OOC query $(\hat{x},d)\in\mathcal{X}\times\delta\mathcal{X}$ at test time, bilinear transduction considers anchors $\hat{x}\in\mathcal{D}_{\mathrm{train}}$
and displacements $d$ that lie in the seen set $\delta\mathcal{X}_{\mathrm{train}}$
(e.g., $\mathrm{dist}(d,\delta\mathcal{X}_{\mathrm{train}})\le \rho$ under a metric on $\delta\mathcal{X}$).
Although both factors $\hat{x}$ and $d$ are individually in-support, their pairing $(\hat{x},d)$ may be unseen in the training data, constituting an OOC input.
Thus, extrapolation reduces to generalizing over novel anchor--displacement combinations in the product space $\mathcal{X}\times\delta\mathcal{X}$.

The key inductive bias is to parameterize $\widehat{\omega}_\theta$ as \emph{bilinear} in two learned embeddings of the anchor and displacement.
Concretely, for each component $b\in[B]$, let
$f_{\theta,b}:\delta\mathcal{X}\to\mathbb{R}^p$ and $g_{\theta,b}:\mathcal{X}\to\mathbb{R}^p$ be embedding functions.
Bilinear transduction models
\begin{equation}
\widehat{\omega}_{\theta,b}(\hat{x}, d)
~=~
 g_{\theta,b}(\hat{x})^\top f_{\theta,b}(d),
\qquad
\widehat{\omega}_\theta(\hat{x},d)
~=~
\big(\widehat{\omega}_{\theta,1}(\hat{x},d),\ldots,\widehat{\omega}_{\theta,B}(\hat{x},d)\big),
\label{eq:bilinear-form}
\end{equation}
where $p$ controls the effective rank of the transductive representation.
While the prediction is bilinear in $(f_{\theta,b}, g_{\theta,b})$, the embeddings themselves may be arbitrary function approximators.

\paragraph{Assumptions for extrapolation.}
Bilinear transduction admits a formal OOC guarantee under three standard conditions.

\begin{assumption}[Bounded combinatorial density ratio] \label{assumption:Bounded_combinatorial_density_ratio}
Let $\overline{P}_{\mathrm{train}}$ and $\overline{P}_{\mathrm{test}}$ denote the induced joint distributions over $(d,\hat{x})\in\delta\mathcal{X}\times\mathcal{X}$
under the training and transduction procedures, respectively.
We assume there exists $\kappa\ge 1$ such that $\overline{P}_{\mathrm{test}}$ has $\kappa$-bounded \emph{combinatorial} density ratio
with respect to $\overline{P}_{\mathrm{train}}$, denoted $\overline{P}_{\mathrm{test}} \ll_{\kappa,\mathrm{comb}} \overline{P}_{\mathrm{train}}$.
Informally, this requires that the training joint distribution sufficiently covers the on-support ``blocks'' needed to identify missing combinations,
up to a bounded multiplicative factor $\kappa$.
\end{assumption}

\begin{assumption}[Bilinearly transducible] \label{assumption:Bilinearly_transducible}
For each $b\in[B]$, there exist functions $f^*_b:\delta\mathcal{X}\to\mathbb{R}^p$ and $g^*_b:\mathcal{X}\to\mathbb{R}^p$ such that, for anchors used in transduction,
\begin{equation}
\omega^*_b(x)
~=~
\widehat{\omega}^*_b\big(\hat{x},\delta(x,\hat{x})\big)
~:=~
g^*_b(\hat{x}) \boldsymbol{\cdot} f^*_b(\delta(x,\hat{x})).
\end{equation}
Moreover, the ground-truth predictions are uniformly bounded:
\begin{equation}
\max_{b\in[B]}~\sup_{\hat{x}\in\mathcal{X},\ d\in\delta\mathcal{X}}
\big|\widehat{\omega}^*_b(\hat{x},d)\big|
~\le~ M
\qquad \text{for some constant } M>0.
\end{equation}
\end{assumption}

\begin{assumption}[Non-degeneracy] \label{assumption:Non-degeneracy}
Under the fully in-support portion of the induced training distribution, the embedding factors are not degenerate:
there exists $\sigma^2>0$ such that, for all $b\in[B]$,
\begin{equation}
\min\Big\{
\sigma_p\big(\mathbb{E}[f^*_b(d) f^*_b(d)^\top]\big),\;
\sigma_p\big(\mathbb{E}[g^*_b(\hat{x}) g^*_b(\hat{x})^\top]\big)
\Big\}
~\ge~ \sigma^2,
\label{eq:nondegeneracy}
\end{equation}
where the expectation is taken over $(d,\hat{x})\sim\overline{P}_{\mathrm{train}}$ (restricted to the fully observed region),
and $\sigma_p(\cdot)$ denotes the smallest singular value.
\end{assumption}

\begin{theorem}[Test risk bound under bilinear transduction \citep{netanyahu2023learning}]
Assume that \cref{assumption:Bounded_combinatorial_density_ratio,assumption:Bilinearly_transducible,assumption:Non-degeneracy}
hold and that $\ell$ is the squared loss.
If the training risk is sufficiently small,
\begin{equation}
\mathcal{R}(\omega_\theta; P_{\mathrm{train}})
~\le~
\frac{\sigma^2}{4\kappa},
\end{equation}
then the test risk under transduction is bounded by
\begin{equation}
\mathcal{R}(\omega_\theta; P_{\mathrm{test}})
~\le~
\mathcal{R}(\omega_\theta; P_{\mathrm{train}})
\cdot
\kappa^2\!\left(1+\frac{64 M^4}{\sigma^4}\right)
~=~
\mathcal{R}(\omega_\theta; P_{\mathrm{train}})
\cdot
\mathrm{poly}\!\left(\kappa, \frac{M}{\sigma}\right).
\label{eq:bt-bound}
\end{equation}
Therefore, bilinear transduction controls the error on OOC anchor--displacement pairs with only a polynomial blow-up,
provided that (i) the OOC regime is induced by anchor selection, (ii) the ground truth admits a low-rank bilinear factorization,
and (iii) the induced training coverage is combinatorially sufficient.
\end{theorem}

\paragraph{Relevance to the analogy--context composition.}
In our setting, bilinear transduction provides a principled mechanism for OOC generalization over two factors:
the \emph{anchor} captures task-exogenous context, while the \emph{displacement} captures the task-endogenous analogy, a well-defined displacement of the task-endogenous components.

Note that the above three assumptions are not overly restrictive in our regime.
First, the anchor is chosen from $\mathcal{D}_{\mathrm{train}}$ by construction, and the displacement is derived from state transitions observed in the offline data, so the individual marginals of $\hat{x}$ and $d$ are naturally in-support, making \cref{assumption:Bounded_combinatorial_density_ratio} plausible in practice.
Second, our value and policy are explicitly parameterized via a low-rank bilinear form between anchor-dependent and displacement-dependent embeddings,
which directly aligns with the bilinear transducibility requirement in \cref{assumption:Bilinearly_transducible}.
Finally, \cref{assumption:Non-degeneracy} corresponds to preventing representation collapse of the anchor or displacement embeddings; this is encouraged by the diversity of contexts in offline datasets
and by standard normalization or regularization used in neural approximation.

Equation~\eqref{eq:bilinear-form} then implements a structured composition rule that enables the value function and the policies in \cref{eq:bilinear-value}, \cref{eq:bilinear-policy-h} and \cref{eq:bilinear-policy-l} extrapolate to unseen analogy--context pairings, turning analogy transduction into a well-posed OOC inference problem under the conditions above.

\newpage
\section{Goal-Conditioned Endogenous Block Controlled Markov Process}\label{appendix:GCE-BCMP}
\begin{definition}[GCE-BCMP] \label{def:GCE-BCMP}
A \emph{goal-conditioned endogenous block controlled Markov process} (GCE-BCMP) is specified by a tuple
$(\bar{\mathcal{S}},\bar{\mathcal{Z}},\mathcal{A},\mathcal{P},f^e)$, where
$\bar{\mathcal{S}}:=\mathcal{S}\times\mathcal{S}$ is the product observation space,
$\bar{\mathcal{Z}}:=\mathcal{Z}\times\mathcal{Z}$ is a product latent state space,
$\mathcal{A}$ is an action space,
$\mathcal{P}:\bar{\mathcal{Z}}\times\mathcal{A}\to\Delta(\bar{\mathcal{Z}})$ is a latent transition dynamics on $\bar{\mathcal{Z}}$,
and $f^e:\bar{\mathcal{Z}}\to\Delta(\bar{\mathcal{S}})$ is an emission function from latent abstractions to distributions over $\bar{\mathcal{S}}$.

We write a state--goal observation pair as $u=(s,g)\in\bar{\mathcal{S}}$. The GCE-BCMP contains the following assumptions.

\medskip
\noindent
\textbf{(\emph{Block assumption})}
The emission distributions corresponding to any two distinct latent states have disjoint supports:
\[
\mathrm{supp}\big(f^e(\cdot\mid \bar z_i)\big)\cap \mathrm{supp}\big(f^e(\cdot\mid \bar z_j)\big)=\emptyset,
\quad \forall\, \bar z_i\neq \bar z_j\in\bar{\mathcal{Z}}.
\]
Let $\mathrm{supp}(f^e):=\bigcup_{\bar z\in\bar{\mathcal Z}}\mathrm{supp}\big(f^e(\cdot\mid \bar z)\big)\subseteq \bar{\mathcal S}$.
This assumption implies the existence of a deterministic decoding function
$f^\ell:\mathrm{supp}(f^e)\to\bar{\mathcal{Z}}$ such that
$f^\ell(u)=\bar z$ for all $u \in \mathrm{supp}\big(f^e(\cdot\mid \bar z)\big)$.
Since $\bar{\mathcal Z}=\mathcal Z\times\mathcal Z$, the decoder induces deterministic families
$\{f^\ell_g:\mathcal S\to\mathcal Z\}_{g\in\mathcal S}$ and
$\{f^\ell_s:\mathcal S\to\mathcal Z\}_{s\in\mathcal S}$ such that for all $(s,g)\in \mathrm{supp}(f^e)$,
\begin{equation}\label{eq:bar_z}
    \bar z=f^\ell(u)=f^\ell(s,g)=\big(f^\ell_g(s),\,f^\ell_s(g)\big):=(z_{s\mid g}, z_{g\mid s}).
\end{equation}
Here each $f^\ell_g$ and $f^\ell_s$ is only used on the relevant domain induced by $\mathrm{supp}(f^e)$:
\[
f^\ell_g:\{s\in\mathcal S:(s,g)\in\mathrm{supp}(f^e)\}\to\mathcal Z,
\qquad
f^\ell_s:\{g\in\mathcal S:(s,g)\in\mathrm{supp}(f^e)\}\to\mathcal Z.
\]

The GCE-BCMP further assumes that each latent state admits an endogenous--exogenous factorization~\citep{efroni2022provably}
as $\mathcal Z=\mathcal Z^{\mathrm{en}}\times\mathcal Z^{\mathrm{ex}}$, and accordingly
$\bar{\mathcal Z}=(\mathcal Z^{\mathrm{en}}\times\mathcal Z^{\mathrm{ex}})\times(\mathcal Z^{\mathrm{en}}\times\mathcal Z^{\mathrm{ex}})$.
Under this factorization, for each $(s,g)\in \mathrm{supp}(f^e)$ we can uniquely write
\[
z_{s\mid g}=f^\ell_g(s)=\big(\nu_g(s),\,\xi_g(s)\big),
\qquad
z_{g\mid s}=f^\ell_s(g)=\big(\nu_s(g),\,\xi_s(g)\big),
\]
where $\nu_g,\nu_s,\xi_g,\xi_s$ are deterministic maps defined on the relevant domains induced by $\mathrm{supp}(f^e)$.
For brevity, we define
\begin{align}\nonumber
z_{s\mid g}^{\mathrm{en}} \coloneqq \nu_g(s),\qquad
z_{s\mid g}^{\mathrm{ex}} \coloneqq \xi_g(s),\qquad
z_{g\mid s}^{\mathrm{en}} \coloneqq \nu_s(g),\qquad
z_{g\mid s}^{\mathrm{ex}} \coloneqq \xi_s(g),
\end{align}
so that
\begin{equation}\label{GCE_BCMP:z_notation}
z_{s\mid g}= \big(z_{s\mid g}^{\mathrm{en}},z_{s\mid g}^{\mathrm{ex}}\big),
\qquad
z_{g\mid s}= \big(z_{g\mid s}^{\mathrm{en}},z_{g\mid s}^{\mathrm{ex}}\big).
\end{equation}

For a given pair $(s,g)$ we define $z_{s\mid g}^{\mathrm{en}}$ as the \textbf{\emph{task-endogenous state}} and $z_{s\mid g}^{\mathrm{ex}}$ as the \textbf{\emph{task-exogenous context}} of $s$ (relative to $g$),
and analogously $z_{g\mid s}^{\mathrm{en}}$ and $z_{g\mid s}^{\mathrm{ex}}$ as those of $g$ (relative to $s$), where this terminology is motivated by the following assumption.

\medskip
\noindent
\textbf{(\emph{Task-endogenous abstraction})}
For any $u=(s,g)\in\mathrm{supp}(f^e)$ and any $a\in\mathcal A$, let $\mathbf{u}'=(\mathbf{s}',g)$ denote the next observation pair after applying $a$,
and define the decoded next latent pair by $\bar{\mathbf{z}}'\coloneqq f^\ell(\mathbf{u}')\in\bar{\mathcal Z}$. Assume that $\mathbf{u}'\in \mathrm{supp}(f^e)$ almost surely, so that $f^\ell(\mathbf{u}')$ is well-defined.

Recall that $\bar z=f^\ell(u)=(z_{s\mid g},z_{g\mid s})$ (\cref{eq:bar_z}), and using the decomposition in \cref{GCE_BCMP:z_notation}, define
\[
\bar z^{\mathrm{en}}_u\coloneqq (z_{s\mid g}^{\mathrm{en}},\,z_{g\mid s}^{\mathrm{en}})\in\mathcal Z^{\mathrm{en}}\times\mathcal Z^{\mathrm{en}},
\qquad
\bar z^{\mathrm{ex}}_u\coloneqq (z_{s\mid g}^{\mathrm{ex}},\,z_{g\mid s}^{\mathrm{ex}})\in\mathcal Z^{\mathrm{ex}}\times\mathcal Z^{\mathrm{ex}},
\]
and similarly define $\bar{\mathbf{z}}'^{\mathrm{en}}_u$ and $\bar{\mathbf{z}}'^{\mathrm{ex}}_u$ component-wise from $\bar{\mathbf{z}}'$.

GCE-BCMP assumes that there exists a Markov kernel
$\mathcal P^{\mathrm{en}}:(\mathcal Z^{\mathrm{en}}\times\mathcal Z^{\mathrm{en}})\times\mathcal A
\to \Delta(\mathcal Z^{\mathrm{en}}\times\mathcal Z^{\mathrm{en}})$
such that for all $u=(s,g)\in\mathrm{supp}(f^e)$ and all $a\in\mathcal A$,
\[
\bar{\mathbf{z}}' \sim \mathcal P(\cdot\mid (\bar z^{\mathrm{en}}_u,\bar z^{\mathrm{ex}}_u),a)
\quad\Longrightarrow\quad
\bar{\mathbf{z}}'^{\mathrm{en}}_u \sim \mathcal P^{\mathrm{en}}(\cdot\mid \bar z^{\mathrm{en}}_u,a).
\]

We refer to $\bar z^{\mathrm{en}}_{u}$ as the \textbf{\emph{task}} associated with $u=(s,g)$.
Equivalently, a \textbf{\emph{task}} is an equivalence class of pairs in $\mathrm{supp}(f^e)$ that share the same endogenous pair:
\[
u_i \sim u_j
\quad \Longleftrightarrow \quad
\bar z^{\mathrm{en}}_{u_i}=\bar z^{\mathrm{en}}_{u_j}.
\]
For any $\bar z^{\mathrm{en}}\in \mathcal Z^{\mathrm{en}}\times \mathcal Z^{\mathrm{en}}$, we define the corresponding \textbf{\emph{task block}} by
\[
\mathcal B_{\bar z^{\mathrm{en}}}
\;\coloneqq\;
\big\{u\in \mathrm{supp}(f^e)\,:\, \bar z^{\mathrm{en}}_{u}=\bar z^{\mathrm{en}}\big\}.
\]
\end{definition}

\medskip
\section{Algorithm Details} \label{appendix:algorithm_detail}
\subsection{Details of the Dual Analogies}\label{appendix:algorithm_detail:analogies}
\paragraph{Learning temporal distances and extracting dual analogies.}
We learn a temporal-distance surrogate via goal-conditioned IQL.
Under the goal-reaching reward $r(s,g)=\mathbf 1_{\{s=g\}}$, we define the temporal distance by
\begin{equation}\label{appendix:eq:temporal_distance_def}
    d^{*}(s,g) := \log_{\gamma} V^{*}(s,g),
\end{equation}
which reduces to the shortest path length from $s$ to $g$ in deterministic environments.
In practice, we use the modified sparse reward $\tilde r(s,g)=-\mathbf 1_{\{s\neq g\}}$, under which the optimal return is a monotone function of the temporal distance. In other words,
\begin{equation}\label{appendix:eq:modified_return}
    \tilde V^{*}(s,g)
    =
    \sum_{t=0}^{\infty}\gamma^{t}\tilde r_t
    =
    -\sum_{t=0}^{d^{*}(s,g)-1}\gamma^{t}
    =
    -\frac{1-\gamma^{d^{*}(s,g)}}{1-\gamma},
\end{equation}
which is strictly monotone in $d^{*}(s,g)$ for $\gamma\in(0,1)$, making $\tilde r$ a valid surrogate signal for temporal-distance learning \citep{park2026dual}.
We also define a goal-conditioned Q function induced by the modified reward $\tilde r$ as
\begin{equation}\label{appendix:eq:tilde_Q_def}
    \tilde Q^{\pi}(s,a,g)
    :=
    \mathbb E^{\pi}\!\left[
        \sum_{t=0}^{\infty}\gamma^{t}\,\tilde r(s_t,g)
        \ \middle|\ s_0=s,\ a_0=a
    \right],
\end{equation}
where $\mathbb E^{\pi}$ denotes the expectation over trajectories generated by the environment dynamics and subsequent actions sampled from $\pi(\cdot\mid s_t,g)$ for $t\ge 1$.

To obtain a practical approximation of temporal-distance relations, we parameterize the goal-conditioned value with an inner-product aggregation,
\begin{equation}\label{appendix:eq:innerprod_value}
    \tilde V(s,g)
    =
    f\!\left(\phi(s),\varphi(g)\right)
    =
    \phi(s)^{\top}\varphi(g),
\end{equation}
where $\phi,\varphi:\mathcal S\to\mathbb R^{d}$ are learnable state and goal encoders, respectively.
The encoders $(\phi,\varphi)$ are trained jointly with a parametric goal-conditioned critic $Q(s,a,g):\mathcal{S}\times\mathcal A\times \mathcal S\to \mathbb R$ using the IQL objectives \citep{kostrikov2022offline}:
\begin{equation}\label{appendix:eq:dual_iql_objective}
\begin{aligned}
    \mathcal L(\phi,\varphi)
    :=
    \mathbb E_{(s,a)\sim\mathcal D_{\mathrm{train}},\, g\sim\rho(g)}
    &\Big[
        \ell^{\iota}_{2}\big(
            \phi(s)^{\top}\varphi(g)-\bar Q(s,a,g)
        \big)
    \Big],\\
    \mathcal L(Q)
    :=
    \mathbb E_{(s,a,s')\sim\mathcal D_{\mathrm{train}},\, g\sim\rho(g)}
    \Big[
        \big(
            Q(s,&a,g)-\tilde r(s,g)-\gamma\,\bar\phi(s')^{\top}\bar\varphi(g)
        \big)^2
    \Big],
\end{aligned}
\end{equation}
where $\rho(g)$ is a hindsight goal relabeling~\citep{andrychowicz2017hindsight, park2025ogbench} distribution that samples $g$ from a mixture of the current state, future states along the same trajectory, and random states from $\mathcal D_{\mathrm{train}}$,
$\ell^{\iota}_{2}(u)=|\iota-\mathbf 1\{u<0\}|\,u^2$ is the expectile loss \citep{newey1987asymmetric} with $\iota\in(0,1)$, and $\bar{\cdot}$ denotes target networks updated by exponential moving average (EMA).

All parameters $( Q,\phi,\varphi)$ are optimized using a single Adam optimizer \citep{kingma2015adam} on the summed objective
\begin{equation}
    \min_{\phi,\varphi,Q}\ \mathcal L_{\mathrm{analogy}}(\phi,\varphi, Q)
    :=
    \mathcal L(\phi,\varphi)+\mathcal L(Q).
\end{equation}
After training, the \emph{dual analogy} is extracted as the displacement in the learned goal embedding space,
\begin{equation}\label{appendix:eq:dual_analogy_def}
    \alpha^\vee(s,g) := \varphi(g)-\varphi(s)\in\mathbb R^{d},
\end{equation}
so that for any probe state $x$,
$
 \tilde V(x,g)- \tilde V(x,s)=\phi(x)^{\top}\alpha^\vee(s,g).
$

\medskip

\subsection{Details of the CTA}\label{appendix:algorithm_detail:CTA}
\paragraph{Analogy compression for practical deployment.}
The dual analogy $\alpha^\vee(s,g)=\varphi(g)-\varphi(s)\in\mathbb R^{d}$ is most informative when the embedding dimension $d$ is sufficiently large, as it increases the expressivity of the inner-product temporal-distance model $f(\phi(s),\varphi(g))=\phi(s)^\top\varphi(g)$ and enriches the representational capacity of $\alpha^\vee$ itself \citep{park2026dual}.
However, CTA requires the high-level policy to directly output a $k$-step analogy as its action, and producing a $d$-dimensional output introduces a severe bottleneck.
To enable stable control while preserving the task-endogenous displacement signal, we introduce a projection network $\eta:\mathbb R^{d}\to\mathbb R^{e}$ and perform control in the compressed analogy space.

Accordingly, the bilinear value function in \eqref{eq:bilinear-value} is modified as
\begin{equation}\label{appendix:eq:bilinear_value}
    V(s,g)
    :=
    \Omega_1(s)\boldsymbol{\cdot}\Omega_2\!\left(\eta(\alpha^\vee(s,g))\right),
\end{equation}
where $\Omega_1:\mathcal S\to\mathbb R^{b}$ is an anchor module and $\Omega_2:\mathbb R^{e}\to\mathbb R^{b}$ is a displacement module.
The value parameters $\Omega$ together with the projection $\eta$ are trained by minimizing the same action-free IQL objective:
\begin{equation}\label{appendix:eq:cta_value_objective}
    \mathcal{L}(\Omega_1, \Omega_2,\eta)
    =
    \mathbb E_{(s,s',g)}
    \Big[
        \ell^\kappa_2\big(
            \tilde r^\ell(s,g) + \gamma \bar V(s',g) - V(s,g)
        \big)
    \Big],
\end{equation}
where $\kappa\in(0,1)$ and $\bar{\cdot}$ denotes the target network.

The high-level policy treats the $k$-step analogy as its action, but predicts it in the compressed space:
\begin{equation} \label{appendix:eq:bilinear_high_policy}
    \pi_h(\,\cdot\mid s,g)
    =
    \mathcal N\!\big(\mu_h(s,g),\Sigma_h\big),
    \qquad
    \mu_h(s,g)
    =
    \omega_{h1}(s)\boldsymbol{\cdot}\omega_{h2}\!\left(\eta(\alpha^\vee(s,g))\right),
\end{equation}
where $\omega_{h1}:\mathcal S\to\mathbb R^{b\times e}$ and $\omega_{h2}:\mathbb R^{e}\to\mathbb R^{b\times e}$ are learnable anchor and displacement encoders, respectively.
Given a transition segment $(s_t,s_{t+k},g)$ from the dataset, the supervision target is the compressed $k$-step analogy $\eta(\alpha^\vee(s_t,s_{t+k}))$.
The high-level actor is trained by maximizing the advantage-weighted regression objective:
\begin{equation}\label{appendix:eq:cta_high_objective}
    \mathcal{L}(\omega_{h1}, \omega_{h2})
    =
    \mathbb E_{(s_t,s_{t+k},g)\sim\mathcal D_{\mathrm{train}}}
    \Big[
        \exp\!\big(\beta_h A(s_t,s_{t+k},g)\big)\,
        \log \pi_h\big(\sg[\eta](\alpha^\vee(s_t,s_{t+k}))\mid s_t,g\big)
    \Big],
\end{equation}
where $\sg[\cdot]$ denotes the stop-gradient and $A(s,s',g):=V(s',g)-V(s,g)$ is computed from the compressed-analogy value $V$.
Also, the low-level policy conditions on the proposed compressed analogy and outputs primitive actions:
\begin{equation} \label{appendix:eq:bilinear_low_policy}
    \pi_\ell(\,\cdot\mid s,\eta(\alpha^\vee))
    =
    \mathcal N\!\big(\mu_\ell(s,\eta(\alpha^\vee)),\Sigma_\ell\big),
    \qquad
    \mu_\ell(s,\eta(\alpha^\vee(s,g)))
    =
    \omega_{\ell1}(s)\boldsymbol{\cdot}\omega_{\ell2}\!\left(\eta(\alpha^\vee(s,g))\right),
\end{equation}
where $\omega_{\ell1}:\mathcal S\to\mathbb R^{b\times \dim(\mathcal A)}$ and $\omega_{\ell2}:\mathbb R^{e}\to\mathbb R^{b\times \dim(\mathcal A)}$ are learnable anchor and displacement encoders.
For a dataset tuple $(s_t,a_t,s_{t+1},s_{t+k})$, the conditioning signal is $\eta(\alpha^\vee(s_t,s_{t+k}))$, and the low-level actor is trained by maximizing
\begin{equation} \label{appendix:eq:cta_low_objective}
    \mathcal{L}(\omega_{\ell1}, \omega_{\ell2})
    =
    \mathbb E_{(s_t,a_t,s_{t+1},s_{t+k})\sim\mathcal D_{\mathrm{train}}}
    \Big[
        \exp\!\big(\beta_\ell A(s_t,s_{t+1},s_{t+k})\big)\,
        \log \pi_\ell\big(a_t\mid s_t,\sg[\eta](\alpha^\vee(s_t,s_{t+k}))\big)
    \Big],
\end{equation}
where $A(s_t,s_{t+1},s_{t+k}) := V(s_{t+1},s_{t+k}) - V(s_t,s_{t+k})$ uses the same compressed-analogy value function.

All parameters are optimized using a single Adam optimizer~\citep{kingma2015adam} on the summed objective,
\begin{equation}
    \min_{\Omega_1,\Omega_2, \omega_{h1},\omega_{h2},\omega_{\ell1},\omega_{\ell2},\eta}\ 
    \mathcal L_{\mathrm{CTA}}
    :=
    \mathcal{L}(\Omega_1, \Omega_2,\eta)
    -\mathcal{L}(\omega_{h1}, \omega_{h2})
    -\mathcal{L}(\omega_{\ell1}, \omega_{\ell2}),
\end{equation}
where the negative signs reflect that the actor objectives are maximized and $\eta$ is updated only through the value objective $\mathcal{L}(\Omega_1, \Omega_2,\eta)$.

\paragraph{Bilinear architecture of the value and policy functions.}
Applying bilinear transduction requires departing from a monolithic MLP and adopting a structured bilinear parameterization for both the value and policy functions.
Motivated by prior work~\citep{song2024compositional}, we implement each of $V$, $\pi_{\omega_h}$, and $\pi_{\omega_\ell}$ using three components: an \emph{anchor module}, a \emph{displacement module}, and a lightweight 2-layer MLP backbone.
The anchor and displacement modules map the anchor state and the analogy into $b\times p$ feature matrices, whose column-wise inner products yield a $p$-dimensional bilinear transduction feature.
This feature is then processed by the backbone MLP to produce the final scalar value or the policy mean vector.

Concretely, for the value function, the bilinear transduction in \eqref{appendix:eq:bilinear_value} is realized as
\begin{equation}
\begin{split}
\label{appendix:eq:impl_bilinear_value}
V(s,g)
=
\mathrm{MLP}_v &\bigg(
\Omega_1(s)\boldsymbol{\cdot}\Omega_2\! \big( \eta(\alpha^\vee(s,g)) \big)
\bigg)
\in\mathbb{R}, \\
\Omega_1(s) \ \boldsymbol{\cdot} \ &\Omega_2\! \big( \eta(\alpha^\vee(s,g)) \big) \in\mathbb R^p,
\end{split}
\end{equation}
where $\Omega_1(s)\in\mathbb{R}^{b\times p}$ and $\Omega_2(\eta(\alpha^\vee(s,g)))\in\mathbb{R}^{b\times p}$ denote the outputs of the anchor and displacement modules, respectively.
For $i=1,\ldots,p$, let $\Omega_1(s)_i\in\mathbb{R}^{b}$ and $\Omega_2(\eta(\alpha^\vee(s,g)))_i\in\mathbb{R}^{b}$ be their $i$-th column vectors.
The bilinear transduction feature is computed by column-wise inner products as
\begin{equation}
\Big[
\Omega_1(s)\boldsymbol{\cdot}\Omega_2\!\big(\eta(\alpha^\vee(s,g))\big)
\Big]_i
=
\Omega_1(s)_i^\top
\Omega_2(\eta(\alpha^\vee(s,g)))_i
\in\mathbb{R}.
\end{equation}

Similarly, the high-level policy in \eqref{appendix:eq:bilinear_high_policy} is implemented as a Gaussian actor $\pi_h(\,\cdot\mid s,g)=\mathcal N\!\big(\mu_h(s,g),\Sigma_h\big)$ such that
\begin{equation}
\begin{split}
\label{appendix:eq:impl_bilinear_high_policy}
\mu_h(s,g) = \mathrm{MLP}_h &\bigg(
\omega_{h1}(s)\boldsymbol{\cdot}\omega_{h2}\!\big(\eta(\alpha^\vee(s,g))\big)
\bigg) \in\mathbb{R}^{e},\\
\omega_{h1}(s) \ \boldsymbol{\cdot} \ &\omega_{h2}\!\big(\eta(\alpha^\vee(s,g))\big)
\in\mathbb R^p,
\end{split}
\end{equation}
where $\omega_{h1}(s)\in\mathbb{R}^{b\times p}$ and $\omega_{h2}(\eta(\alpha^\vee(s,g)))\in\mathbb{R}^{b\times p}$ denote the outputs of the anchor and displacement modules, respectively.
For $i=1,\ldots,p$, let $\omega_{h1}(s)_i\in\mathbb{R}^{b}$ and $\omega_{h2}(\eta(\alpha^\vee(s,g)))_i\in\mathbb{R}^{b}$ be their $i$-th column vectors.
The underlying bilinear transduction feature is computed by
\begin{equation}
\Big[
\omega_{h1}(s)\boldsymbol{\cdot}\omega_{h2}\!\big(\eta(\alpha^\vee(s,g))\big)
\Big]_i
=
\omega_{h1}(s)_i^\top
\omega_{h2}(\eta(\alpha^\vee(s,g)))_i
\in\mathbb{R},
\end{equation}
and the backbone $\mathrm{MLP}_h(\cdot)$ maps this $p$-dimensional feature to the mean vector in $\mathbb{R}^{e}$.

Similarly, the low-level policy in \eqref{appendix:eq:bilinear_low_policy} is implemented as $\pi_\ell(\,\cdot\mid s,\eta(\alpha^\vee))
=
\mathcal N\!\big(\mu_\ell(s,\eta(\alpha^\vee)),\Sigma_\ell\big)$ such that
\begin{equation}
\begin{split}
\label{appendix:eq:impl_bilinear_low_policy}
\mu_\ell(s,\eta(\alpha^\vee(s,g)))
=
\mathrm{MLP}_\ell &\bigg(
\omega_{\ell1}(s)\boldsymbol{\cdot}\omega_{\ell2}\!\big(\eta(\alpha^\vee(s,g))\big)
\bigg)
\in\mathbb{R}^{\dim(\mathcal A)},\\
\omega_{\ell1}(s) \ \boldsymbol{\cdot} \ &\omega_{\ell2}\!\big(\eta(\alpha^\vee(s,g))\big)
\in\mathbb R^p,
\end{split}
\end{equation}
where $\omega_{\ell1}(s)\in\mathbb{R}^{b\times p}$ and $\omega_{\ell2}(\eta(\alpha^\vee(s,g)))\in\mathbb{R}^{b\times p}$ are the anchor and displacement module outputs.
For $i=1,\ldots,p$, letting $\omega_{\ell1}(s)_i,\omega_{\ell2}(\eta(\alpha^\vee(s,g)))_i\in\mathbb{R}^{b}$ denote the $i$-th columns, the bilinear feature is given by
\begin{equation}
\Big[
\omega_{\ell1}(s)\boldsymbol{\cdot}\omega_{\ell2}\!\big(\eta(\alpha^\vee(s,g))\big)
\Big]_i
=
\omega_{\ell1}(s)_i^\top
\omega_{\ell2}(\eta(\alpha^\vee(s,g)))_i
\in\mathbb{R},
\end{equation}
which is then processed by the backbone $\mathrm{MLP}_\ell(\cdot)$ to output the policy mean in $\mathbb{R}^{\dim(\mathcal A)}$. 

\medskip

\subsection{Full algorithm}
The training procedures for dual analogy and CTA are provided in \cref{alg:dual-analogy} and \cref{alg:cta}, respectively.
\newpage
\begin{algorithm}[t]
\caption{Extracting dual analogies}
\label{alg:dual-analogy}
\begin{algorithmic}[1]
\REQUIRE Offline dataset $\mathcal D_{\mathrm{train}}$, hindsight goal relabeling distribution $\rho(g)$,
discount $\gamma$, expectile $\iota$, EMA rate $\tau$
\STATE Initialize parameters $(\phi,\varphi,Q)$ and target networks
$(\bar\phi,\bar\varphi,\bar Q)\leftarrow(\phi,\varphi,Q)$
\FOR{each gradient step}
    \STATE Sample a minibatch $\{(s,a,s')\}\sim\mathcal D_{\mathrm{train}}$ and sample goals $g\sim\rho(g)$
    \STATE Compute losses $\mathcal L(\phi,\varphi)$ and $\mathcal L(Q)$ in \eqref{appendix:eq:dual_iql_objective}
    \STATE Update $(\phi,\varphi,Q)$ minimizing $\mathcal L(\phi,\varphi)+\mathcal L(Q)$
    \STATE Update target networks by EMA:
    \STATE $\quad\bar Q \leftarrow \tau Q + (1-\tau)\bar Q$
    \STATE $\quad\bar\phi \leftarrow \tau \phi + (1-\tau)\bar\phi$
    \STATE $\quad\bar\varphi \leftarrow \tau \varphi + (1-\tau)\bar\varphi$
\ENDFOR
\STATE \textbf{return} $\alpha^\vee(s,g)=\varphi(g)-\varphi(s)$
\end{algorithmic}
\end{algorithm}
\begin{algorithm}[H]
\caption{Training CTA}
\label{alg:cta}
\begin{algorithmic}[1]
\REQUIRE Offline dataset $\mathcal D_{\mathrm{train}}$, dual analogy $\alpha^\vee(s,g)$, subgoal steps $k$, discount $\gamma$, expectile $\kappa$, temperatures $(\beta_h,\beta_\ell)$, EMA rate $\tau$
\ENSURE Value $V$ in \eqref{appendix:eq:bilinear_value}, high-level policy $\pi_h$ in \eqref{appendix:eq:bilinear_high_policy}, and low-level policy $\pi_\ell$ in \eqref{appendix:eq:bilinear_low_policy}
\STATE Initialize parameters $(\Omega_1,\Omega_2,\omega_{h1},\omega_{h2},\omega_{\ell1},\omega_{\ell2}, \eta)$ and target value network $\bar V \leftarrow V$
\FOR{each gradient step}
    \STATE {\color{juncolorblue}{{\# \ Value update}}}
    \STATE Sample $(s,s',g)\sim\mathcal D_{\mathrm{train}}$
    \STATE Compute $V(s,g)=\Omega_1(s)\boldsymbol{\cdot}\Omega_2\!\left(\eta(\alpha^\vee(s,g))\right)$
    \STATE Update $(\Omega_1,\Omega_2,\eta)$ minimizing $\mathcal L(\Omega_1,\Omega_2,\eta)$ in \eqref{appendix:eq:cta_value_objective}
    \STATE Update target value network by EMA:
    \STATE $\quad \bar V \leftarrow \tau V + (1-\tau)\bar V$
    \STATE {\color{juncolorblue}{{\# \ High-level actor update}}}
    \STATE Sample $(s_t,s_{t+k},g)\sim\mathcal D_{\mathrm{train}}$
    \STATE Compute $A_h \leftarrow V(s_{t+k},g)-V(s_t,g)$
    \STATE Maximize the actor objective $\mathcal L(\omega_{h1},\omega_{h2})$ in \eqref{appendix:eq:cta_high_objective}
    \STATE {\color{juncolorblue}{{\# \ Low-level actor update}}}
    \STATE Sample $(s_t,a_t,s_{t+1},s_{t+k})\sim\mathcal D_{\mathrm{train}}$
    \STATE Compute $A_\ell \leftarrow V(s_{t+1},s_{t+k})-V(s_t,s_{t+k})$
    \STATE Maximize the actor objective $\mathcal L(\omega_{\ell1},\omega_{\ell2})$ in \eqref{appendix:eq:cta_low_objective}
\ENDFOR
\STATE \textbf{return} $(\Omega_1,\Omega_2,\omega_{h1},\omega_{h2},\omega_{\ell1},\omega_{\ell2},\eta)$
\end{algorithmic}
\end{algorithm}

\newpage
\section{Experimental Details} \label{appendix:experimental_detail}

\subsection{OGBench Benchmark}
\paragraph{Environments.} \label{appendix:experimental_detail:environments}
Our main experiments in \cref{sec:main_experiments} are conducted on the OGBench~\citep{park2025ogbench} benchmark manipulation suite, which consists of the following three environments: \texttt{cube}, \texttt{scene}, and \texttt{puzzle}.
These tasks are built on MuJoCo with a 6-DoF UR5e robot arm, and are explicitly designed to probe object manipulation, sequential (long-horizon) reasoning, and combinatorial generalization---making them a natural testbed for compositional generalization. 
\texttt{cube} requires arranging cube blocks into target configurations via multi-object pick-and-place behaviors.
\texttt{scene} involves interacting with multiple everyday objects (e.g., drawer/window and button locks), where evaluation goals often require composing a sequence of atomic behaviors such as unlocking, opening, placing, and closing.
\texttt{puzzle} instantiates a ``Lights Out'' task, whose enormous button-state space demands strong combinatorial generalization in addition to precise low-level control.
Following the OGBench protocol, performance is evaluated by the average success rate over five pre-defined evaluation tasks, using 50 rollouts per task under slight randomization of the initial and goal states, and the final score is reported as the mean over the last three evaluation checkpoints during training.

We additionally evaluate CTA on OGBench maze navigation environments, focusing on \texttt{AntMaze} and \texttt{HumanoidMaze} in the \texttt{medium, large} and \texttt{giant} variants.
These tasks require controlling an agent to reach a goal location in a maze, coupling long-horizon navigation with low-level locomotion learned purely from offline trajectories.
The \texttt{large} mazes are more challenging than the \texttt{medium} mazes and are designed to stress long-horizon reasoning under limited offline coverage.

\paragraph{Datasets.}
For all environments, we use the standard OGBench offline datasets following the benchmark protocol.
For the manipulation suites (\texttt{cube}, \texttt{scene}, and \texttt{puzzle}), we use the standard \texttt{play} and \texttt{noisy} datasets provided by OGBench.
In OGBench, \texttt{play} trajectories are collected by open-loop, non-Markovian expert policies with temporally correlated noise, whereas \texttt{noisy} datasets are collected by closed-loop, Markovian expert policies with larger, uncorrelated Gaussian noise; consequently, \texttt{play} often appears more natural, while \texttt{noisy} typically attains higher state coverage.

OGBench supports both state-based and pixel-based observations.
In the default state-based dataset, the agent observes the full low-dimensional state.
In the \texttt{pixel} dataset, the agent receives only $64\times64\times3$ RGB images rendered from a third-person camera and no additional low-dimensional proprioceptive features (e.g., joint angles) are provided.

\begin{figure*}[h]
    \centering
    \makebox[\textwidth][c]{%
    \begin{minipage}{0.9\textwidth}
        \centering
        % Row 1
        \includegraphics[width=0.118\linewidth]{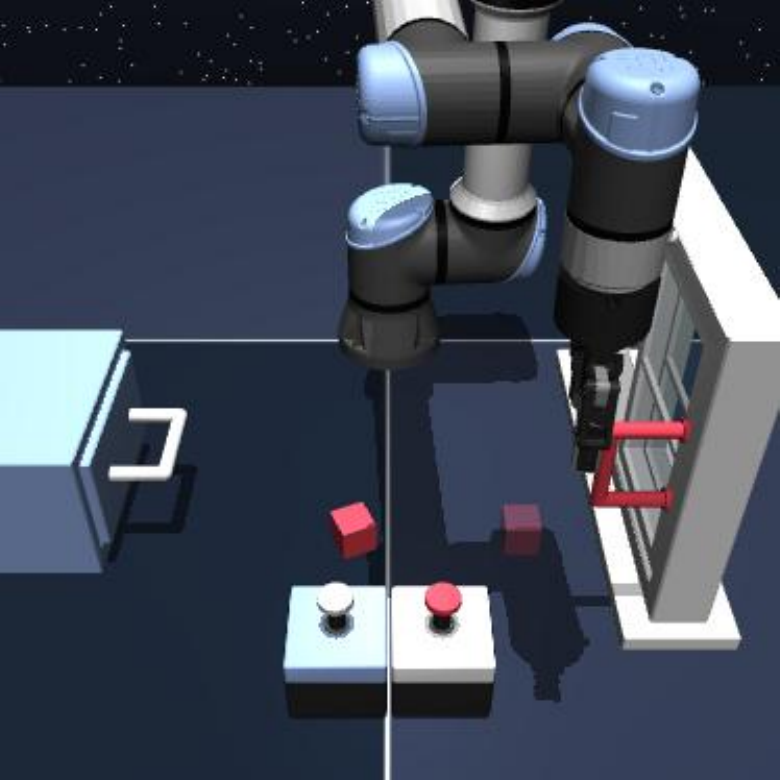}\hfill
        \includegraphics[width=0.118\linewidth]{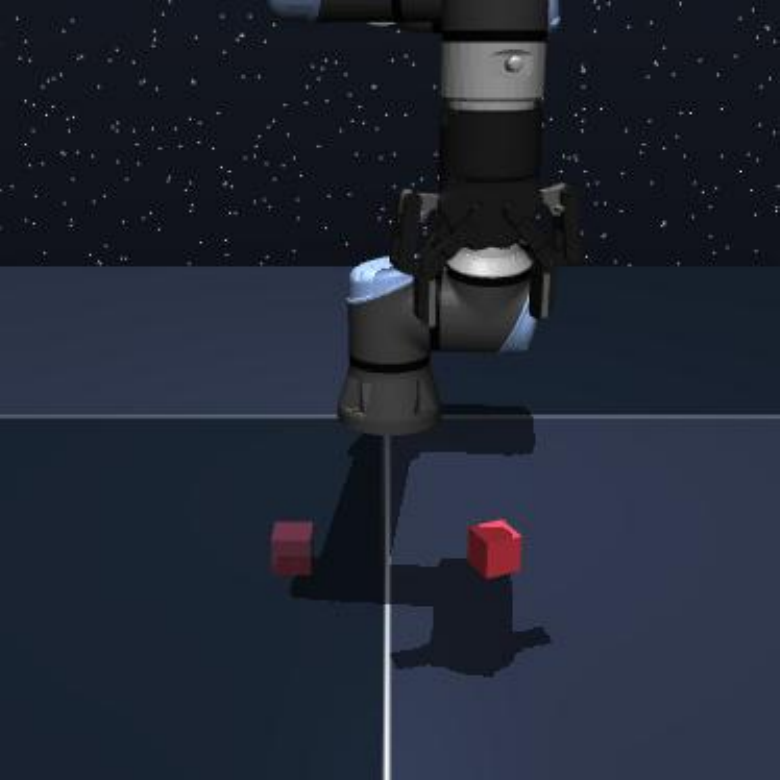}\hfill
        \includegraphics[width=0.118\linewidth]{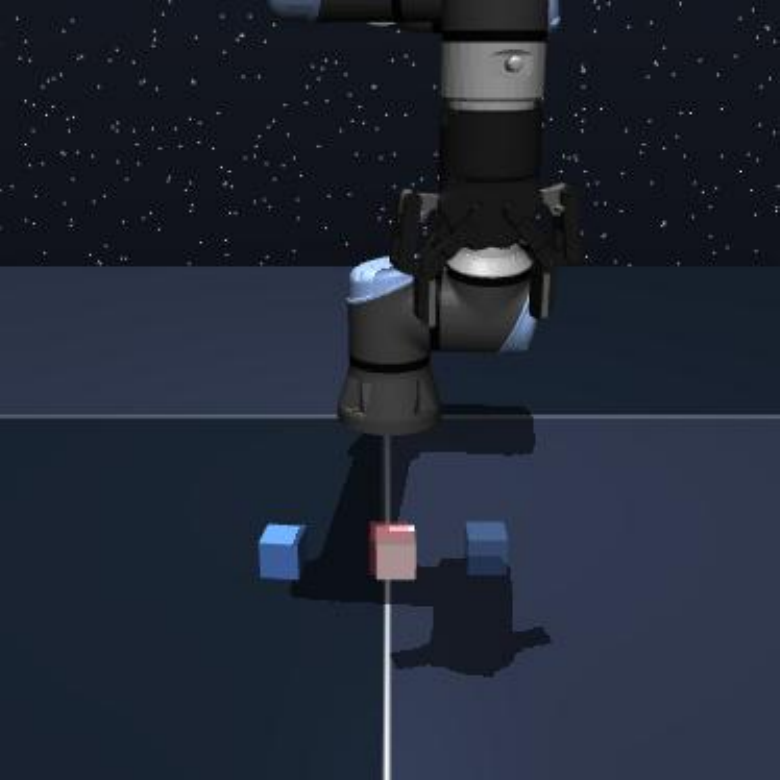}\hfill
        \includegraphics[width=0.118\linewidth]{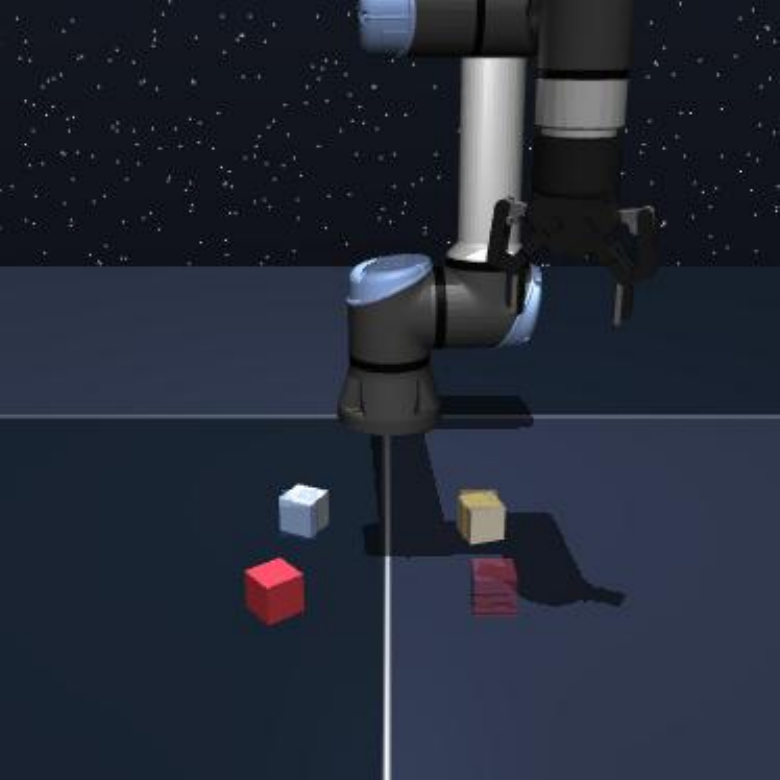}\hfill
        \includegraphics[width=0.118\linewidth]{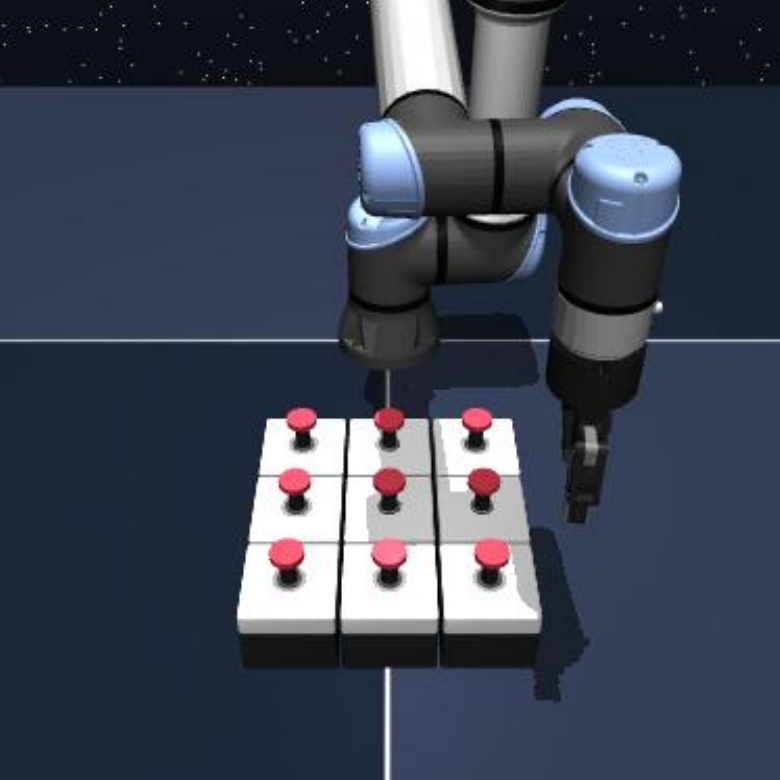}\hfill
        \includegraphics[width=0.118\linewidth]{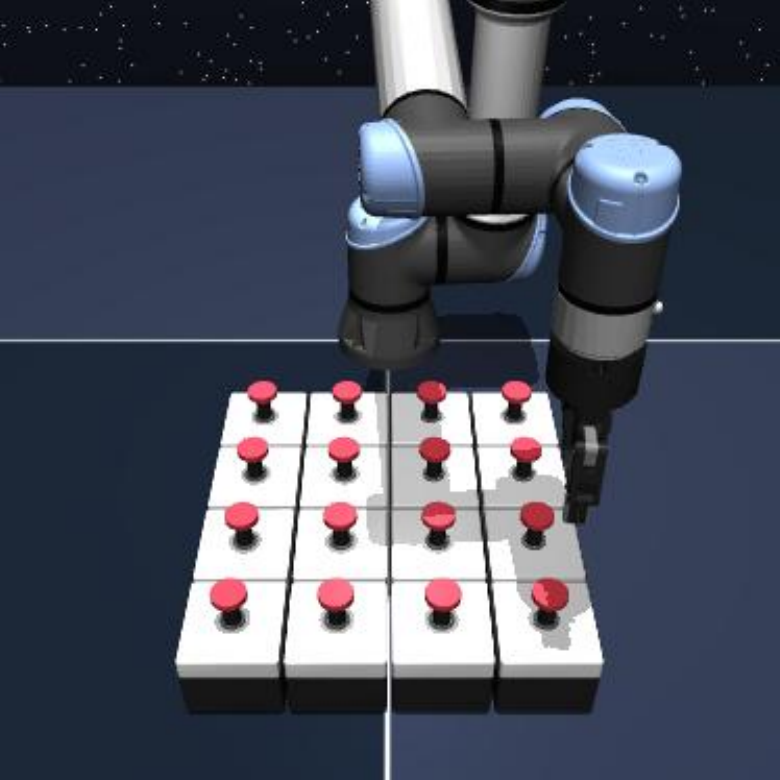}\hfill
        \includegraphics[width=0.118\linewidth]{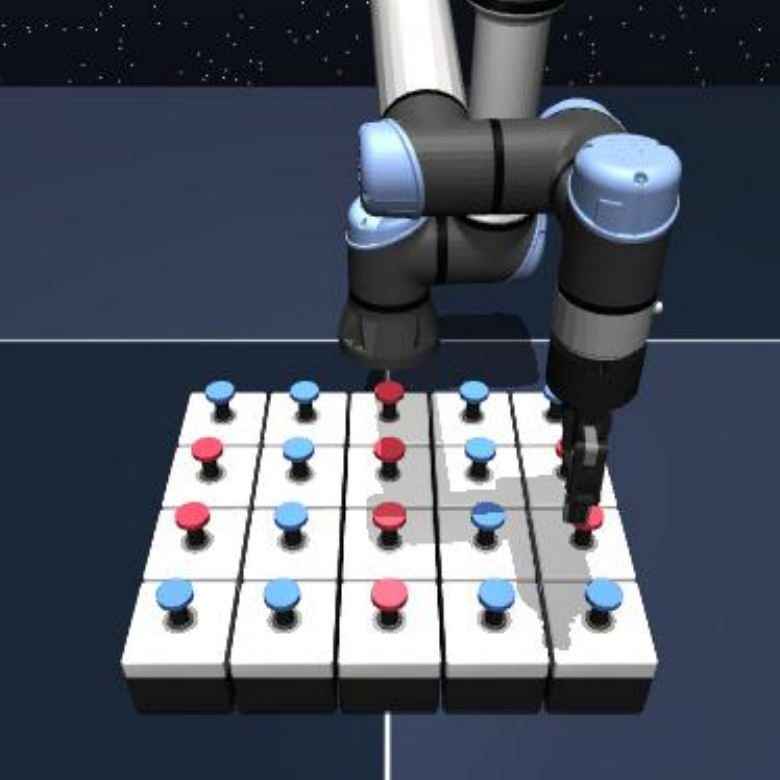}\hfill
        \includegraphics[width=0.118\linewidth]{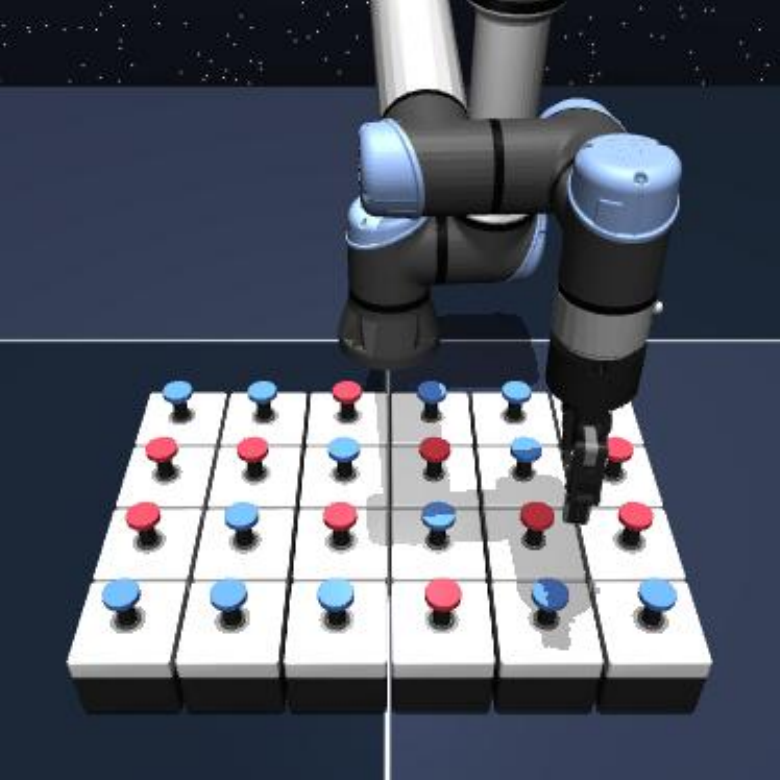}

        \par\vspace{1mm}

        % Row 2
        \includegraphics[width=0.118\linewidth]{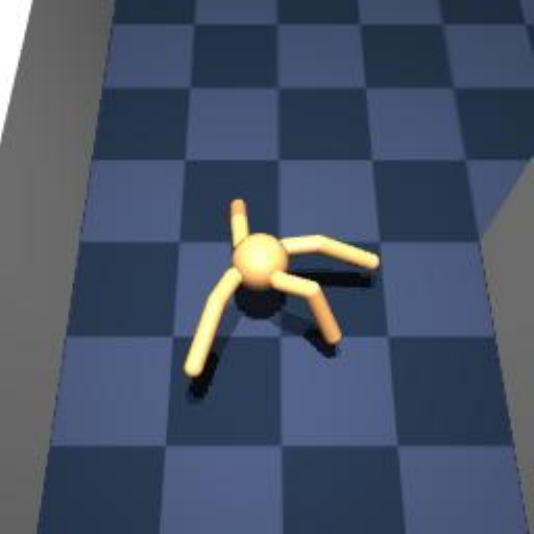}\hfill
        \includegraphics[width=0.118\linewidth]{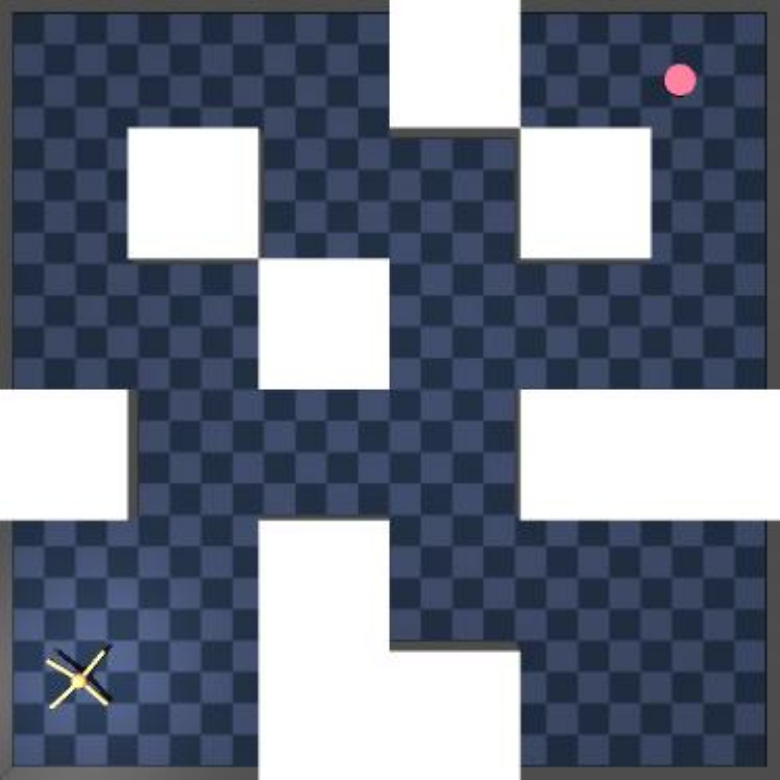}\hfill
        \includegraphics[width=0.118\linewidth]{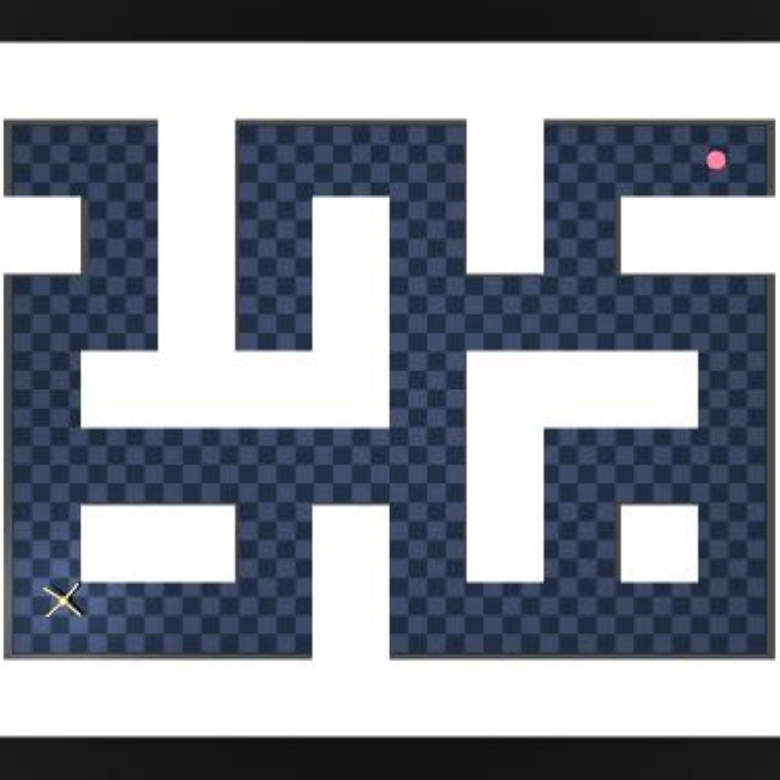}\hfill
        \includegraphics[width=0.118\linewidth]{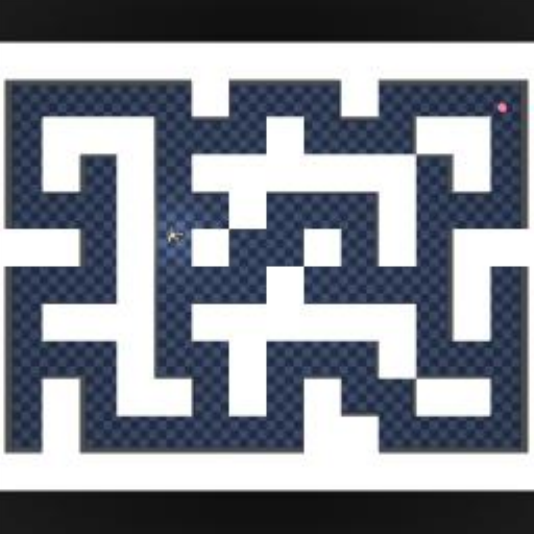}\hfill
        \includegraphics[width=0.118\linewidth]{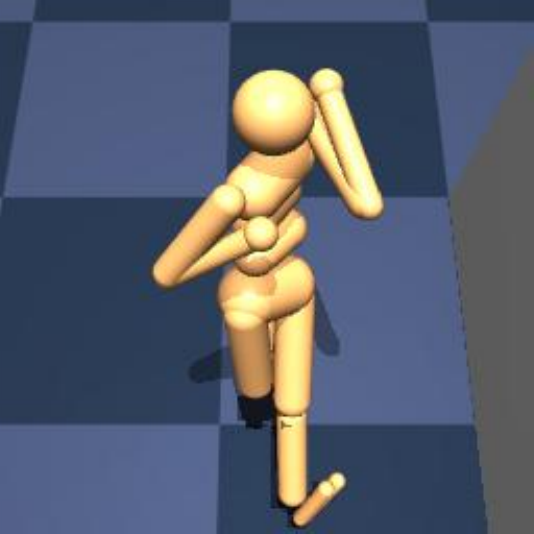}\hfill
        \includegraphics[width=0.118\linewidth]{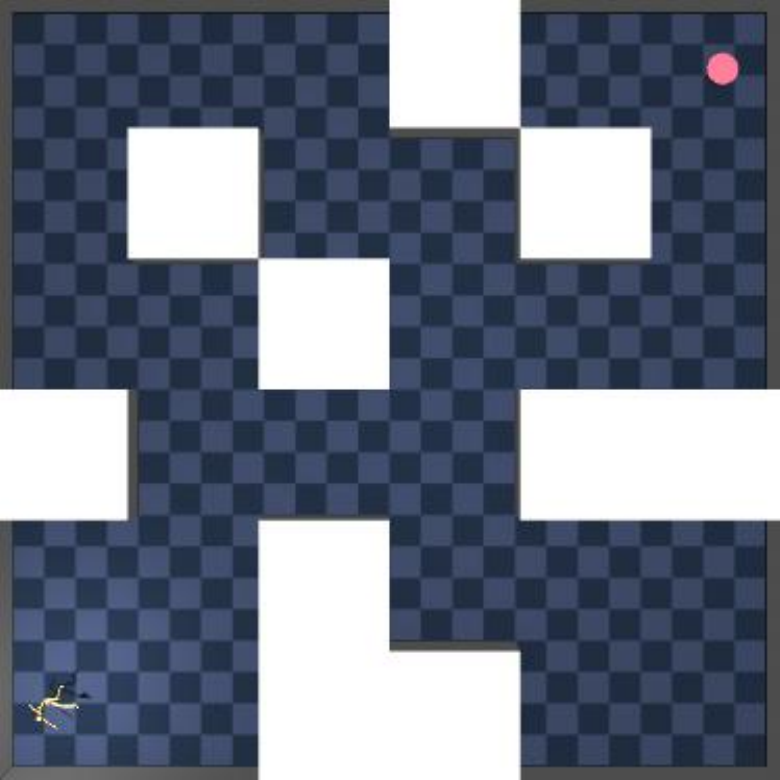}\hfill
        \includegraphics[width=0.118\linewidth]{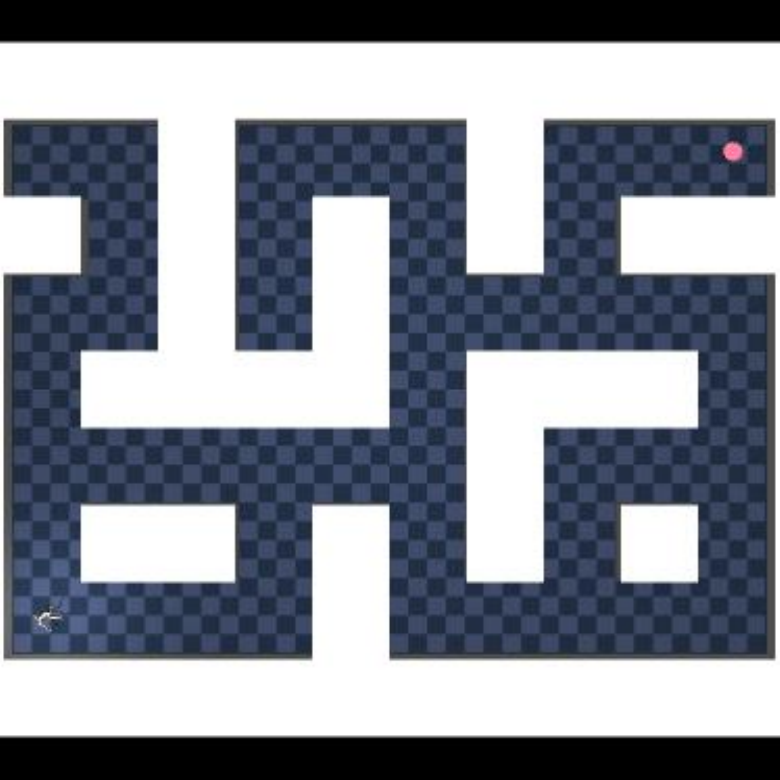}\hfill
        \includegraphics[width=0.118\linewidth]{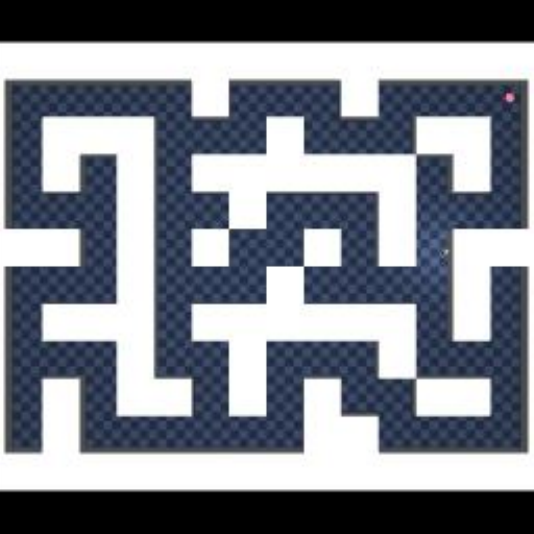}
    \end{minipage}%
    }
    \caption{\textbf{Environments.} (Top row) From left to right: \texttt{scene}, \texttt{cube-single}, \texttt{cube-double}, \texttt{cube-triple}, \texttt{puzzle-3x3}, \texttt{puzzle-4x4}, \texttt{puzzle-4x5}, and \texttt{puzzle-4x6}. (Bottom row) From left to right: ant, \texttt{antmaze-medium}, \texttt{antmaze-large}, \texttt{antmaze-giant}, humanoid, \texttt{humanoidmaze-medium}, \texttt{humanoidmaze-large}, and \texttt{humanoidmaze-giant}.}
    \label{fig:environments}
\end{figure*}

\subsection{Baselines} \label{appendix:experimental_detail:baselines}
We compare CTA against prior methods that have reported strong performance on the OGBench manipulation environments. Specifically, \cref{tab:benchmark} includes \textbf{GCBC}~\citep{ghosh2021learning, park2025ogbench}, \textbf{QRL}~\citep{wang2023optimal}, \textbf{CRL}~\citep{eysenbach2022contrastive}, \textbf{GCIVL}~\citep{kostrikov2022offline, park2025ogbench}, \textbf{GCIQL}~\citep{kostrikov2022offline}, and \textbf{HIQL}~\citep{park2023hiql}; we refer readers to the original papers for detailed descriptions. 
We also include baselines equipped with the dual goal representation~\citep{park2026dual} and refer readers to \citet{park2026dual} for details. 
For \cref{tab:benchmark} and \cref{tab:maze_result}, whenever results for a given environment are reported in OGBench~\citep{park2025ogbench} or the dual goal representation paper~\citep{park2026dual}, we use those reported numbers; all remaining results are obtained from our own experiments.
In particular, we implement \textbf{GCIQL}$^{\vee}$ in the same manner by replacing the TD update with the IQL update while keeping the representation module identical to \textbf{GCIVL}$^{\vee}$ in the original implementation of \citet{park2026dual}. 
Finally, we report two representation-augmented hierarchical baselines that we implement, namely \textbf{HIQL}$^{\vee}$ and \textbf{HIQL}$^{\vee}_{+\alpha^\vee}$.

\paragraph{HIQL$^\vee$.}
HIQL$^\vee$ augments HIQL with the dual goal representation $\varphi(\cdot)$ and conditions the goal-conditioned value and hierarchical policies on $\varphi(g)$.
To match CTA’s practical deployment setting, we apply the same projection network $\eta:\mathbb R^{d}\to\mathbb R^{e}$ and condition all modules on the compressed goal embedding $\eta(\varphi(g))$.
The goal-conditioned value function is parameterized as $V(s,g)\coloneqq V(s,\eta(\varphi(g)))$ and trained with the same action-free IQL objective:
\[
\mathcal L(V,\eta)
=
\mathbb E_{(s,s',g)\sim\mathcal D_{\mathrm{train}}}
\Big[
\ell_2^\kappa\big(
\tilde r^\ell(s,g) + \gamma \bar V(s',g) - V(s,g)
\big)
\Big],
\]
where $\kappa\in(0,1)$ and $\bar{\cdot}$ denotes the target network.
The high- and low-level actors are defined as Gaussian policies with fixed covariances $\Sigma_h$ and $\Sigma_\ell$:
\begin{align}
\pi_h(\cdot\mid s_t,g)=\mathcal N(\mu_h(s_t,g),\Sigma_h),
&\qquad
\pi_\ell(\cdot\mid s_t,s_{t+k})=\mathcal N(\mu_\ell(s_t,s_{t+k}),\Sigma_\ell),\nonumber\\
\mu_h(s_t,g)=\mu_h\big(s_t,\eta(\varphi(g))\big),
&\qquad
\mu_\ell(s_t,s_{t+k})=\mu_\ell\big(s_t,\eta(\varphi(s_{t+k}))\big).\nonumber
\end{align}
Both actors are trained by maximizing the following advantage-weighted regression objectives:
\begin{align}
\mathcal L(\pi_h)
&=
\mathbb E_{(s_t,s_{t+k},g)\sim\mathcal D_{\mathrm{train}}}
\Big[
\exp\!\big(\beta_h A(s_t,s_{t+k},g)\big)\,
\log \pi_h\big(\sg[\eta](\varphi(s_{t+k})) \mid s_t, \sg[\eta](\varphi(g))\big)
\Big],\nonumber\\
\mathcal L(\pi_\ell)
&=
\mathbb E_{(s_t,a_t,s_{t+1},s_{t+k})\sim\mathcal D_{\mathrm{train}}}
\Big[
\exp\!\big(\beta_\ell A(s_t,s_{t+1},s_{t+k})\big)\,
\log \pi_\ell\big(a_t \mid s_t, \sg[\eta](\varphi(s_{t+k}))\big)
\Big],\nonumber
\end{align}
where $A(s,s',g)\coloneqq V(s',g)-V(s,g)$ is the advantage computed from the value function.
Finally, all parameters are optimized using a single optimizer on the summed objective
\[
\min_{V,\pi_h,\pi_\ell,\eta}\ 
\mathcal L_{\mathrm{HIQL}^\vee}
:=
\mathcal L(V,\eta)
-\mathcal L(\pi_h)
-\mathcal L(\pi_\ell),
\]
and $\eta$ is updated only through $\mathcal L(V,\eta)$.

\medskip

\paragraph{HIQL$^{\vee}\!\!_{+\!\alpha^\vee}$.}
HIQL$^{\vee}\!\!_{+\!\alpha^\vee}$ replaces the dual goal representation $\varphi(g)$ in HIQL$^\vee$ with our dual analogy
\[
\alpha^\vee(s,g) \coloneqq \varphi(g)-\varphi(s)\in\mathbb R^{d},
\]
and conditions the goal-conditioned value and hierarchical policies on $\alpha^\vee(s,g)$ through the same projection $\eta$.
The goal-conditioned value function is parameterized as $V(s,g)\coloneqq V(s,\eta(\alpha^\vee(s,g)))$ and trained with the same action-free IQL objective:
\[
\mathcal L(V,\eta)
=
\mathbb E_{(s,s',g)\sim\mathcal D_{\mathrm{train}}}
\Big[
\ell_2^\kappa\big(
\tilde r^\ell(s,g) + \gamma \bar V(s',g) - V(s,g)
\big)
\Big],
\]
where $\kappa\in(0,1)$ and $\bar{\cdot}$ denotes the target network.
The high- and low-level actors are defined as Gaussian policies with fixed covariances $\Sigma_h$ and $\Sigma_\ell$:
\begin{align}
\pi_h(\cdot\mid s_t,g)=\mathcal N(\mu_h(s_t,g),\Sigma_h),
&\qquad
\pi_\ell(\cdot\mid s_t,s_{t+k})=\mathcal N(\mu_\ell(s_t,s_{t+k}),\Sigma_\ell),\nonumber\\
\mu_h(s_t,g)=\mu_h\big(s_t,\eta(\alpha^\vee(s_t,g))\big),
&\qquad
\mu_\ell(s_t,s_{t+k})=\mu_\ell\big(s_t,\eta(\alpha^\vee(s_t,s_{t+k}))\big).\nonumber
\end{align}
Both actors are trained by maximizing the following advantage-weighted regression objectives:
\begin{align}
\mathcal L(\pi_h)
&=
\mathbb E_{(s_t,s_{t+k},g)\sim\mathcal D_{\mathrm{train}}}
\Big[
\exp\!\big(\beta_h A(s_t,s_{t+k},g)\big)\,
\log \pi_h\big(\sg[\eta](\alpha^\vee(s_t,s_{t+k})) \mid s_t, \sg[\eta](\alpha^\vee(s_t,g))\big)
\Big],\nonumber\\
\mathcal L(\pi_\ell)
&=
\mathbb E_{(s_t,a_t,s_{t+1},s_{t+k})\sim\mathcal D_{\mathrm{train}}}
\Big[
\exp\!\big(\beta_\ell A(s_t,s_{t+1},s_{t+k})\big)\,
\log \pi_\ell\big(a_t \mid s_t, \sg[\eta](\alpha^\vee(s_t,s_{t+k}))\big)
\Big],\nonumber
\end{align}
where $A(s,s',g)\coloneqq V(s',g)-V(s,g)$ is the advantage computed from the value function.
Finally, all parameters are optimized using a single optimizer on the summed objective
\[
\min_{V,\pi_h,\pi_\ell,\eta}\ 
\mathcal L_{\mathrm{HIQL}^{\vee}_{+\alpha^\vee}}
:=
\mathcal L(V,\eta)
-\mathcal L(\pi_h)
-\mathcal L(\pi_\ell),
\]
and $\eta$ is updated only through $\mathcal L(V,\eta)$.

The key distinction from CTA is that HIQL$^{\vee}\!\!_{+\!\alpha^\vee}$ simply substitutes the conditioning signal in the original goal-conditioned HIQL structure, using $\alpha^\vee(s,g)$ in place of $\varphi(g)$, but does not introduce a transduction mechanism for inferring out-of-combination (OOC) analogy--context compositions.
In contrast, CTA explicitly adopts an anchor--displacement parameterization, separating the anchor state $s$ from the displacement $\alpha^\vee(s,g)$ and enforcing the low-rank structure required for bilinear transduction via a shared bottleneck.
This design targets OOC extrapolation to novel analogy--context combinations absent from the training data, even when the underlying analogy representation is shared.

\subsection{Detailed explanations of the main experiments in \cref{sec:main_experiments}}\label{appendix:main_exp_details}
\paragraph{OOC case study.}

To provide stronger evidence that CTA is indeed performing OOC extrapolation, we constructed additional experiments. 
In \texttt{scene-play-v0}, where task information and context can be separated most intuitively, and in \texttt{puzzle-4x4-play-v0}, which is specifically designed to evaluate compositional generalization, we designed experiments in which demonstrations of a single task under a particular context were completely removed from the training dataset. 
We then measured the success rate of inference on the removed context--task pair.

Specifically, in \texttt{scene-play-v0}, we removed the following three context--task pairs from the training data and evaluated the corresponding inference success rates:
\begin{figure}[h]
    \centering
    \captionsetup[subfigure]{font=tiny}
    \begin{minipage}{0.6\textwidth}
        \centering
        \begin{subfigure}[t]{0.32\textwidth}
            \centering
            \includegraphics[width=\linewidth]{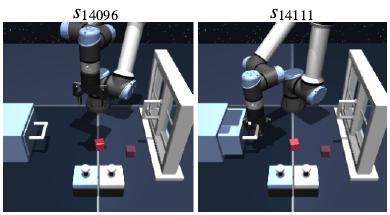}
            \caption{\textbf{Pair 1}}
        \end{subfigure}
        \hfill
        \begin{subfigure}[t]{0.32\textwidth}
            \centering
            \includegraphics[width=\linewidth]{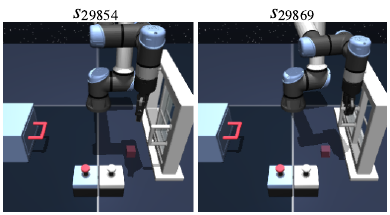}
            \caption{\textbf{Pair 2}}
        \end{subfigure}
        \hfill
        \begin{subfigure}[t]{0.32\textwidth}
            \centering
            \includegraphics[width=\linewidth]{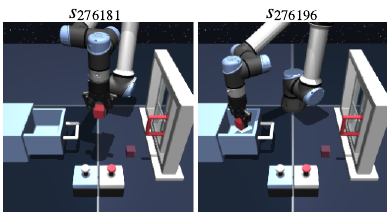}
            \caption{\textbf{Pair 3}}
        \end{subfigure}
    \end{minipage}
    \caption{Examples of the removed context--task pairs in \texttt{scene-play-v0}.}
    \label{fig:ooc_scene_pairs}
    \vspace{-1em}
\end{figure}
\begin{enumerate}[label=$\circ$, leftmargin=1.2em, itemsep=0em, topsep=0.1em, parsep=0em]
    \item \textbf{Pair 1:} context: \texttt{window closed}, \texttt{window unlocked}, \texttt{drawer closed} / task: \texttt{open drawer}
    \item \textbf{Pair 2:} context: \texttt{drawer closed}, \texttt{drawer locked}, \texttt{window open} / task: \texttt{close window}
    \item \textbf{Pair 3:} context: \texttt{window open}, \texttt{window locked}, \texttt{drawer open}, \texttt{cube not in drawer} / task: \texttt{put the cube into the drawer}
\end{enumerate}
The 15 timesteps preceding each task completion event were removed whenever the corresponding context was satisfied, and the resulting split segments were treated as separate episodes. 
After dataset reprocessing, 5070, 6960, and 1110 transitions were removed from \texttt{scene-play-v0} for the three context--task pairs, respectively. Examples for each pair are shown in \cref{fig:ooc_scene_pairs}.

In \texttt{puzzle-4x4-play-v0}, we removed the following five context--task pairs from the training data and evaluated the corresponding inference success rates:
\begin{figure}[h]
    \centering
    \captionsetup[subfigure]{font=tiny}
    \begin{subfigure}[t]{0.19\textwidth}
        \centering
        \includegraphics[width=\linewidth]{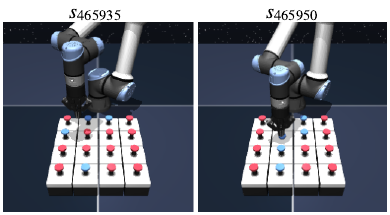}
        \caption{\textbf{Pair 1}}
    \end{subfigure}
    \hfill
    \begin{subfigure}[t]{0.19\textwidth}
        \centering
        \includegraphics[width=\linewidth]{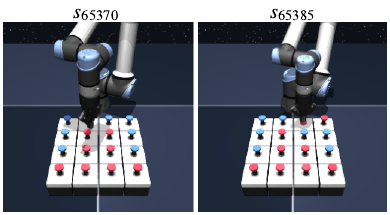}
        \caption{\textbf{Pair 2}}
    \end{subfigure}
    \hfill
    \begin{subfigure}[t]{0.19\textwidth}
        \centering
        \includegraphics[width=\linewidth]{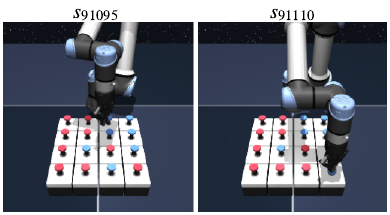}
        \caption{\textbf{Pair 3}}
    \end{subfigure}
    \hfill
    \begin{subfigure}[t]{0.19\textwidth}
        \centering
        \includegraphics[width=\linewidth]{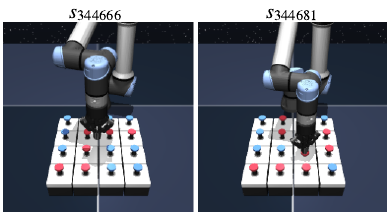}
        \caption{\textbf{Pair 4}}
    \end{subfigure}
    \hfill
    \begin{subfigure}[t]{0.19\textwidth}
        \centering
        \includegraphics[width=\linewidth]{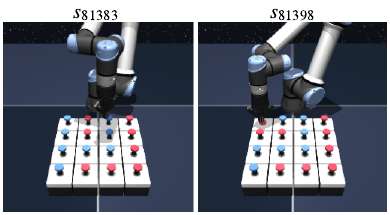}
        \caption{\textbf{Pair 5}}
    \end{subfigure}
    \caption{Examples of the removed context--task pairs in \texttt{puzzle-4x4-play-v0}.}
    \label{fig:ooc_puzzle_pairs}
    \vspace{-1em}
\end{figure}
\begin{enumerate}[label=$\circ$, leftmargin=1.2em, itemsep=0em, topsep=0.1em, parsep=0em]
    \item \textbf{Pair 1:} context: \texttt{button5 = 0}, \texttt{button1 = 1}, \texttt{button9 = 0}, \texttt{button4 = 1}, \texttt{button6 = 0} / task: \texttt{press button5}
    \item \textbf{Pair 2:} context: \texttt{button2 = 1}, \texttt{button1 = 0}, \texttt{button3 = 1}, \texttt{button6 = 0} / task: \texttt{press button2}
    \item \textbf{Pair 3:} context: \texttt{button15 = 0}, \texttt{button11 = 1}, \texttt{button14 = 1} / task: \texttt{press button15}
    \item \textbf{Pair 4:} context: \texttt{button10 = 1}, \texttt{button6 = 0}, \texttt{button14 = 1}, \texttt{button9 = 0}, \texttt{button11 = 1} / task: \texttt{press button10}
    \item \textbf{Pair 5:} context: \texttt{button0 = 1}, \texttt{button1 = 0}, \texttt{button4 = 1} / task: \texttt{press button0}
\end{enumerate}
The 20 timesteps preceding each task completion event were removed whenever the corresponding context was satisfied, and the resulting split segments were treated as separate episodes. 
After dataset reprocessing, 1080, 2740, 4680, 1080, and 4320 transitions were removed from \texttt{puzzle-4x4-play-v0} for the five context--task pairs, respectively. Examples for each pair are shown in \cref{fig:ooc_puzzle_pairs}.

The removed context--task pairs become OOC pairs at inference time. 
When evaluating state--goal pairs that require solving the task under the corresponding context, we report two evaluation metrics. 
In addition to the standard success rate, we also measure a \emph{direct success rate}, which checks whether the agent solves the given single task directly, rather than reaching the goal by detouring through in-distribution context--task combinations. 
This distinction is important because the standard success rate alone may overestimate OOC generalization. 
For example, in Pair 1 of \texttt{scene-play-v0}, instead of directly opening the drawer under the held-out context, the agent could first open the window, then open the drawer, and finally close the window again to reach the same goal. 
Such a trajectory clearly has a longer temporal distance than directly performing the intended \texttt{open drawer} action, and therefore does not correspond to the higher-value optimal path. 
Accordingly, it is not counted in the direct success rate.
The full results are shown in Tables~\ref{tab:ooc_scene} and~\ref{tab:ooc_puzzle}. 
Each entry is reported in the form of \emph{direct success rate} (\emph{success rate}). 
Following the OGBench evaluation protocol, evaluation was performed every 100,000 training steps, and we report the mean ($\pm$ std) over the final three evaluations during training.

\begin{table}[t]
    \centering
    \caption{OOC case study results on \texttt{scene-play-v0} with 4 seeds. Each entry is reported as direct success rate (success rate). \textbf{Bold} indicates the best score, and values within 95\% of the best are also bold.}
    \label{tab:ooc_scene}
    \resizebox{\textwidth}{!}{
    \begin{tabular}{llccccc}
        \toprule
        Dataset & Pair & HIQL & GCIQL$^{\vee}$ & HIQL$^{\vee}$ & HIQL$^{\vee}_{+\alpha^{\vee}}$ & CTA \\
        \midrule
        \multirow{4}{*}{\texttt{scene-play-v0}} 
        & Pair 1 
        & $8\pm2$ ($39\pm10$)
        & $53\pm17$ ($81\pm16$)
        & $21\pm14$ ($91\pm6$)
        & $29\pm12$ ($99\pm1$)
        & $\mathbf{57\pm12}$ ($98\pm2$) \\

        & Pair 2 
        & $23\pm19$ ($48\pm18$)
        & $79\pm4$ ($84\pm6$)
        & $42\pm9$ ($92\pm8$)
        & $64\pm14$ ($99\pm1$)
        & $\mathbf{83\pm7}$ ($98\pm2$) \\

        & Pair 3 
        & $26\pm8$ ($40\pm8$)
        & $22\pm8$ ($25\pm10$)
        & $71\pm11$ ($79\pm7$)
        & $50\pm15$ ($60\pm15$)
        & $\mathbf{80\pm8}$ ($86\pm7$) \\

        \cmidrule(lr){2-7}
        & Overall 
        & $19\pm10$ ($42\pm12$)
        & $51\pm10$ ($63\pm11$)
        & $45\pm11$ ($87\pm7$)
        & $48\pm14$ ($86\pm6$)
        & $\mathbf{73\pm9}$ ($94\pm4$) \\
        \bottomrule
    \end{tabular}
    }
\end{table}

\begin{table}[t]
    \centering
    \caption{OOC case study results on \texttt{puzzle-4x4-play-v0} with 4 seeds. Each entry is reported as direct success rate (success rate). \textbf{Bold} indicates the best score, and values within 95\% of the best are also bold.}
    \label{tab:ooc_puzzle}
    \resizebox{\textwidth}{!}{
    \begin{tabular}{llccccc}
        \toprule
        Dataset & Pair & HIQL & GCIQL$^{\vee}$ & HIQL$^{\vee}$ & HIQL$^{\vee}_{+\alpha^{\vee}}$ & CTA \\
        \midrule
        \multirow{6}{*}{\texttt{puzzle-4x4-play-v0}} 
        & Pair 1 
        & $38\pm11$ ($70\pm3$)
        & $21\pm8$ ($29\pm9$)
        & $28\pm11$ ($60\pm9$)
        & $\mathbf{74\pm11}$ ($94\pm5$)
        & $\mathbf{74\pm11}$ ($100\pm0$) \\

        & Pair 2 
        & $39\pm12$ ($72\pm7$)
        & $53\pm7$ ($74\pm9$)
        & $43\pm22$ ($70\pm12$)
        & $64\pm9$ ($96\pm3$)
        & $\mathbf{80\pm14}$ ($100\pm0$) \\

        & Pair 3 
        & $37\pm16$ ($64\pm14$)
        & $57\pm27$ ($64\pm26$)
        & $42\pm7$ ($68\pm3$)
        & $57\pm11$ ($96\pm3$)
        & $\mathbf{81\pm7}$ ($100\pm1$) \\

        & Pair 4 
        & $37\pm13$ ($74\pm14$)
        & $24\pm10$ ($34\pm11$)
        & $35\pm25$ ($56\pm23$)
        & $72\pm9$ ($95\pm4$)
        & $\mathbf{79\pm7}$ ($100\pm1$) \\

        & Pair 5 
        & $36\pm5$ ($66\pm9$)
        & $66\pm2$ ($74\pm4$)
        & $27\pm19$ ($54\pm20$)
        & $61\pm13$ ($94\pm5$)
        & $\mathbf{85\pm3}$ ($100\pm1$) \\

        \cmidrule(lr){2-7}
        & Overall 
        & $37\pm11$ ($69\pm9$)
        & $44\pm11$ ($55\pm12$)
        & $35\pm17$ ($62\pm13$)
        & $66\pm11$ ($95\pm4$)
        & $\mathbf{80\pm8}$ ($100\pm1$) \\
        \bottomrule
    \end{tabular}
    }
\end{table}

CTA consistently outperformed the baselines in terms of direct success rate. Interestingly, many baselines showed low direct success rates despite achieving relatively high success rates. This implies that, even when directly performing the task is clearly more efficient, those baselines assign higher value to behaviors that detour through in-distribution context--task pairs. By contrast, CTA, through its bilinear value and policy structure, correctly assigns high value to behaviors that directly solve the task in previously unseen situations via extrapolation, and indeed executes the action sequence corresponding to the intended direct task. We believe this provides clear evidence that CTA is indeed performing OOC extrapolation as intended, and that this capability leads to stronger generalization performance.

\paragraph{Additional qualitative visualization of the dual analogies.}
To qualitatively verify that the dual analogy captures task-endogenous analogies, we
visualize the top-10 nearest analogies for each query corresponding to OOC pairs
used in the direct case study: three pairs from \texttt{scene} and three pairs
from \texttt{puzzle-4x4}. Results for \texttt{scene} are shown in \cref{fig:dual_analogy_scene},
and results for \texttt{puzzle-4x4} are shown in \cref{fig:dual_analogy_puzzle}.

\newpage

\section{Additional Experiments} \label{appendix:additional_experiments}
\subsection{Additional Benchmark Results}  \label{appendix:additional_experiments:additional_benchmark_results}
\paragraph{Comparison to the GCB analogy.}
We compare the dual analogy with the GCB analogy~\citep{hansen2022bisimulation} to assess the utility of the dual analogy in the offline GCRL setting. Table~\ref{tab:noisy_gcb_vs_cta} reports results obtained by replacing only the analogy in the CTA architecture with the GCB analogy $\psi$ in~\eqref{eq:org_gcbisim}. Here, CTA w/ $\psi$ and CTA w/ $\alpha^\vee$ denote variants that keep the CTA structure fixed while using the GCB analogy and the dual analogy, respectively. We evaluate this comparison on the \texttt{play} dataset across eight OGBench manipulation tasks.
CTA with dual analogies succeeds where GCB analogies struggle because GCB defines behavioral equivalence through on-policy and reward-based matching, which is brittle and noise-amplifying under suboptimal offline data and distribution shift, whereas dual analogies leverage temporal-distance structure independent of policy and reward, yielding more reliable transduction and OOC extrapolation.

\begin{table}[t]
\centering
\captionsetup{skip=12pt, font=small, width=\columnwidth}
\caption{\textbf{GCB analogy vs. dual analogy (4 seeds).} GCB analogy fails in offline GCRL. \textbf{Bold} indicates the best score, and values within 95\% of the best are also bold.}
\label{tab:noisy_gcb_vs_cta}
\renewcommand{\arraystretch}{1.08}
\small
\begin{tabularx}{0.5\columnwidth}{l *{2}{>{\centering\arraybackslash}X}}
\toprule
\textbf{Environment} &
\textbf{CTA w/ $\psi$} &
\textbf{CTA w/ $\alpha^\vee$} \\
\midrule
\texttt{scene-play}
& $3{\scalebox{0.6}{$\pm1$}}$
& $\mathbf{90}{\scalebox{0.6}{$\pm4$}}$ \\
\texttt{cube-single-play}
& $8{\scalebox{0.6}{$\pm1$}}$
& $\mathbf{86}{\scalebox{0.6}{$\pm3$}}$ \\
\texttt{cube-double-play}
& $0{\scalebox{0.6}{$\pm0$}}$
& $\mathbf{50}{\scalebox{0.6}{$\pm5$}}$ \\
\texttt{cube-triple-play}
& $0{\scalebox{0.6}{$\pm0$}}$
& $\mathbf{17}{\scalebox{0.6}{$\pm1$}}$ \\
\texttt{puzzle-3x3-play}
& $0{\scalebox{0.6}{$\pm0$}}$
& $\mathbf{94}{\scalebox{0.6}{$\pm11$}}$ \\
\texttt{puzzle-4x4-play}
& $0{\scalebox{0.6}{$\pm0$}}$
& $\mathbf{84}{\scalebox{0.6}{$\pm3$}}$ \\
\texttt{puzzle-4x5-play}
& $0{\scalebox{0.6}{$\pm0$}}$
& $\mathbf{17}{\scalebox{0.6}{$\pm1$}}$ \\
\texttt{puzzle-4x6-play}
& $0{\scalebox{0.6}{$\pm0$}}$
& $\mathbf{12}{\scalebox{0.6}{$\pm2$}}$ \\
\midrule
\multicolumn{1}{c}{\textbf{Average}}
& $1.4$
& $\mathbf{56.3}$ \\
\bottomrule
\end{tabularx}
\end{table}

\paragraph{Results with maze environments.}
Since our dual analogy grounded in GCE-BCMP is designed to isolate task-endogenous states, CTA can be less effective in environments where task-endogenous factors and task-exogenous contexts are not cleanly separable. \texttt{Maze} environments are representative examples, where the task-endogenous component relevant to reward is the agent's global $(x,y)$ position, while the task-exogenous context involves the underlying joint configuration that realizes transitions in the global space.
In such settings, even if the analogy abstracts away joint states and depends only on the global $(x,y)$ coordinates, analogy transduction largely reduces to in-distribution trajectory stitching.
Nevertheless, as shown in \cref{tab:maze_result}, CTA remains competitive with strong maze baselines in this regime.

\begin{table}[t]
\centering
\captionsetup{skip=12pt, font=small, width=\columnwidth}
\caption{\textbf{OGBench Maze results (4 seeds).} CTA remains competitive even under in-distribution evaluation. \textbf{Bold} indicates the best score, and values within 95\% of the best are also bold.}
\label{tab:maze_result}
\renewcommand{\arraystretch}{1.08}
\small
\begin{tabularx}{\columnwidth}{l *{7}{>{\centering\arraybackslash}X}}
\toprule
\textbf{Dataset (-navigate)} &
\textbf{GCIVL} &
\textbf{CRL} &
\textbf{HIQL} &
\textbf{GCIVL}$^{\vee}$ &
\textbf{CRL}$^{\vee}$ &
\textbf{HIQL}$^{\vee}$ &
\textbf{CTA} \\
\midrule
\texttt{antmaze-medium}
& $71{\scalebox{0.6}{$\pm4$}}$
& $\mathbf{95}{\scalebox{0.6}{$\pm1$}}$
& $\mathbf{96}{\scalebox{0.6}{$\pm1$}}$
& $75{\scalebox{0.6}{$\pm4$}}$
& $\mathbf{93}{\scalebox{0.6}{$\pm3$}}$
& $\mathbf{96}{\scalebox{0.6}{$\pm1$}}$
& $\mathbf{96}{\scalebox{0.6}{$\pm1$}}$ \\
\texttt{antmaze-large}
& $16{\scalebox{0.6}{$\pm3$}}$
& $83{\scalebox{0.6}{$\pm4$}}$
& $\mathbf{91}{\scalebox{0.6}{$\pm2$}}$
& $28{\scalebox{0.6}{$\pm11$}}$
& $\mathbf{87}{\scalebox{0.6}{$\pm2$}}$
& $75{\scalebox{0.6}{$\pm2$}}$
& $85{\scalebox{0.6}{$\pm3$}}$ \\
\texttt{antmaze-giant}
& $0{\scalebox{0.6}{$\pm0$}}$
& $16{\scalebox{0.6}{$\pm3$}}$
& $\mathbf{65}{\scalebox{0.6}{$\pm5$}}$
& $0{\scalebox{0.6}{$\pm0$}}$
& $21{\scalebox{0.6}{$\pm4$}}$
& $44{\scalebox{0.6}{$\pm3$}}$ 
& $54{\scalebox{0.6}{$\pm4$}}$ \\
\texttt{humanoidmaze-medium}
& $27{\scalebox{0.6}{$\pm3$}}$
& $60{\scalebox{0.6}{$\pm4$}}$
& $\mathbf{89}{\scalebox{0.6}{$\pm2$}}$
& $29{\scalebox{0.6}{$\pm3$}}$
& $57{\scalebox{0.6}{$\pm4$}}$
& $\mathbf{89}{\scalebox{0.6}{$\pm3$}}$ 
& $\mathbf{90}{\scalebox{0.6}{$\pm2$}}$ \\
\texttt{humanoidmaze-large}
& $3{\scalebox{0.6}{$\pm0$}}$
& $24{\scalebox{0.6}{$\pm4$}}$
& $49{\scalebox{0.6}{$\pm4$}}$
& $3{\scalebox{0.6}{$\pm2$}}$
& $18{\scalebox{0.6}{$\pm4$}}$
& $48{\scalebox{0.6}{$\pm3$}}$ 
& $\mathbf{60}{\scalebox{0.6}{$\pm3$}}$ \\
\texttt{humanoidmaze-giant}
& $0{\scalebox{0.6}{$\pm0$}}$
& $3{\scalebox{0.6}{$\pm2$}}$
& $\mathbf{12}{\scalebox{0.6}{$\pm4$}}$
& $0{\scalebox{0.6}{$\pm0$}}$
& $3{\scalebox{0.6}{$\pm1$}}$
& $10{\scalebox{0.6}{$\pm4$}}$ 
& $5{\scalebox{0.6}{$\pm1$}}$ \\
\midrule
\multicolumn{1}{c}{\textbf{Average}}
& $19.5$
& $46.8$
& $\mathbf{67.0}$
& $22.5$
& $46.5$
& $60.3$
& $\mathbf{65.0}$ \\
\bottomrule
\end{tabularx}
\end{table}

\paragraph{Results in pixel-based environments.}
CTA can be brittle in pixel-based environments because it inherits the same failure mode identified for dual goal representations~\citep{park2026dual} (see \cref{tab:visual_subset}).
In particular, \citet{park2026dual} argue that representation-conditioned formulations cannot directly exploit early fusion of visual state and goal, since the goal must be processed separately before conditioning the policy.
This architectural constraint effectively enforces a late-fusion design, which is often weaker than early fusion in visual robotics.
Because CTA is likewise conditioned on a learned representation rather than the raw goal observation, it shares this fusion bottleneck, and its performance in pixel-based tasks can therefore be highly non-robust, exhibiting gains when the learned conditioning aligns with the task but collapsing when it does not.

\begin{table}[t]
\centering
\captionsetup{skip=12pt, font=small, width=\columnwidth}
\caption{\textbf{Results in pixel-based environments (4 seeds).} \textbf{Bold} indicates the best score, and values within 95\% of the best are also bold.}
\label{tab:visual_subset}
\renewcommand{\arraystretch}{1.08}
\small
\begin{tabularx}{\columnwidth}{l *{5}{>{\centering\arraybackslash}X}}
\toprule
\textbf{Environment} &
\textbf{HIQL} &
\textbf{GCIVL}$^{\vee}$ &
\textbf{HIQL}$^{\vee}$ &
\textbf{HIQL}$^{\vee}_{+\alpha^\vee}$ &
\textbf{CTA} \\
\midrule
\texttt{visual-scene-play}
& $49{\scalebox{0.6}{$\pm4$}}$
& $26{\scalebox{0.6}{$\pm5$}}$
& $39{\scalebox{0.6}{$\pm27$}}$
& $53{\scalebox{0.6}{$\pm8$}}$
& $\mathbf{59}{\scalebox{0.6}{$\pm11$}}$ \\
\texttt{visual-cube-single-play}
& $\mathbf{89}{\scalebox{0.6}{$\pm0$}}$
& $58{\scalebox{0.6}{$\pm5$}}$
& $\mathbf{88}{\scalebox{0.6}{$\pm1$}}$
& $\mathbf{87}{\scalebox{0.6}{$\pm1$}}$
& $\mathbf{89}{\scalebox{0.6}{$\pm2$}}$ \\
\texttt{visual-cube-double-play}
& $\mathbf{39}{\scalebox{0.6}{$\pm2$}}$
& $9{\scalebox{0.6}{$\pm2$}}$
& $11{\scalebox{0.6}{$\pm2$}}$
& $11{\scalebox{0.6}{$\pm1$}}$
& $8{\scalebox{0.6}{$\pm2$}}$ \\
\texttt{visual-puzzle-3x3-play}
& $\mathbf{73}{\scalebox{0.6}{$\pm8$}}$
& $0{\scalebox{0.6}{$\pm0$}}$
& $0{\scalebox{0.6}{$\pm0$}}$
& $0{\scalebox{0.6}{$\pm0$}}$
& $0{\scalebox{0.6}{$\pm0$}}$ \\
\texttt{visual-puzzle-4x4-play}
& $\mathbf{60}{\scalebox{0.6}{$\pm41$}}$
& $0{\scalebox{0.6}{$\pm0$}}$
& $0{\scalebox{0.6}{$\pm0$}}$
& $0{\scalebox{0.6}{$\pm0$}}$
& $0{\scalebox{0.6}{$\pm0$}}$ \\
\midrule
\multicolumn{1}{c}{\textbf{Average}}
& $\mathbf{62.0}$
& $18.6$
& $27.6$
& $30.2$
& $31.2$ \\
\bottomrule
\end{tabularx}
\end{table}

\paragraph{Results with noisy data.}
CTA is robust to noisy observations because it builds its analogy representation on the optimal temporal distance (see \cref{subsec:GCE-BCMP}). \cref{tab:noisy_subset} reports results on the \texttt{noisy} datasets in OGBench. Consistent with the BCMP-based perspective, both the dual goal representation and the dual analogy exhibit strong robustness to noise. Leveraging an analogy extracted from the optimal temporal distance, CTA continues to perform reliable analogy transduction via stable OOC extrapolation in noisy settings, and outperforms the competing baselines.

\begin{table}[t]
\centering
\captionsetup{skip=12pt, font=small, width=\columnwidth}
\caption{\textbf{Noisy OGBench results (8 seeds).} The dual analogy is robust to noise. \textbf{Bold} indicates the best score, and values within 95\% of the best are also bold.}
\label{tab:noisy_subset}
\renewcommand{\arraystretch}{1.08}
\small
\begin{tabularx}{\columnwidth}{l *{5}{>{\centering\arraybackslash}X}}
\toprule
\textbf{Environment} &
\textbf{GCIVL} &
\textbf{HIQL} &
\textbf{HIQL}$^{\vee}$ &
\textbf{HIQL}$^{\vee}_{+\alpha^\vee}$ &
\textbf{CTA} \\
\midrule
\texttt{scene-noisy}
& $26{\scalebox{0.6}{$\pm5$}}$
& $25{\scalebox{0.6}{$\pm4$}}$
& $46{\scalebox{0.6}{$\pm3$}}$
& $57{\scalebox{0.6}{$\pm2$}}$
& $\mathbf{71}{\scalebox{0.6}{$\pm4$}}$ \\
\texttt{cube-single-noisy}
& $71{\scalebox{0.6}{$\pm9$}}$
& $41{\scalebox{0.6}{$\pm6$}}$
& $80{\scalebox{0.6}{$\pm10$}}$
& $80{\scalebox{0.6}{$\pm14$}}$
& $\mathbf{91}{\scalebox{0.6}{$\pm11$}}$ \\
\texttt{cube-double-noisy}
& $14{\scalebox{0.6}{$\pm3$}}$
& $2{\scalebox{0.6}{$\pm1$}}$
& $\mathbf{15}{\scalebox{0.6}{$\pm2$}}$
& $15{\scalebox{0.6}{$\pm2$}}$
& $\mathbf{16}{\scalebox{0.6}{$\pm8$}}$ \\
\texttt{puzzle-3x3-noisy}
& $42{\scalebox{0.6}{$\pm19$}}$
& $\mathbf{51}{\scalebox{0.6}{$\pm11$}}$
& $39{\scalebox{0.6}{$\pm2$}}$
& $47{\scalebox{0.6}{$\pm12$}}$
& $42{\scalebox{0.6}{$\pm6$}}$ \\
\texttt{puzzle-4x4-noisy}
& $20{\scalebox{0.6}{$\pm3$}}$
& $16{\scalebox{0.6}{$\pm4$}}$
& $3{\scalebox{0.6}{$\pm1$}}$
& $40{\scalebox{0.6}{$\pm3$}}$
& $\mathbf{96}{\scalebox{0.6}{$\pm1$}}$ \\
\midrule
\multicolumn{1}{c}{\textbf{Average}}
& $27.3$
& $18.0$
& $24.1$
& $30.1$
& $\mathbf{42.3}$ \\
\bottomrule
\end{tabularx}
\end{table}

\subsection{Ablations} \label{appendix:additional_experiments:hierarchical_structure}
\paragraph{Hierarchical structure.}
\begin{table}[t]
\centering
\captionsetup{skip=12pt, font=small, width=\columnwidth}
\caption{\textbf{Importance of hierarchical structure (4 seeds).} Comparison between GCIVL$^{\vee}_{+\alpha^\vee}$ and CTA. \textbf{Bold} indicates the best score, and values within 95\% of the best are also bold.}
\label{tab:noisy_gcivlvee_alpha_vs_cta}
\renewcommand{\arraystretch}{1.08}
\small
\begin{tabularx}{0.6\columnwidth}{l *{2}{>{\centering\arraybackslash}X}}
\toprule
\textbf{Environment} &
\textbf{GCIVL}$^{\vee}_{+\alpha^\vee}$ &
\textbf{CTA} \\
\midrule
\texttt{scene-play}
& $74{\scalebox{0.6}{$\pm1$}}$
& $\mathbf{90}{\scalebox{0.6}{$\pm4$}}$ \\
\texttt{cube-single-play}
& $\mathbf{96}{\scalebox{0.6}{$\pm1$}}$
& $86{\scalebox{0.6}{$\pm3$}}$ \\
\texttt{cube-double-play}
& $40{\scalebox{0.6}{$\pm6$}}$
& $\mathbf{50}{\scalebox{0.6}{$\pm5$}}$ \\
\texttt{cube-triple-play}
& $6{\scalebox{0.6}{$\pm2$}}$
& $\mathbf{17}{\scalebox{0.6}{$\pm1$}}$ \\
\texttt{puzzle-3x3-play}
& $10{\scalebox{0.6}{$\pm1$}}$
& $\mathbf{94}{\scalebox{0.6}{$\pm11$}}$ \\
\texttt{puzzle-4x4-play}
& $3{\scalebox{0.6}{$\pm1$}}$
& $\mathbf{84}{\scalebox{0.6}{$\pm3$}}$ \\
\texttt{puzzle-4x5-play}
& $2{\scalebox{0.6}{$\pm1$}}$
& $\mathbf{17}{\scalebox{0.6}{$\pm1$}}$ \\
\texttt{puzzle-4x6-play}
& $5{\scalebox{0.6}{$\pm2$}}$
& $\mathbf{12}{\scalebox{0.6}{$\pm2$}}$ \\
\midrule
\multicolumn{1}{c}{\textbf{Average}}
& $29.5$
& $\mathbf{56.3}$ \\
\bottomrule
\end{tabularx}
\end{table}

To isolate the contribution of CTA's hierarchical structure, we construct a non-hierarchical baseline by applying the same bilinear transduction parameterization used in CTA to GCIVL$^{\vee}$, yielding GCIVL$^{\vee}_{+\alpha^\vee}$.
Empirically, as shown in \cref{tab:noisy_gcivlvee_alpha_vs_cta}, CTA substantially outperforms this baseline and exhibits markedly more reliable analogy transduction, indicating that bilinear transduction alone is insufficient without hierarchy.
The hierarchical structure improves compositional generalization by making analogy transduction both more effective and more stable.
Since long-horizon analogies are sparse in offline datasets~\citep{hong2023diffused, myers2025horizon}, we decompose behavior into $k$-step analogies for CTA, increasing trajectory overlap and the pool of reusable analogies.
Shorter-horizon analogies are also more feasible, and conditioning the low-level policy on proposed analogies stabilizes execution while avoiding out-of-distribution analogy queries outside the intended OOC regime.

\paragraph{Subgoal steps $k$.}
To assess the robustness of the hierarchical structure to the choice of the subgoal step $k$, we conduct an ablation study over $k=10, 20, 30, 40$.
We find that $k=10$ for \texttt{scene}, $k=30$ for \texttt{cube}, and $k=20$ for \texttt{puzzle} yield the most stable analogy transduction, achieving both higher mean performance and lower standard deviation (see \cref{fig:ss_ablation_grid_tall}).
We attribute this to the fact that these horizons align with the characteristic interaction timescales of each environment, so that the resulting $k$-step segments are more likely to induce meaningful task-endogenous displacements through object contact and manipulation.

\paragraph{Transductive feature dimension $b$.}
To examine how the low-rank bottleneck affects OOC extrapolation, we conduct an ablation study with $b=4,8,16,32$. Figure \ref{fig:b_ablation_grid_tall} suggests that a larger bottleneck helps in the more complex \texttt{scene} environment, where higher analogy expressivity appears beneficial, whereas smaller bottlenecks work better in the \texttt{puzzle} environments, where behavior is simple but combinatorial generalization is critical. Overall, these trends are consistent with the bilinear transduction theory, where there exists a tradeoff between expressivity (larger $b$) and OOC generalization (smaller $b$).

\section{Hyperparameters}
We report the hyperparameters for CTA, $\mathrm{HIQL}^\vee$, and $\mathrm{HIQL}^\vee\!+\!\alpha^\vee$ in \cref{tab:hyperparams} and \cref{tab:hyperparams_newdualhiql}.

% requires: \usepackage{booktabs,tabularx,array}
\begin{table}[t]
\centering
\caption{\textbf{CTA Hyperparameters.}}
\label{tab:hyperparams}
\renewcommand{\arraystretch}{1.08}
\scriptsize
\begin{tabularx}{0.8\textwidth}{>{\raggedright\arraybackslash}p{0.4\textwidth} >{\raggedright\arraybackslash}X}
\toprule
\textbf{Hyperparameter} & \textbf{Value} \\
\midrule
Learning rate & \texttt{3e-4} \\
Batch size &
\texttt{256 (cube-single, puzzle-3x3, puzzle-4x4)} \newline
\texttt{512 (cube-double, puzzle-4x5)} \newline
\texttt{1024 (cube-triple, puzzle-4x6)} \\
Analogy projection MLP size $\eta$ & \texttt{(256, 256)} \\
Dual representation MLP size $\varphi$ & \texttt{(512, 512, 512)} \\
Transductive anchor MLP size  & \texttt{(128, 128, 128)} \\
Transductive displacement MLP size & \texttt{(128, 128, 128)} \\
Transductive feature dimension $b$ & \texttt{8} \\
Actor backbone MLP size & \texttt{(128, 128)} \\
Value backbone MLP size & \texttt{(128, 128)} \\
Nonlinearity & \texttt{GELU} \citep{hendrycks2016gaussian} \\
Layer normalization \citep{ba2016layer} & \texttt{True} \\
Discount factor $\gamma$ & \texttt{0.99} \\
Target network update rate $\tau$ & \texttt{0.005} \\
Dual representation expectile $\iota$ & \texttt{0.7} \\
IQL expectile $\kappa$ & \texttt{0.7} \\
Low-level AWR temperature $\beta_{\ell}$ & \texttt{3.0} \\
High-level AWR temperature $\beta_{h}$ & \texttt{3.0} \\
Subgoal steps $k$ &
\texttt{10 (scene)} \newline
\texttt{30 (cube)} \newline
\texttt{20 (puzzle)} \newline
\texttt{25 (maze)} \\
Analogy projection $\eta$ representation dimension & \texttt{32} \\
Dual representation dimension $d$ & \texttt{256} \\
Visual encoder & \texttt{impala small} \citep{espeholt2018impala} \\
Value goal: current-state probability $p_{\mathrm{cur}}$ & \texttt{0.2} \\
Value goal: trajectory-future probability $p_{\mathrm{traj}}$ & \texttt{0.5} \\
Value goal: random probability $p_{\mathrm{rand}}$ & \texttt{0.3} \\
Value goal: geometric sampling & \texttt{True} \\
Actor goal: current-state probability $p_{\mathrm{cur}}$ & \texttt{0.0} \\
Actor goal: trajectory-future probability $p_{\mathrm{traj}}$ & \texttt{1.0} \\
Actor goal: random probability $p_{\mathrm{rand}}$ & \texttt{0.0} \\
Actor goal: geometric sampling & \texttt{False} \\
Image augmentation probability & \texttt{0.5} \\
\bottomrule
\end{tabularx}
\end{table}

\begin{table}[t]
\centering
\caption{\textbf{HIQL$^\vee$ and HIQL$^\vee_{+\alpha^\vee}$ Hyperparameters.}}
\label{tab:hyperparams_newdualhiql}
\renewcommand{\arraystretch}{1.08}
\scriptsize
\begin{tabularx}{0.8\textwidth}{>{\raggedright\arraybackslash}p{0.4\textwidth} >{\raggedright\arraybackslash}X}
\toprule
\textbf{Hyperparameter} & \textbf{Value} \\
\midrule
Learning rate & \texttt{3e-4} \\
Batch size &
\texttt{256 (cube-single, puzzle-3x3, puzzle-4x4)} \newline
\texttt{512 (cube-double, puzzle-4x5)} \newline
\texttt{1024 (cube-triple, puzzle-4x6)} \\
Goal projection MLP size $\eta$ & \texttt{(256, 256)} \\
Dual representation MLP size $\varphi$ & \texttt{(512, 512, 512)} \\
Actor MLP size & \texttt{(512, 512, 512)} \\
Value MLP size & \texttt{(512, 512, 512)} \\
Nonlinearity & \texttt{GELU} \citep{hendrycks2016gaussian} \\
Layer normalization \citep{ba2016layer} & \texttt{True} \\
Discount factor $\gamma$ & \texttt{0.99} \\
Target network update rate $\tau$ & \texttt{0.005} \\
Dual representation expectile $\iota$ & \texttt{0.7} \\
IQL expectile $\kappa$ & \texttt{0.7} \\
Low-level AWR temperature $\beta_{\ell}$ & \texttt{3.0} \\
High-level AWR temperature $\beta_{h}$ & \texttt{3.0} \\
Subgoal steps $k$ &
\texttt{10 (scene)} \newline
\texttt{30 (cube)} \newline
\texttt{20 (puzzle)} \newline
\texttt{25 (maze)} \\
Goal representation $\eta$ dimension & \texttt{32} \\
Dual representation dimension $d$ & \texttt{256} \\
Visual encoder & \texttt{impala small} \citep{espeholt2018impala}\\
Value goal: current-state probability $p_{\mathrm{cur}}$ & \texttt{0.2} \\
Value goal: trajectory-future probability $p_{\mathrm{traj}}$ & \texttt{0.5} \\
Value goal: random probability $p_{\mathrm{rand}}$ & \texttt{0.3} \\
Value goal: geometric sampling & \texttt{True} \\
Actor goal: current-state probability $p_{\mathrm{cur}}$ & \texttt{0.0} \\
Actor goal: trajectory-future probability $p_{\mathrm{traj}}$ & \texttt{1.0} \\
Actor goal: random probability $p_{\mathrm{rand}}$ & \texttt{0.0} \\
Actor goal: geometric sampling & \texttt{False} \\
Image augmentation probability & \texttt{0.5} \\
\bottomrule
\end{tabularx}
\end{table}

\begin{figure}[t]
    \centering
    \captionsetup[subfigure]{font=tiny}
    \newcommand{\colw}{0.20\textwidth}
    \newcommand{\colsep}{0.04\textwidth}

    % Row 1: query pairs
    \makebox[\textwidth][l]{%
    \hspace*{0.15\textwidth}%
    \begin{subfigure}[t]{\colw}
        \centering
        \includegraphics[width=\linewidth]{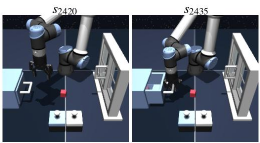}
        \caption*{Query 1}
    \end{subfigure}%
    \hspace{\colsep}%
    \begin{subfigure}[t]{\colw}
        \centering
        \includegraphics[width=\linewidth]{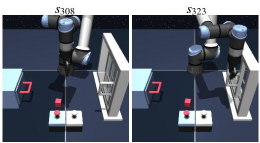}
        \caption*{Query 2}
    \end{subfigure}%
    \hspace{\colsep}%
    \begin{subfigure}[t]{\colw}
        \centering
        \includegraphics[width=\linewidth]{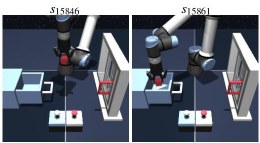}
        \caption*{Query 3}
    \end{subfigure}%
    }

    \vspace{2mm}

    % Row 2: nearest analogies
    \makebox[\textwidth][l]{%
    \hspace*{0.15\textwidth}%
    \begin{subfigure}[t]{\colw}
        \centering
        \includegraphics[width=\linewidth]{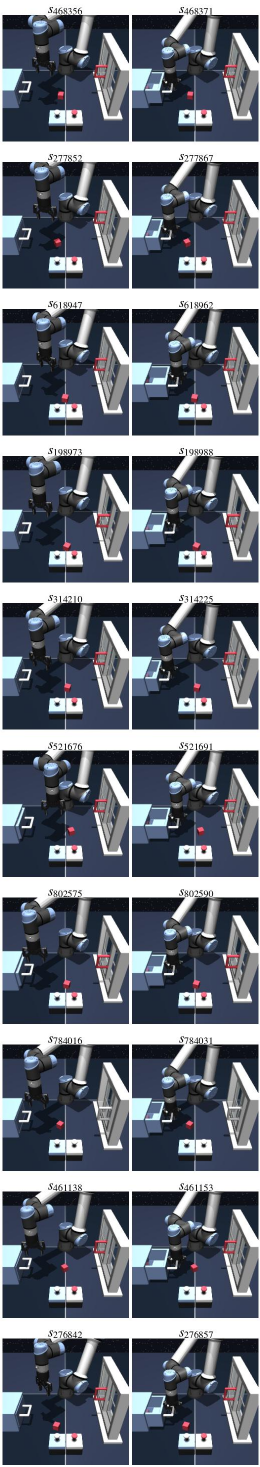}
        \caption*{Top-10 Nearest Analogy}
    \end{subfigure}%
    \hspace{\colsep}%
    \begin{subfigure}[t]{\colw}
        \centering
        \includegraphics[width=\linewidth]{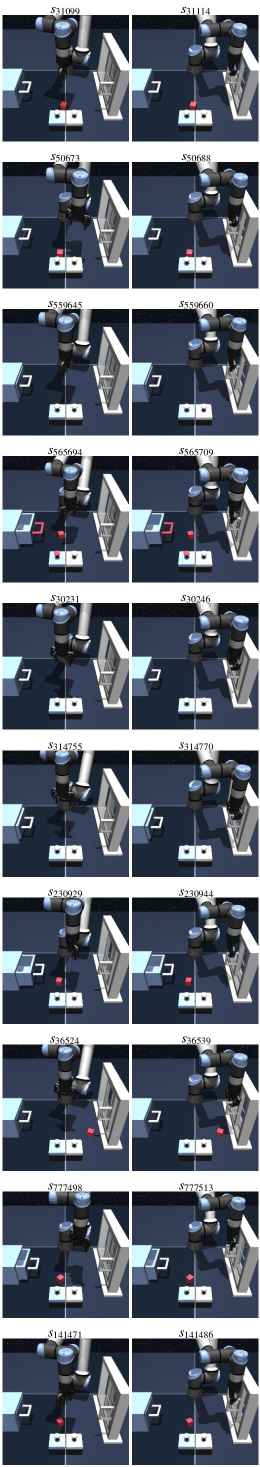}
        \caption*{Top-10 Nearest Analogy}
    \end{subfigure}%
    \hspace{\colsep}%
    \begin{subfigure}[t]{\colw}
        \centering
        \includegraphics[width=\linewidth]{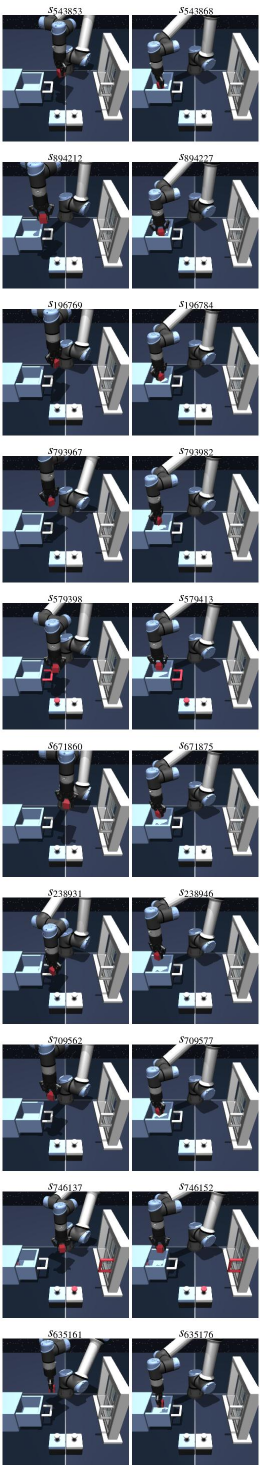}
        \caption*{Top-10 Nearest Analogy}
    \end{subfigure}%
    }

    \caption{\textbf{Qualitative visualization of dual analogies.}
    For each OOC query pair, we visualize the query and its top-10 nearest analogies.}
    \label{fig:dual_analogy_scene}
\end{figure}

\begin{figure}[t]
    \centering
    \captionsetup[subfigure]{font=tiny}
    \newcommand{\colw}{0.2\textwidth}
    \newcommand{\colsep}{0.04\textwidth}

    % Row 1: query pairs
    \makebox[\textwidth][l]{%
    \hspace*{0.15\textwidth}%
    \begin{subfigure}[t]{\colw}
        \centering
        \includegraphics[width=\linewidth]{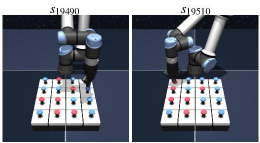}
        \caption*{Query 1}
    \end{subfigure}%
    \hspace{\colsep}%
    \begin{subfigure}[t]{\colw}
        \centering
        \includegraphics[width=\linewidth]{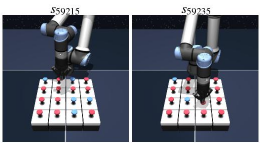}
        \caption*{Query 2}
    \end{subfigure}%
    \hspace{\colsep}%
    \begin{subfigure}[t]{\colw}
        \centering
        \includegraphics[width=\linewidth]{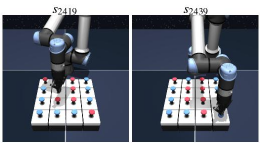}
        \caption*{Query 3}
    \end{subfigure}%
    }

    \vspace{2mm}

    % Row 2: nearest analogies
    \makebox[\textwidth][l]{%
    \hspace*{0.15\textwidth}%
    \begin{subfigure}[t]{\colw}
        \centering
        \includegraphics[width=\linewidth]{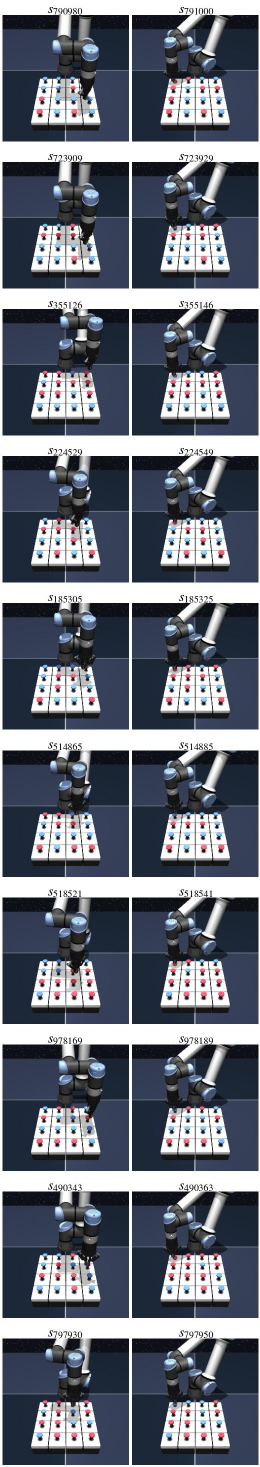}
        \caption*{Top-10 Nearest Analogy}
    \end{subfigure}%
    \hspace{\colsep}%
    \begin{subfigure}[t]{\colw}
        \centering
        \includegraphics[width=\linewidth]{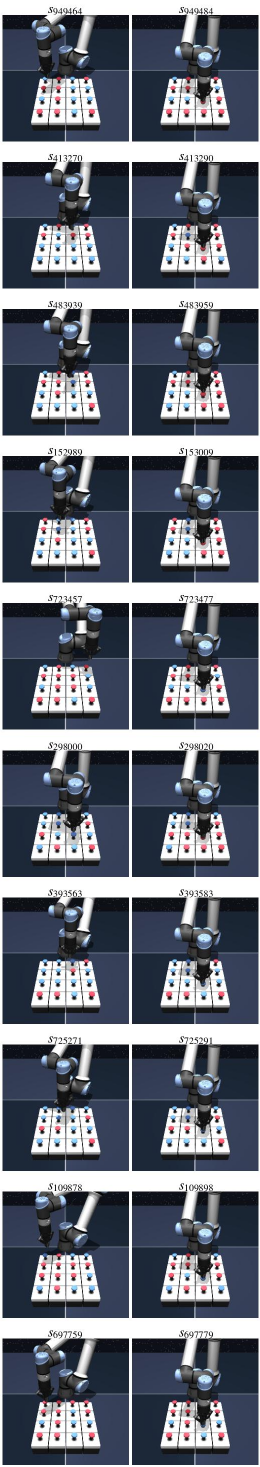}
        \caption*{Top-10 Nearest Analogy}
    \end{subfigure}%
    \hspace{\colsep}%
    \begin{subfigure}[t]{\colw}
        \centering
        \includegraphics[width=\linewidth]{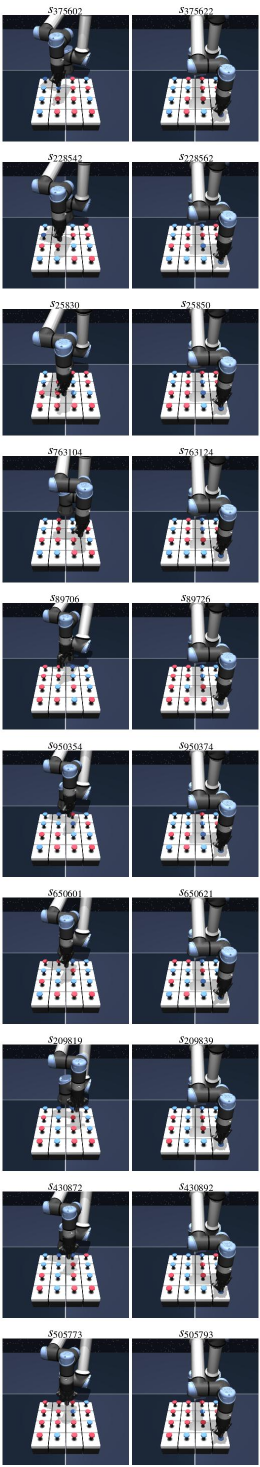}
        \caption*{Top-10 Nearest Analogy}
    \end{subfigure}%
    }

    \caption{\textbf{Qualitative visualization of dual analogies.}
    For each OOC query pair, we visualize the query and its top-10 nearest analogies.}
    \label{fig:dual_analogy_puzzle}
\end{figure}

\begin{figure*}[t]
    \centering
    \setlength{\tabcolsep}{3pt}
    \renewcommand{\arraystretch}{0}

    \begin{tabular}{@{}p{0.15\textwidth}ccp{0.15\textwidth}@{}}
        % scene
        &
        \includegraphics[width=0.34\textwidth]{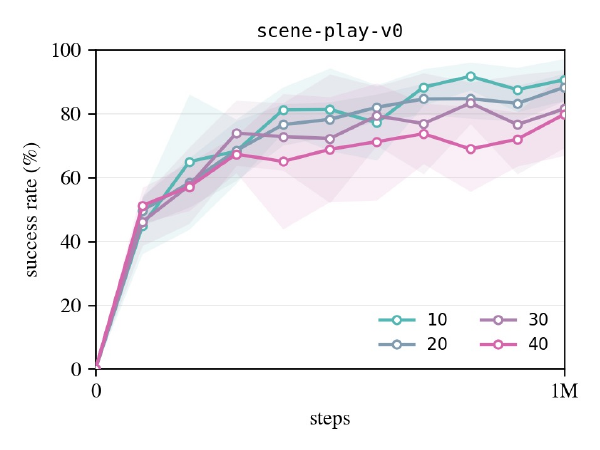} &
        \includegraphics[width=0.34\textwidth]{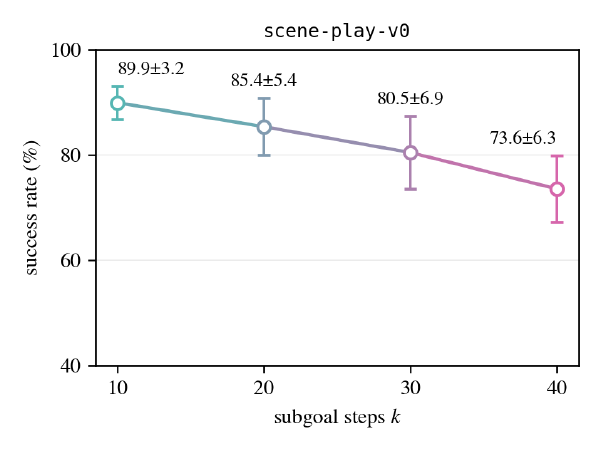} &
        \\
        
        % cube-single
        &
        \includegraphics[width=0.34\textwidth]{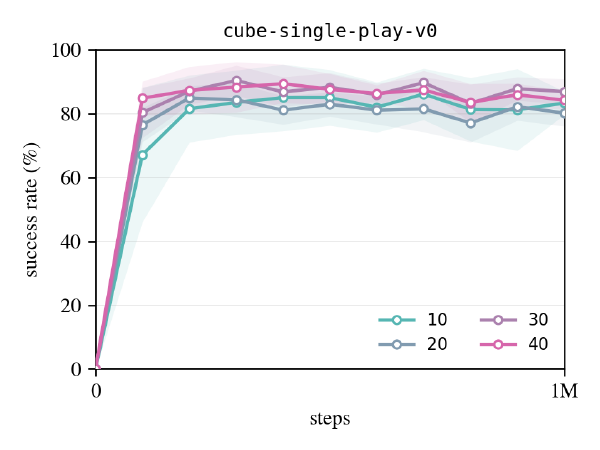} &
        \includegraphics[width=0.34\textwidth]{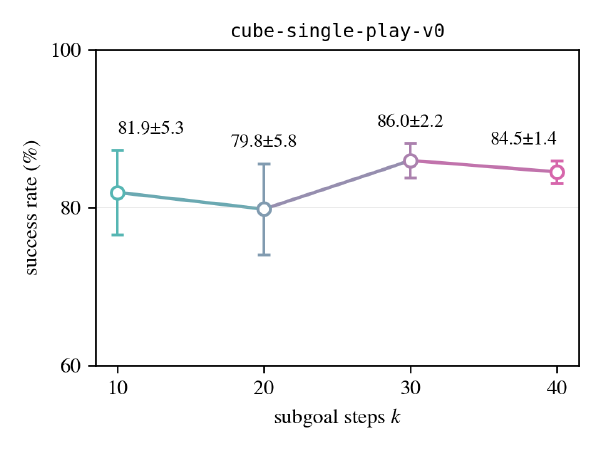} &
        \\

        % cube-double
        &
        \includegraphics[width=0.34\textwidth]{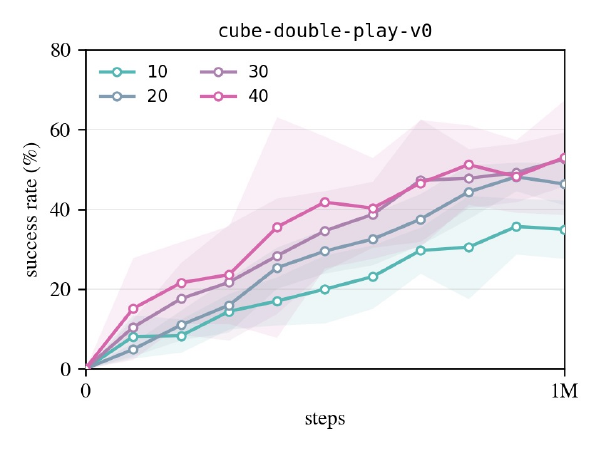} &
        \includegraphics[width=0.34\textwidth]{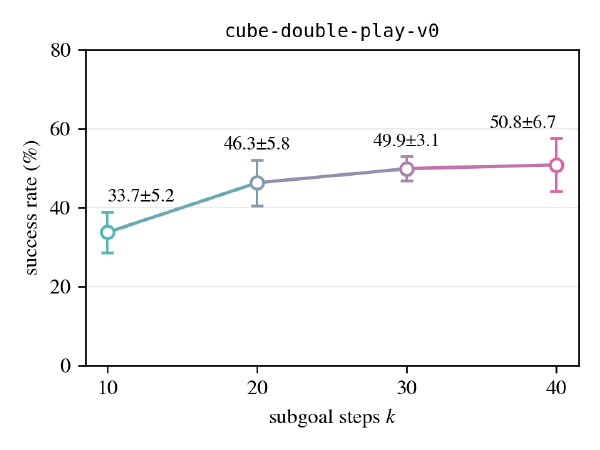} &
        \\

        % puzzle-3x3
        &
        \includegraphics[width=0.34\textwidth]{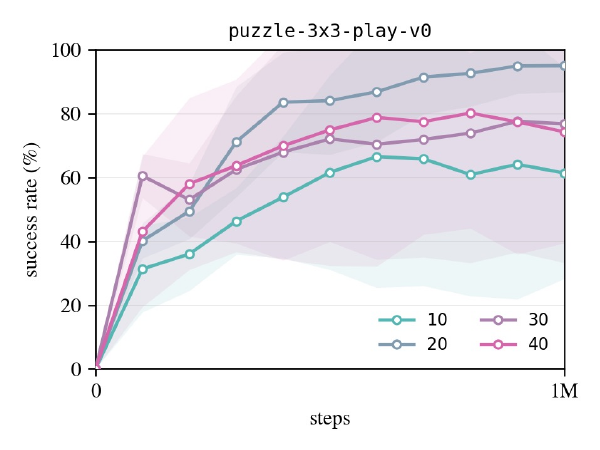} &
        \includegraphics[width=0.34\textwidth]{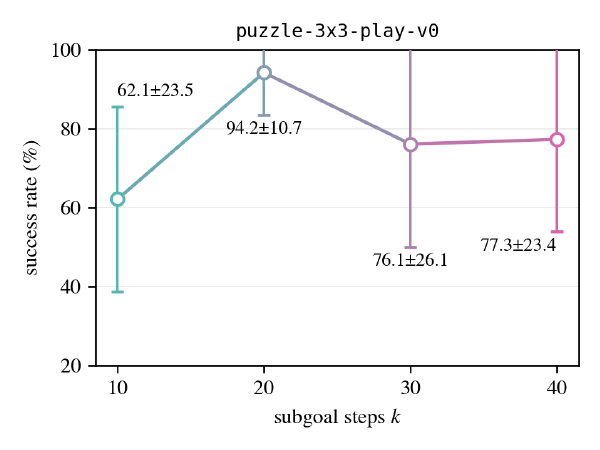} &
        \\

        % puzzle-4x4
        &
        \includegraphics[width=0.34\textwidth]{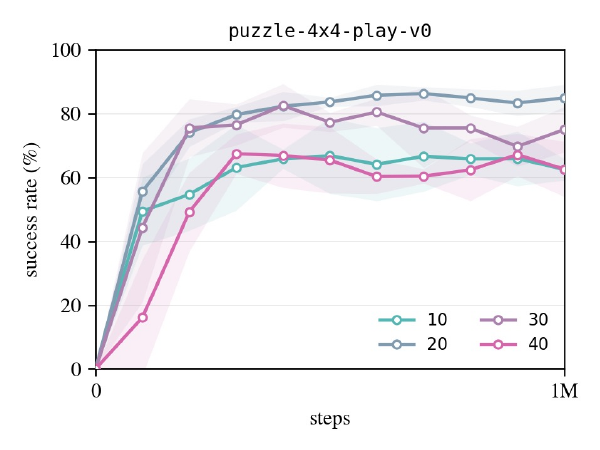} &
        \includegraphics[width=0.34\textwidth]{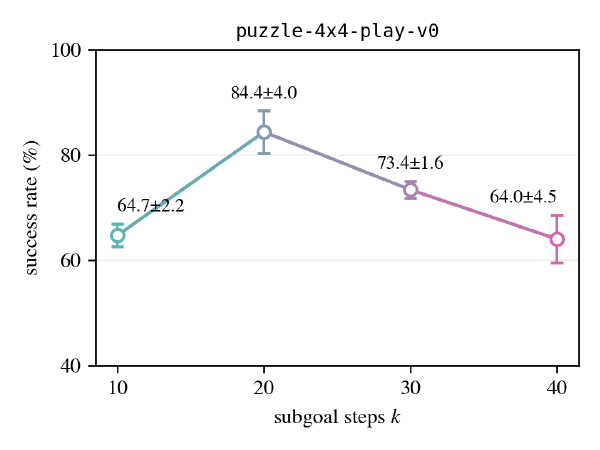} &
        \\
    \end{tabular}

    \captionsetup{skip=3pt, font=small, width=\textwidth}
    \caption{\textbf{Ablation results on subgoal steps $k$.}
    Left: step-wise success rate curves. Right: final performance aggregated over the last three evaluation steps.}
    \label{fig:ss_ablation_grid_tall}
\end{figure*}

\begin{figure*}[t]
    \centering
    \setlength{\tabcolsep}{3pt}
    \renewcommand{\arraystretch}{0}

    \begin{tabular}{@{}p{0.15\textwidth}ccp{0.15\textwidth}@{}}
        % scene
        &
        \includegraphics[width=0.34\textwidth]{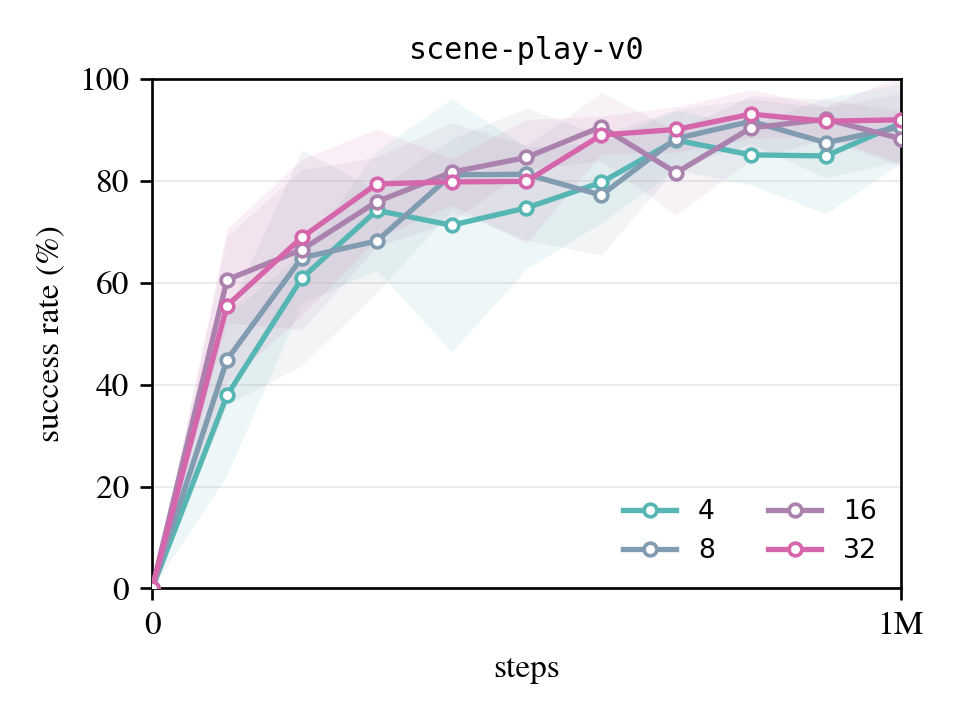} &
        \includegraphics[width=0.34\textwidth]{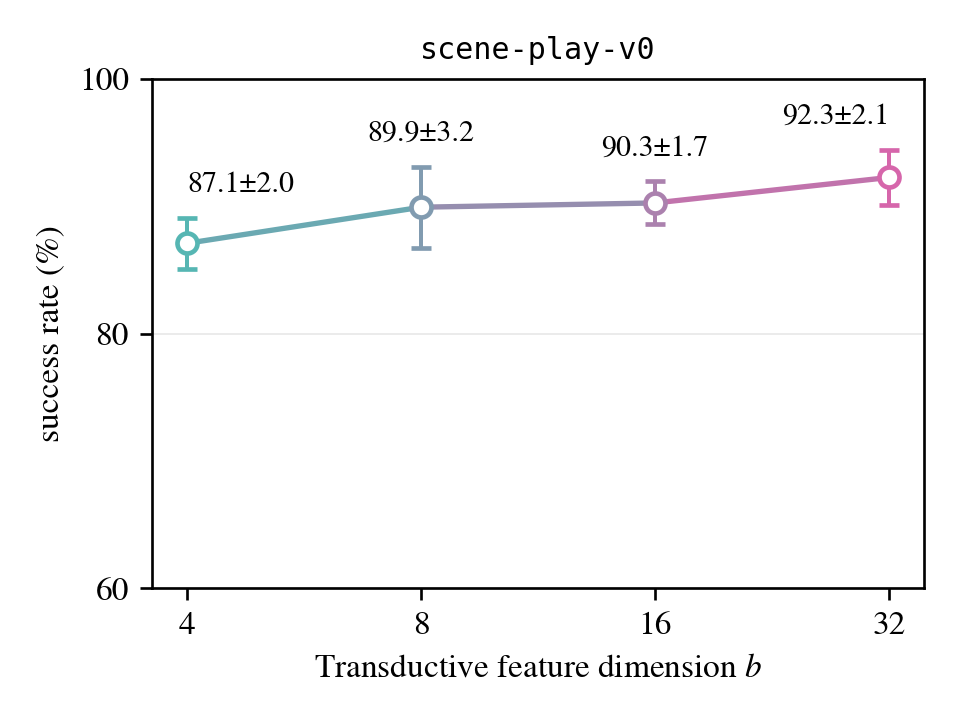} &
        \\
        
        % cube-single
        &
        \includegraphics[width=0.34\textwidth]{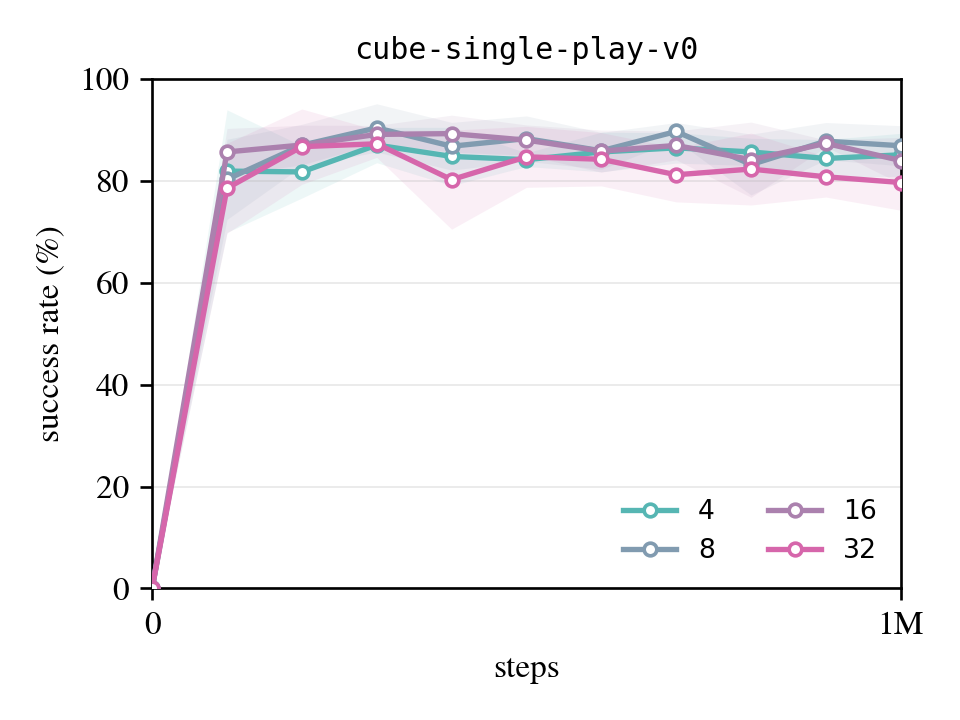} &
        \includegraphics[width=0.34\textwidth]{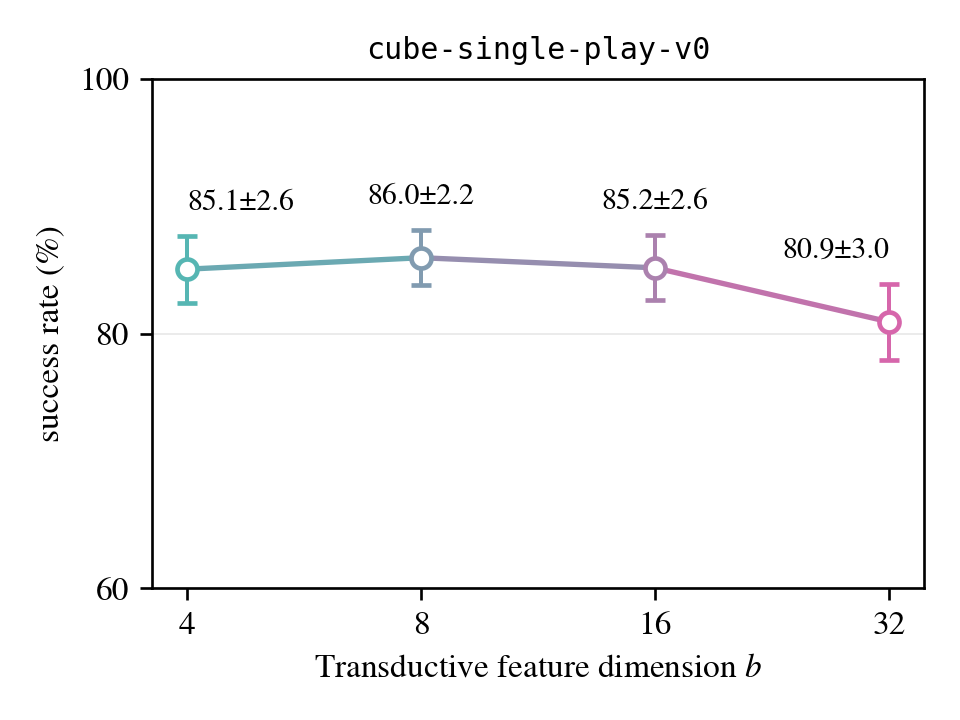} &
        \\

        % cube-double
        &
        \includegraphics[width=0.34\textwidth]{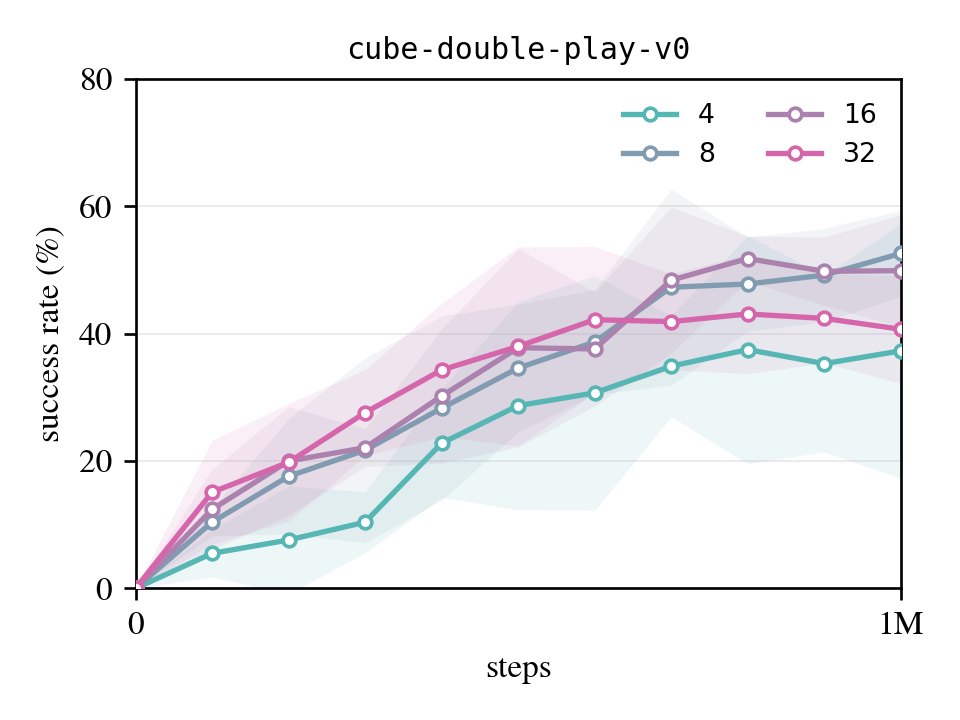} &
        \includegraphics[width=0.34\textwidth]{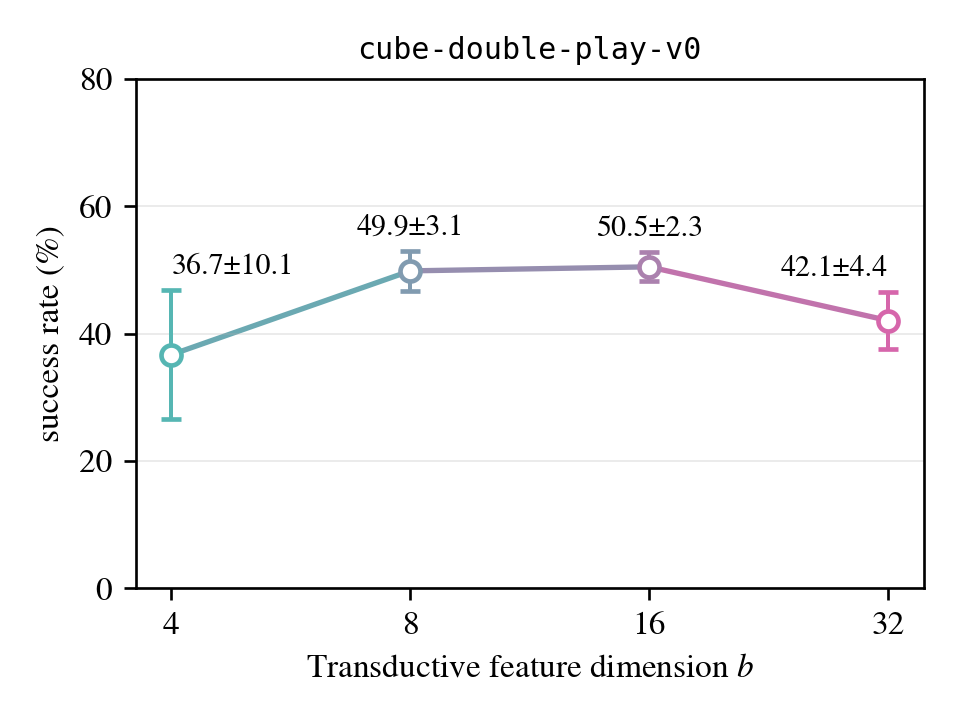} &
        \\

        % puzzle-3x3
        &
        \includegraphics[width=0.34\textwidth]{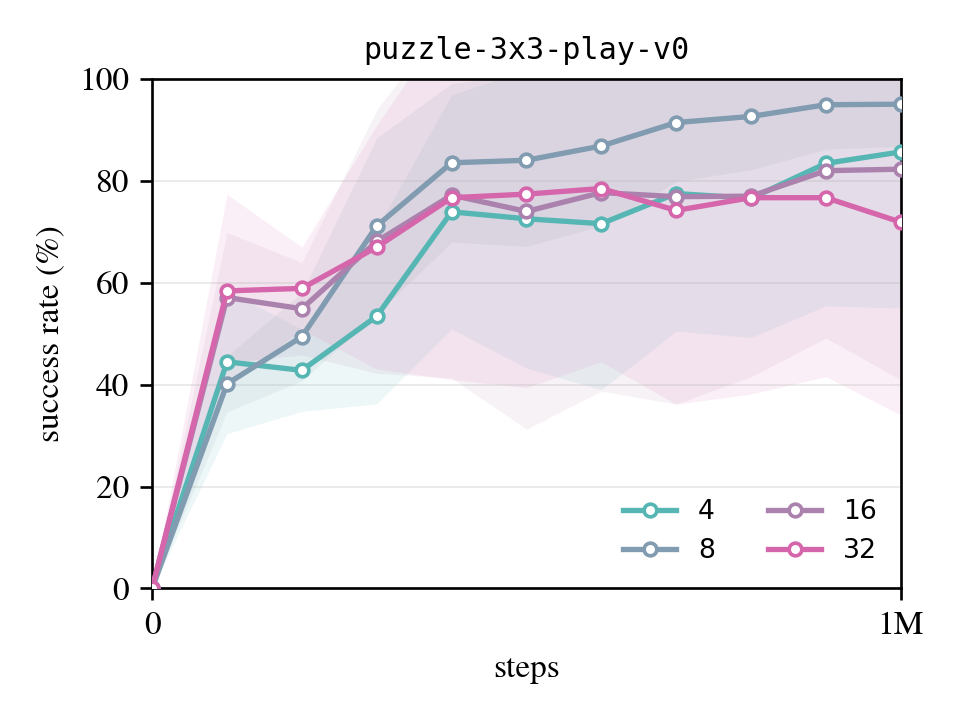} &
        \includegraphics[width=0.34\textwidth]{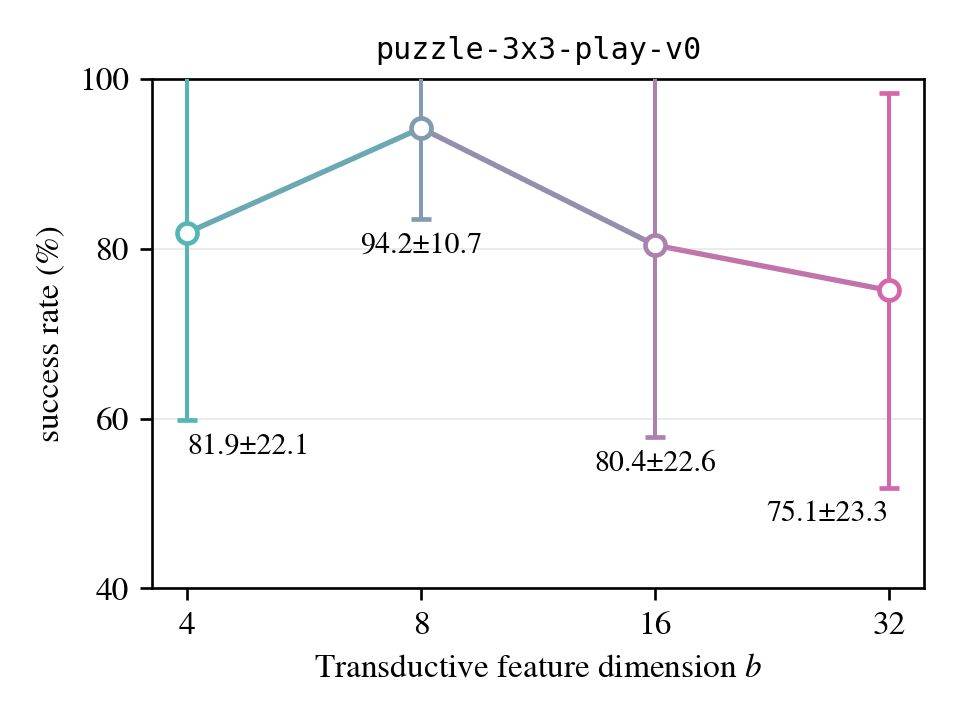} &
        \\

        % puzzle-4x4
        &
        \includegraphics[width=0.34\textwidth]{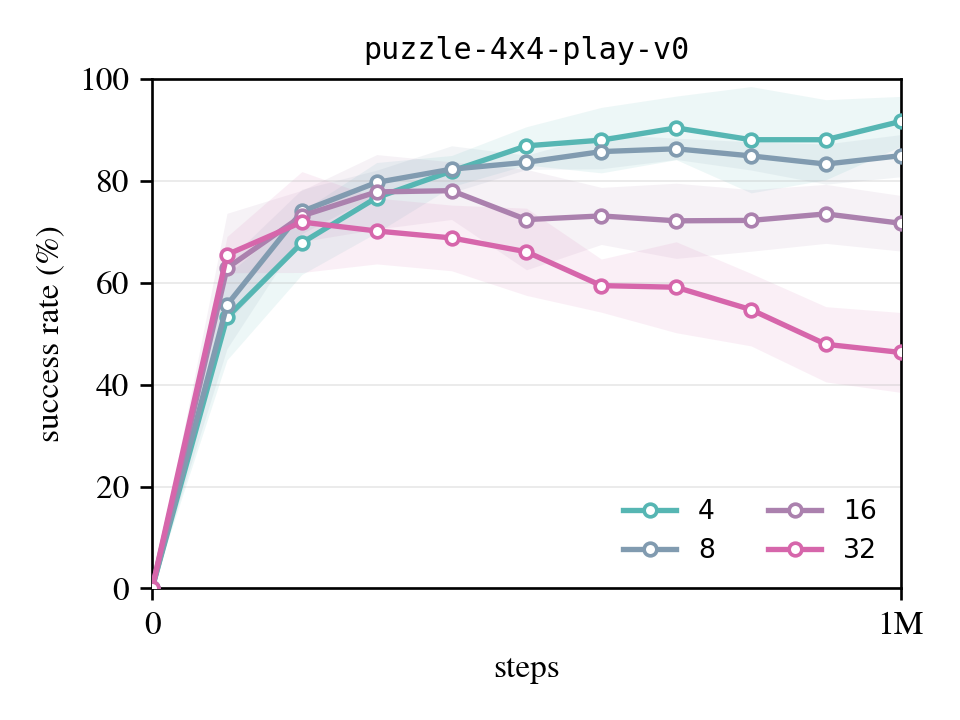} &
        \includegraphics[width=0.34\textwidth]{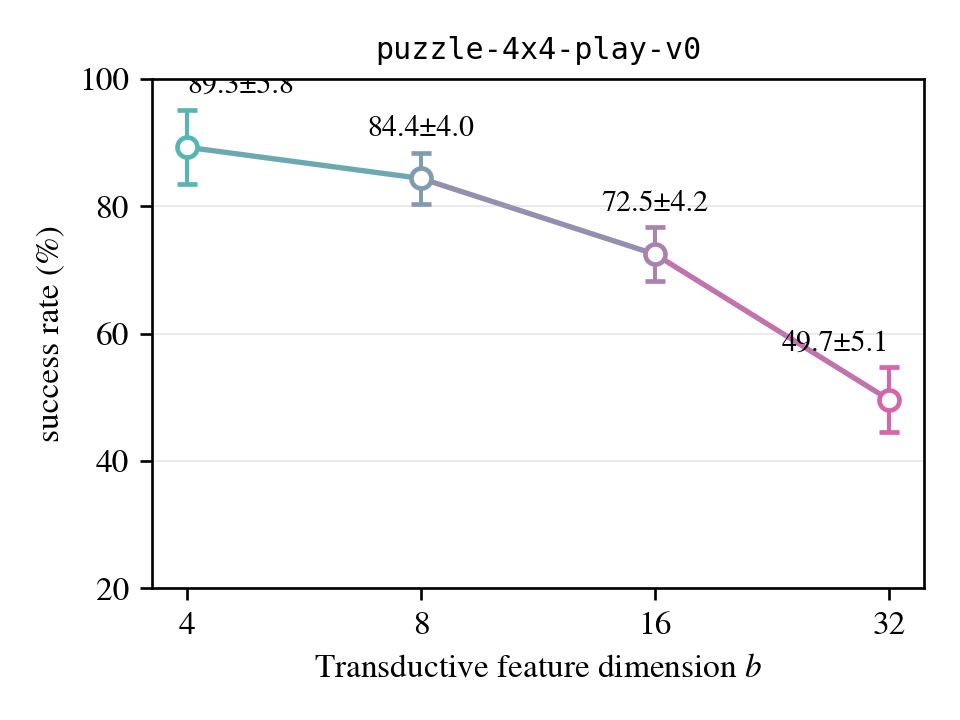} &
        \\
    \end{tabular}

    \captionsetup{skip=3pt, font=small, width=\textwidth}
    \caption{\textbf{Ablation results on transductive feature dimension $b$.}
    Left: step-wise success rate curves. Right: final performance aggregated over the last three evaluation steps.}
    \label{fig:b_ablation_grid_tall}
\end{figure*}
%%%%%%%%%%%%%%%%%%%%%%%%%%%%%%%%%%%%%%%%%%%%%%%%%%%%%%%%%%%%%%%%%%%%%%%%%%%%%%%
%%%%%%%%%%%%%%%%%%%%%%%%%%%%%%%%%%%%%%%%%%%%%%%%%%%%%%%%%%%%%%%%%%%%%%%%%%%%%%%

\end{document}